\documentclass[12pt]{article}
\usepackage{amsmath,amssymb,amsthm}
\usepackage{psfrag, url}
\usepackage{graphicx}
\usepackage{times}
\usepackage{multirow}
\usepackage{color}
\usepackage[authoryear]{natbib}
\usepackage{rotating}
\usepackage{bbm}
\usepackage{latexsym}
\usepackage{subfigure}

\def\pmb#1{\setbox0=\hbox{#1}%
  \kern-.02em\copy0\kern-\wd0
  \kern-.035em\copy0\kern-\wd0
  \kern-.02em\raise.03em\box0 }

\newcommand{\vv}[1]{\mbox{{\boldmath $#1$}}}
\newcommand{\cc}[1]{{\cal{#1}}}

\newcommand{\tp}{^{\top}}

\def\p3{\cc{P}^3}
\def\avect{\mbox{\boldmath $a$}}
\def\bvect{\mbox{\boldmath $b$}}

\def\evect{\mbox{\boldmath $e$}}
\def\fvect{\mbox{\boldmath $f$}}
\def\gvect{\mbox{\boldmath $g$}}

\def\svect{\mbox{\boldmath $s$}}

\def\vvect{\mbox{\boldmath $v$}}
\def\wvect{\mbox{\boldmath $w$}}

\def\Avect{\mbox{\boldmath $A$}}
\def\Bvect{\mbox{\boldmath $B$}}

\def\Fvect{\mbox{\boldmath $F$}}
\def\Gvect{\mbox{\boldmath $G$}}

\def\fmat{\mbox{\bf f}}
\def\gmat{\mbox{\bf g}}

\def\smat{\mbox{\bf s}}

\def\Hmat{\mbox{\bf H}}
\def\Imat{\mbox{\bf I}}

\def\Vmat{\mbox{\bf V}}

\def\EE{\mathbb{E}}

\def\spkr{\svect}

\def\sind{n}

\def\smind{N}
\def\outind{\smind+1}
\def\rcx{x}
\def\rcy{y}
\def\rcz{z}

\def\vobs{\fvect}
\def\vobss{\fmat}
\def\vfunc{\mathcal{F}}
\def\vind{m}

\def\vcu{u}
\def\vcv{v}
\def\vcd{d}

\def\aobss{\gvect}
\def\afunc{\mathcal{G}}

\def\iterind{q}
\def\iter{{(\iterind)}}
\def\itern{{(\iterind+1)}}

\def\vobsind{\vobs_\vind}

\def\params{\mbox{\boldmath $\theta$}}
\def\paramsq{\params^{\iter}}
\def\paramsqn{\params^{\itern}}

\def\normal{{\mathcal{N}}}

\def\uni{\mathcal{U}}
\def\lhood{{\mathcal{L}}}
\def\nablavect{\mbox{\boldmath $\nabla$}}
\def\condexp{Q}
\def\fcondexp{Q_\fosm}
\def\scondexp{Q_\sosm}

\newcommand{\argmax}{\operatornamewithlimits{argmax}}

\def\fos{\mathbb{F}}
\def\fosm{\mathcal{F}}
\def\fobs{\fvect}
\def\Fobs{\Fvect}
\def\fobsavg{\bar{\fobs}}
\def\fobss{\fmat}
\def\find{m}
\def\fmind{M}
\def\fdim{r}
\def\fass{A}
\def\fsass{a}
\def\fpprob{\alpha}
\def\fasss{\Avect}
\def\fsasss{\avect}
\def\fpr{\pi}
\def\fvar{{\bf \Sigma}}

\def\fvol{V}

\def\fobsind{\fobs_\find}
\def\fobsmind{\fobs_\fmind}
\def\fobsavgind{\fobsavg_\cind}

\def\fassind{\fass_\find}
\def\fassmind{\fass_\fmind}
\def\fsassind{\fsass_\find}

\def\fprind{\fpr_\cind}
\def\fprindq{\fprind^{\iter}}
\def\fprindiq{\fpr_i^{\iter}}
\def\fprindqn{\fprind^{\itern}}
\def\fprmind{\fpr_\cmind}
\def\fprout{\fpr_\outind}
\def\fproutq{\fprout^{\iter}}
\def\fvarind{\fvar_\cind}

\def\fvarindq{\fvarind^{\iter}}
\def\fvarindiq{\fvar_i^{\iter}}
\def\fvarindqn{\fvarind^{\itern}}
\def\fvarmind{\fvar_\cmind}

\def\fpprobind{\fpprob_{\find\cind}}
\def\fpprobindq{\fpprobind^{\iter}}
\def\fpprobout{\fpprob_{\find,\outind}}
\def\fpproboutq{\fpprobout^{\iter}}

\def\sos{\mathbb{G}}
\def\sosm{\mathcal{G}}
\def\sobs{\gvect}
\def\Sobs{\Gvect}
\def\sobsavg{\bar{\gvect}}
\def\sobss{\gmat}
\def\sind{k}
\def\smind{K}
\def\sdim{p}
\def\sass{B}
\def\ssass{b}
\def\spprob{\beta}
\def\sasss{\Bvect}
\def\ssasss{\bvect}
\def\spr{\lambda}
\def\svar{{\bf \Gamma}}

\def\svol{U}

\def\sobsind{\sobs_\sind}
\def\sobsmind{\sobs_\smind}
\def\sobsavgind{\sobsavg_\cind}

\def\sassind{\sass_\sind}
\def\sassmind{\sass_\smind}
\def\ssassind{\ssass_\sind}

\def\sprind{\spr_\cind}
\def\sprindq{\sprind^{\iter}}
\def\sprindiq{\spr_i^{\iter}}
\def\sprindqn{\sprind^{\itern}}
\def\sprmind{\spr_\cmind}
\def\sprout{\spr_\outind}
\def\sproutq{\spr^{(q)}_\outind}
\def\svarind{\svar_\cind}

\def\svarindq{\svarind^{\iter}}
\def\svarindiq{\svar_i^{\iter}}
\def\svarindqn{\svarind^{\itern}}
\def\svarmind{\svar_\cmind}

\def\spprobind{\spprob_{\sind\cind}}
\def\spprobindq{\spprobind^{\iter}}
\def\spprobout{\spprob_{\sind,\outind}}
\def\spproboutq{\spprobout^{\iter}}

\def\pars{\mathbb{S}}
\def\param{\svect}
\def\pparams{\smat}

\def\cind{n}
\def\cmind{N}
\def\pdim{d}
\def\outind{{\cmind+1}}

\def\fpprobsum{\bar{\fpprob}}
\def\fpprobsumind{\fpprobsum_{\cind}}

\def\spprobsum{\bar{\spprob}}
\def\spprobsumind{\spprobsum_{\cind}}

\def\corrmat{\Vmat}
\def\fcorrmat{\corrmat_{f}}
\def\scorrmat{\corrmat_{g}}

\def\paramind{\param_\cind}
\def\paramindq{\paramind^{\iter}}
\def\paramindiq{\param_i^{\iter}}
\def\paramindqn{\paramind^{\itern}}
\def\parammind{\param_\cmind}

\newtheorem{theorem}{Theorem}

{%
\color{black}
}

{ \begin{list}%
	{$\bullet$}%
	{\setlength{\labelwidth}{25pt}%
	 \setlength{\leftmargin}{30pt}%
	 \setlength{\itemsep}{\parsep}}}%
{ \end{list} }

\theoremstyle{plain}
\newtheorem{lemma}{Lemma}

\begin{document}




\hspace{13.9cm}1

\begin{center}
{\LARGE Conjugate Mixture Models for Clustering Multimodal
  Data\footnote{This work was supported by the Perception-on-Purpose
    project, under EU grant FP6-IST-2004-027268.}}\\ \vspace{3mm}
\textit{Neural Computation, Volume 23, Issue 2,  February 2011, 517-557}
\end{center}
\ \\
{\bf Vasil Khalidov$^{\dagger}$,  Florence Forbes$^{\dagger}$ and  Radu Horaud$^{\dagger}$}\\
{\small $^{\dagger}$INRIA Grenoble Rh\^one-Alpes,
655, avenue de l'Europe\\
38330 Montbonnot Saint-Martin, FRANCE}\\

\ \vspace{-15mm}
\begin{flushleft}
\end{flushleft}
\thispagestyle{empty}
\markboth{}{NC instructions}
\ \vspace{-20mm}\\

\begin{abstract}
The problem of multimodal clustering arises whenever the data are
gathered with several physically different sensors. Observations
from different modalities are not necessarily aligned in the sense
there there is no obvious way to associate or to compare them in
some common space. A solution may consist in considering multiple
clustering tasks independently for each modality. The main
difficulty with such an approach is to guarantee that the unimodal
clusterings are mutually consistent. In this paper we show that
multimodal clustering can be addressed within a novel framework,
namely \textit{conjugate mixture models}. These models exploit the
explicit transformations that are often available between an
unobserved parameter space (objects) and each one of the
observation spaces (sensors). We formulate the problem as a
likelihood maximization task and we derive the associated
\textit{conjugate expectation-maximization} algorithm. The
convergence properties of the proposed algorithm are thoroughly
investigated. Several local/global optimization techniques are
proposed in order to increase its convergence speed. Two
initialization strategies are proposed and compared. A consistent
model-selection criterion is proposed. The algorithm and its
variants are tested and evaluated within the task of 3D
localization of several speakers using both auditory and visual
data.

\end{abstract}

\section{Introduction}
The unsupervised clustering of multimodal data is a key capability
whenever the goal is to group observations that are gathered using
several physically different sensors. A typical example is the
computational modelling of biological \textit{multisensory
perception}. This includes the issues of how a human detects
objects that are both seen and touched
\citep{PougetDeneveDuhamel2002,ernst02humans}, seen and heard
\citep{AnastasioPattonBelkacemBoussaid2000,King2004,King2005} or
how a human localizes one source of sensory input in a natural
environment in the presence of competing stimuli and of a variety
of noise sources~\citep{HaykinChen2005}. More generally,
\textit{multisensory fusion}~\citep{HallMcMullen2004,Mitchell2007}
is highly relevant in various other research domains, such as
target tracking~\citep{SmithSingh2006} based on radar and sonar
data~\citep{naus04simultaneous,coiras07rigid}, mobile robot
localization with laser rangefinders and
cameras~\citep{CastellanosTardos99}, robot manipulation and object
recognition using both tactile and visual
data~\citep{Allen95,JoshiSanderson99}, underwater navigation based
on active sonar and underwater
cameras~\citep{MajumderSchedingDurrantWhyte2001}, audio-visual
speaker detection~\citep{beal03graphical,
perez04data,fisher04speaker}, speech
recognition~\citep{Heckmann02,nefian02dynamic,ShaoBarker2008}, and
so forth.

When the data originates from a single object, finding the best
estimates for the object's characteristics is usually referred to as a
{\it pure fusion} task and it reduces to combining multisensor observations
in some optimal way~\citep{beal03graphical,kushal06audiovisual,SmithSingh2006}.
For example, land and underwater robots fuse data from several sensors to
build a 3D map of the ambient space irrespective of the number of objects
present in the environment~\citep{CastellanosTardos99,MajumderSchedingDurrantWhyte2001}.
The problem is much more complex when several objects are present and
when the task implies their detection, identification, and
localization. In this case one has to consider two processes simultaneously:
{\em (i)~segregation}~\citep{FisherIII2001} which assigns each observation either to an object
or to an \textit{outlier} category and
{\em (ii)~estimation} which computes the parameters of each object based on
the group of observations that were assigned to that object.
In other words, in addition to fusing observations from different sensors,
multimodal analysis requires the assignment of each observation to
one of the objects.

This observation-to-object association problem can be cast into a
probabilistic framework. Recent
multisensor data fusion methods able to handle several objects are based on
particle filters~\citep{checka04multiple,chen04realtime,gaticaperez07audiovisual}.
Notice, however, that the dimensionality of the parameter space
grows exponentially with the number of objects,
causing the number of required particles to increase dramatically
and augmenting computational costs.
A number of efficient sampling procedures were
suggested~\citep{chen04realtime,gaticaperez07audiovisual} to keep the problem
tractable. Of course this is done at the cost of loss in model generality, and
hence these attempts are strongly application-dependent.
Another drawback of such models is that they cannot provide estimates of
accuracy and importance of each modality with respect to each object.
The sampling and distribution estimation are performed in the parameter space, but
no statistics are gathered for the observation spaces.
Recently~\citep{hospedales08structure} extended the single-object model
of~\citep{beal03graphical} to multiple objects:
several single-object models are incorporated into the multiple-object model
and the number of objects is selected by an additional hidden node,
which thus accounts for model selection.
We remark that this method also suffers from exponential growth in
the number of possible models.

In the case of unimodal data, the problems of grouping observations and
of associating groups with objects can be cast into the framework of
standard data clustering which can be solved using a variety of parametric or
non-parametric techniques.
The problem of \textit{clustering multimodal
data} raises the difficult question of how to group together
observations that belong to different physical spaces with
different dimensionalities, e.g., how to group visual data
with auditory data? When the observations from two different
modalities can be {\it aligned} pairwise, a
natural solution is to consider
the Cartesian product of two unimodal spaces.
Unfortunately, such an alignment is not possible in most practical
cases. Different sensors operate at different frequency rates and
hence the number of observations gathered with one sensor can be quite
different from the number of observations gathered with another
sensor. Consequently, there is no
obvious way to align the observations pairwise. Considering all
possible pairs would result in a combinatorial blow-up and
typically create abundance of erroneous observations corresponding
to inconsistent solutions.

Alternatively, one may consider several unimodal
clusterings, provided that the relationships between a common object
space and several observation spaces can be explicitly specified.
\textit{Multimodal clustering} then results in a
number of unimodal clusterings that are jointly governed by
the same unknown parameters characterizing the object space.

%

The original contribution of this paper is to show how
the problem of \textit{clustering multimodal data} can be addressed
within the framework of mixture models~\citep{McLachlanPeel2000}.
We propose a variant of the EM algorithm~\citep{DLR,McL96} specifically
designed to estimate object-space parameters that are indirectly
observed in several sensor spaces. The convergence properties of the
proposed algorithm are thoroughly investigated and several efficient
implementations are described in detail.
The proposed model is composed of a number of modality-specific mixtures.
These mixtures are jointly governed by a set of common \textit{object-space
parameters} (which will be referred to as the \textit{tying
parameters}), thus insuring consistency between the sensory data
and the object space being sensed. This is
done using explicit transformations from the
unobserved parameter space (object space) to each of the observed spaces
(sensor spaces). Hence, the proposed model is able to deal with
observations that live in spaces with
different physical properties such as dimensionality,
space metric, sensor sampling rate, etc.
We believe that linking the object space with the sensor spaces
based on object-space-to-sensor-space transformations
has more discriminative power than existing multisensor fusion
techniques and hence performs better in terms of multiple object
identification and localization.
To the best of our knowledge, there has been no attempt to use a generative
model, such as ours, for
the task of multimodal data interpretation.

In Section~\ref{sec:conjclust} we formally introduce the concept of
\textit{conjugate mixture models}. Standard Gaussian mixture models
(GMM) are used to model the unimodal data. The parameters of these
Gaussian mixtures are governed by the object parameters
through a number of object-space-to-sensor-space
transformations (one transformation for each sensing
modality). Through the paper we will assume a very general class of
transformations, namely non-linear Lipschitz continuous functions (see
below).
In Section~\ref{sec:conjem} we cast the multimodal data clustering
problem in the framework of maximum likelihood and
we explicitly derive the expectation and
maximization steps of the associated EM algorithm. While the E-step of
the proposed algorithm is
standard, the M-step implies non-linear optimization of the
expected complete-data log-likelihood with respect to the object
parameters. We investigate efficient local and global optimization
methods.
More specifically, in Section~\ref{sec:alganal_ls} we prove that, provided that the
object-to-sensor functions as well as their first derivatives are
Lipschitz continuous, the gradient of the expected complete-data
log-likelihood is Lipschitz continuous as well. The immediate
consequence is that a number of
recently proposed optimization algorithms
specifically designed to solve Lipschitzian
global optimization problems can be used within the M-step of the
proposed algorithm~\citep{zhigljavsky08stochastic}. Several of these
algorithms combine a local maximum search procedure with an
initializing scheme to determine, at each iteration, {\it good}
initial values from which the local search should be performed. This
implies that the proposed EM algorithm has guaranteed convergence properties.
Section~\ref{sec:alganal_gs} discusses several possible local search
initialization schemes, leading to different convergence speeds.
In Section~\ref{sec:init} we propose and compare two possible strategies
to initialize the EM algorithm.
Section~\ref{sec:bic} is devoted to a consistent criterion to determine
the number of objects.
Section~\ref{sec:results} illustrates the proposed method with
the task of audiovisual object detection and localization using
binocular vision and binaural hearing.
Section~\ref{section:experimental-validation} analyses in
detail the performances of the proposed model under various practical
conditions with both simulated and real data. Finally,
Section~\ref{sec:discuss} the paper and provides directions for future
work.

\section{Mixture Models for Multimodal Data}
\label{sec:conjclust}

We consider $\cmind$ objects $n=1 \ldots N$. Each object $n$
is characterized by a parameter vector of dimension $d$, denoted by $\paramind
\in
\pars\subseteq\mathbb{R}^\pdim$. The set
$\pparams=\{\param_1,\ldots,\paramind,\ldots,\parammind\}$ corresponds
to the unknown \textit{tying parameters}. The objects are observed
with a number of physically different sensors. Although, for the sake
of clarity, we will consider two modalities, generalization is straightforward.
Therefore, the observed data consists of
two sets of observations denoted respectively by
$\fobss=\{\fobs_1,\ldots,\fobsind,\ldots,\fobsmind\}$ and
$\sobss=\{\sobs_1,\ldots,\sobsind,\ldots,\sobsmind\}$ lying in two
different observation spaces of dimensions $r$ and $p$,
$\fobsind\in\fos\subseteq\mathbb{R}^{\,\fdim}$ and
$\sobsind\in\sos\subseteq\mathbb{R}^\sdim$.

One key ingredient of our approach is that we consider the
transformations:
\begin{equation}
\label{eq:the-mappings}
\left\{
\begin{array}{c}
\fosm:\pars\to\fos \\
\sosm:\pars\to\sos
\end{array}
\right.
\end{equation}
that map $\pars$ respectively into the observation spaces  $\fos$ and
$\sos$. These transformations are defined by the physical and
geometric properties of the sensors and they are supposed to be
known. We treat the general case when both $\fosm$ and $\sosm$
are non-linear.

An assignment variable is associated with each observation, thus
indicating the object that generated the observation:
$\fasss=\{\fass_1,\ldots,\fassind,\ldots,\fassmind\}$
and $\sasss=\{\sass_1,\ldots,\sassind,$ $\ldots,\sassmind\}$.
Hence, the
segregation process is cast into a hidden variable problem. The
notation $\fassind=\cind$ (resp. $\sassind=\cind$) means that the
observation $\fobsind$ (resp. $\sobsind$) was generated by object
$\cind$. In order to account for erroneous observations, an additional
$\outind$-th fictitious object is introduced to represent an outlier
category. The notation $\fassind=\outind$ (resp.
$\sassind=\outind$) means that $\fobsind$ (resp.
$\sobsind$) is an outlier. Note that we will also use the following
standard convention: upper case letters for random variables ($\fasss$
and $\sasss$) and lower case letters for their realizations ($\fsasss$
and $\ssasss$).
The usual conditional
independence assumption leads to:
\begin{equation}
P(\fobss,\sobss|\fsasss,\ssasss)=\prod\limits_{\find=1}^{\fmind}P(\fobsind|\fsassind)\prod\limits_{\sind=1}^{\smind}P(\sobsind|\ssassind).
\label{eq:condindependency}
\end{equation}
In addition, all assignment variables are assumed to be
independent, i.e.:
\begin{equation}
P(\fsasss,\ssasss)=\prod\limits_{\find=1}^{\fmind}P(\fsassind)\prod\limits_{\sind=1}^{\smind}P(\ssassind).
\label{eq:independency}
\end{equation}
As discussed in Section~\ref{sec:discuss}, more
general cases could be considered. However, we focus on the
independent case for it captures most of the features relevant to
the conjugate clustering task and because more general dependence structures
could be reduced to the independent case via the use of
appropriate variational approximation techniques
\citep{jordan98introduction,cfp03}.

Next we define the following probability density functions, for all $n=1 \ldots N,
N+1$, for all $\fobsind \in \fos$ and for all $\sobsind \in \sos$:
\begin{eqnarray}
\label{eq:first_density}
P_{\cind}^{\fos}(\fobsind) &=&  P(\fobsind|\fassind=\cind), \\
\label{eq:second_density}
P_{\cind}^{\sos}(\sobsind)    &=&  P(\sobsind|\sassind=\cind).
\end{eqnarray}

More specifically, the likelihoods for an observation to belong to an object $n$ are
Gaussian distributions whose means $\fosm(\paramind)$ and
$\sosm(\paramind)$ correspond to the object's parameter vector $\paramind$
mapped to the observations spaces by the transformations $\fosm$ and $\sosm$:
\begin{eqnarray}
\label{eq:lhoods-1}
P_{\cind}^{\fos}(\fobsind) &=&
\normal (\fobsind;\;\fosm(\paramind), \fvarind), \\
\label{eq:lhoods-2}
P_{\cind}^{\sos}(\sobsind) &=&
\normal (\sobsind;\;\sosm(\paramind),\svarind),
\end{eqnarray}
with:
\begin{equation}
\label{eq:gdistribution}
\normal (\fobsind;\;\fosm(\paramind), \fvarind) =
\frac{1}{(2\pi)^{\fdim/2}|\fvarind |^{1/2}} \exp
\left(-\frac{1}{2}\|\fobsind - \fosm(\paramind)\|^2_{\fvarind}\right),
\end{equation}
where the notation $\|\vvect-\wvect\|^2_\fvar$ stands for the
Mahalanobis distance $(\vvect-\wvect)\tp\fvar^{-1}(\vvect-\wvect)$
and $\tp$ stands for the transpose of a matrix.
The likelihoods of outliers are taken as two uniform distributions:
\begin{eqnarray}
\label{eq:lhoods-3}
P_{\outind}^{\fos}(\fobsind)=\uni(\fobsind;\fvol), \\
\label{eq:lhoods-4}
P_{\outind}^{\sos}(\sobsind)=\uni(\sobsind;\svol),
\end{eqnarray}
where $\fvol$ and $\svol$ denote the respective support volumes.
We also define the prior probabilities $\pi = (\pi_1, \ldots, \pi_n, \ldots,\pi_\outind)$ and $\lambda
= (\lambda_1, \ldots, \lambda_n, \ldots, \lambda_\outind)$:
\begin{eqnarray}
\fprind &=& P(\fassind=\cind), \quad\forall \find=1 \ldots \fmind, \\
\label{eq:priors}
\sprind &=& P(\sassind=\cind), \quad\forall \sind=1 \ldots \smind.
\end{eqnarray}

Therefore, $\fobsind$ and $\sobsind$ are distributed according to two
($\outind$)-component mixture models, where each mixture is made of
$\cmind$ Gaussian components and one uniform component:
\begin{eqnarray}
\label{eq:first-mixture}
P(\fobsind) &=&
\sum\limits_{\cind=1}^{\cmind} \fprind
\normal(\fobsind;\;\fosm(\paramind), \fvarind) + \fprout
\uni(\fobsind;\fvol), \\
\label{eq:second-mixture}
 P(\sobsind) &=&
\sum\limits_{\cind=1}^{\cmind} \sprind
\normal(\sobsind;\;\sosm(\paramind), \svarind) + \sprout
\uni(\sobsind;\svol).
\end{eqnarray}

The log-likelihood of the observed data can then be written as:
\begin{eqnarray}
\lhood(\fobss, \sobss, \params) &=&\sum\limits_{\find=1}^\fmind
\log
\left(\sum\limits_{\cind=1}^{\cmind} \fprind
  \normal(\fobsind;\;\fosm(\paramind), \fvarind) + \fprout
  \uni(\fobsind;\fvol)\right) +\nonumber\\
 \label{eq:lhood}
&+& \sum\limits_{\sind=1}^\smind \log
\left(\sum\limits_{\cind=1}^{\cmind} \sprind
\normal(\sobsind;\;\sosm(\paramind), \svarind) + \sprout
\uni(\sobsind;\svol)\right)
\label{eq:log-lhood}
\end{eqnarray}
where:
\begin{equation}
\params=\{\fpr_1,\ldots,\fprmind,\fprout,\spr_1,\ldots,\sprmind,\sprout,\param_1,\ldots,\parammind,
\fvar_1,\ldots,\fvarmind,\svar_1,\ldots,\svarmind\}
\label{eq:parameter-vector}
\end{equation}
denotes the
set of all unknown parameters to be estimated using a
maximum likelihood principle.



\begin{figure*}[tb]
\begin{center}
     \includegraphics[width=\textwidth, type=pdf, ext=.pdf, read=.pdf]{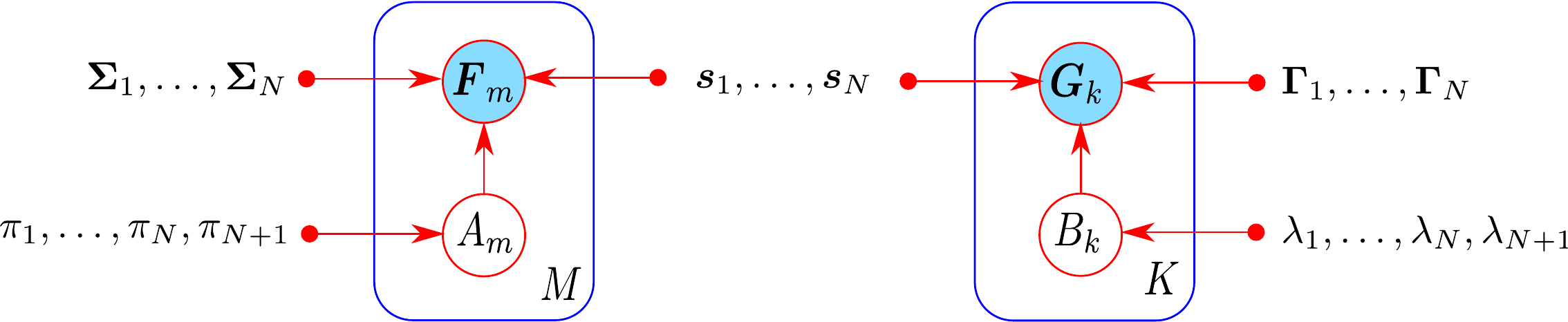}
\end{center}
\caption{\label{fig:graphmod} Graphical representation of the
conjugate mixture model. Circles denote random variables, plates
(rectangles) around them represent multiple similar nodes, their
number being given in the plates.}
\end{figure*}

The graphical representation of our conjugate mixture model is shown
in Figure~\ref{fig:graphmod}. We adopted the graphical
notation introduced in~\citep{Bishop2006} to represent similar nodes in a more
compact way: the $\fmind$ (resp. $\smind$) similar nodes are indicated with a
{\it plate}. The two sensorial modalities are linked by the \textit{tying
parameters} $\param_1, \ldots \param_\cmind$ shown in between the two plates.

\section{Generalized EM for Clustering Multimodal Data}
\label{sec:conjem}

Given the probabilistic model just described, we wish to determine the
parameter vectors associated with the objects that generated
observations in two different sensory spaces. It is well known that direct maximization of
the observed-data log-likelihood~(\ref{eq:lhood}) is
difficult to achieve. The expectation-maximization (EM)
algorithm~\citep{DLR,McL96} is a
standard approach to maximize likelihood functions of type
(\ref{eq:lhood}). It is based on the following representation, for
two arbitrary values of the parameters $\params$ and
$\tilde{\params}$:
\begin{eqnarray}
\lhood(\fobss, \sobss, \params)&=&\condexp(\params,\tilde{\params}) +
H(\params,\tilde{\params}), \\
\mbox{with}\quad \condexp(\params,\tilde{\params})&=&\mathbb{E}[\log P(\fobss,
\sobss, \fasss,\sasss; \params)\;|\;\fobss, \sobss; \tilde{\params}],\\
\mbox{and}\quad H(\params,\tilde{\params})&=&-\mathbb{E}[\log
P(\fasss,\sasss\;|\;\fobss, \sobss; \params)|\fobss, \sobss;
\tilde{\params}],
\end{eqnarray}
where the expectations are taken over the hidden variables
$\fasss$ and $\sasss$. Each iteration $q$ of EM proceeds in two steps:
\begin{itemize}
\item {\it Expectation}.
For the current values $\paramsq$ of the
parameters, compute the conditional
expectation with respect to variables $\fasss$ and $\sasss$:
\begin{equation}
\condexp(\params,\paramsq)=\sum\limits_{\fsasss\in\{1
\ldots\outind\}^\fmind} \sum\limits_{\ssasss\in\{1
\ldots\outind\}^\smind} P(\fsasss, \ssasss|\fobss,
\sobss;\;\paramsq) \; \log P(\fobss, \sobss, \fsasss,
\ssasss;\;\params)
\label{eq:condexp}
\end{equation}
\item {\it Maximization}. Update  the parameter set $\paramsq$ by maximizing
(\ref{eq:condexp}) with respect to $\params$:
\begin{equation}
\paramsqn=\argmax\limits_{\params}\condexp(\params,\paramsq)
\label{label:argmaxtheta}
\end{equation}
\end{itemize}

It is well known that the EM algorithm
increases the target function $\lhood(\fobss, \sobss,
\params)$ in (\ref{eq:lhood}), i.e., the sequence of estimates
$\{\paramsq\}_{q \in \mathbb{N}}$  satisfies $\lhood(\fobss, \sobss,
\paramsqn)\geq\lhood(\fobss, \sobss, \paramsq)$. Standard EM deals
with the parameter estimation of a single mixture model, and a closed
form solution for (\ref{label:argmaxtheta}) exists in this case.
When the maximization (\ref{label:argmaxtheta}) is difficult to
achieve, various generalizations of EM are proposed. The M step
can be relaxed by requiring just an increase rather than an
optimum. This yields Generalized EM (GEM) procedures~\citep{McL96}
(see~\citep{boyles83convergence} for a result on the convergence
of this class of algorithms). The GEM algorithm
 searches for some $\paramsqn$ such that
$\condexp(\paramsqn,\paramsq)\geq\condexp(\paramsq,\paramsq)$.
Therefore  it provides a sequence of estimates that still verifies
the non-decreasing likelihood property although the convergence
speed is likely to decrease. In the case of conjugate mixture
models, we describe in more detail the specific forms of the E
and M steps in the following sections.

\subsection{The Expectation Step}
Using (\ref{eq:independency})-(\ref{eq:priors}) the conditional
expectation (\ref{eq:condexp}) can be decomposed as:
\begin{equation}
\condexp(\params,\paramsq)=\fcondexp(\params,\paramsq)+\scondexp(\params,\paramsq),
\label{eq:expectation-decompose}
\end{equation}
with
\begin{align}
\fcondexp(\params,\paramsq)&=\sum\limits_{\find=1}^\fmind\sum\limits_{\cind=1}^\outind
\fpprobindq\log\big(\fprind P(\fobsind|\fassind=\cind;\;\params)\big),\\
\quad\scondexp(\params,\paramsq)&=\sum\limits_{\sind=1}^\smind\sum\limits_{\cind=1}^\outind
\spprobindq\log\big(\sprind
P(\sobsind|\sassind=\cind;\;\params)\big),
\end{align}
where $\fpprobindq$ and $\spprobindq$ denote the posterior
probabilities $\fpprobindq=P(\fassind=\cind | \fobsind; \paramsq)$
and $\spprobindq=P(\sassind=\cind | \sobsind; \paramsq)$. Their
expressions can be derived straightforwardly from Bayes' theorem, $\forall \cind=1\ldots\cmind$:
\begin{align}
\fpprobindq=&\frac{\fprindq\normal(\fobsind;\fosm(\paramindq),
\fvarindq)} {\sum\limits_{i=1}^\cmind
\fprindiq\normal(\fobsind;\fosm(\paramindiq), \fvarindiq) +
\fvol^{-1}\fproutq},\label{eq:posts1}\\
\quad
\spprobindq=&\frac{\sprindq\normal(\sobsind;\;\sosm(\paramindq),
\svarindq)} {\sum\limits_{i=1}^\cmind
\sprindiq\normal(\sobsind;\sosm(\paramindiq), \svarindiq) +
\svol^{-1}\sproutq}. \label{eq:posts2}
\end{align}
and $\fpproboutq=1-\sum\limits_{\cind=1}^\cmind \fpprobindq$ and $\spproboutq=1-\sum\limits_{\cind=1}^\cmind \spprobindq$.
Using (\ref{eq:lhoods-1})-(\ref{eq:lhoods-4}) the expressions above
further lead to:
\begin{align}
\fcondexp(\params,\paramsq)=&-\frac{1}{2}\sum\limits_{\find=1}^\fmind\sum\limits_{\cind=1}^\cmind
\fpprobindq\left(\|\fobsind-\fosm(\paramind)\|^2_{\fvarind}+\log((2\pi)^\fdim|\fvarind|\fprind^{-2})\right)-\nonumber\\
&-\frac{1}{2}\sum\limits_{\find=1}^\fmind\fpproboutq \log(\fvol^2\fprout^{-2}),
\label{eq:condexpgauss1}\\
\scondexp(\params,\paramsq)=&-\frac{1}{2}\sum\limits_{\sind=1}^\smind\sum\limits_{\cind=1}^\cmind
\spprobindq\left(\|\sobsind-\sosm(\paramind)\|^2_{\svarind}+\log((2\pi)^\sdim|\svarind|\sprind^{-2})\right)-\nonumber\\
&-\frac{1}{2}\sum\limits_{\sind=1}^\smind\spproboutq \log(\svol^2\sprout^{-2}).
\label{eq:condexpgauss2}
\end{align}

\subsection{The Maximization Step}

In order to carry out the maximization (\ref{label:argmaxtheta}) of the
conditional expectation (\ref{eq:condexp}), its
derivatives with respect to the model parameters are set to zero.
This leads to the standard update expressions for priors, more
specifically $\forall \cind=1,\ldots,\outind$:
\begin{eqnarray}
\label{eq:msteppriors}
\fprindqn &=&
\frac{1}{\fmind}\sum\limits_{\find=1}^{\fmind}\fpprobindq, \\
\label{eq:msteppriors-more}
\sprindqn &=&
\frac{1}{\smind}\sum\limits_{\sind=1}^{\smind}\spprobindq.
\end{eqnarray}

The covariance matrices are governed by the tying parameters
$\paramindqn \in \pars$ through the functions $\fosm$ and $\sosm$, $\forall
\cind=1,\ldots,\cmind$:
\begin{eqnarray}
\label{eq:optvars}
\fvarindqn (\paramindqn) &=& \frac{1}{\sum\limits_{\find=1}^{\fmind}\fpprobindq}
\sum\limits_{\find=1}^{\fmind}\fpprobindq
(\fobsind-\fosm(\paramindqn))(\fobsind-\fosm(\paramindqn))\tp, \\
\label{eq:optvars-more}
\svarindqn (\paramindqn) &=& \frac{1}{\sum\limits_{\sind=1}^{\smind}\spprobindq}
\sum\limits_{\sind=1}^{\smind}\spprobindq
(\sobsind-\sosm(\paramindqn))(\sobsind-\sosm(\paramindqn))\tp.
\end{eqnarray}
For every $\cind=1,\ldots,\cmind$, $\paramindqn$ is the parameter vector
such that:
\begin{equation}
\paramindqn = \argmax\limits_{\param} Q_{\cind}^{\iter}(\param),
\end{equation}
where
\begin{eqnarray}
Q_{\cind}^{\iter}(\param)=
&-&\sum\limits_{\find=1}^{\fmind}\fpprobindq
(\|\fobsind-\fosm(\param)\|^2_{\fvar_n(\param)} +
\log|\fvar_n(\param)|)-\nonumber\\
 \label{eq:mstepmax}
&-&\sum\limits_{\sind=1}^{\smind}\spprobindq
(\|\sobsind-\sosm(\param)\|^2_{\svar_n(\param)} +
\log|\svar_n(\param)|).
\end{eqnarray}
We stress that the covariances $\fvar_n(\param)$ and $\svar_n(\param)$ in
(\ref{eq:optvars}) and (\ref{eq:optvars-more}) are considered as  functions of
$\param \in \pars$.  Hence, at each iteration of the algorithm, the overall
update of the tying parameters can be split into $\cmind$ identical
optimization tasks of the form~(\ref{eq:mstepmax}).
These tasks  can be solved in parallel. In general,
$\fosm$ and $\sosm$ are non-linear transformations and hence
there is no simple closed-form expression for the estimation of the tying parameters.


\subsection{Generalized EM for Conjugate Mixture Models}
\label{sec:indep:em:gem}

The initial parameters selection of the proposed EM algorithm for conjugate mixture models uses the procedure
{\it Initialize} that is given in Section~\ref{sec:init}.
The maximization step uses two procedures, referred to as {\it Choose} and {\it Local Search} which are explained in
detail in Sections~\ref{sec:alganal_ls} and~\ref{sec:alganal_gs}.
To determine the number of objects we define the procedure
{\it Select} that is derived in Section~\ref{sec:bic}. The overall EM procedure is outlined below:
\begin{enumerate}
\item Apply procedure {\it Initialize} to initialize the parameter vector:\\
$\params^{(0)}=\{
\fpr_1^{(0)},\ldots,\fpr_\outind^{(0)}, \spr_1^{(0)},\ldots,\spr_\outind^{(0)},
\param_1^{(0)},\ldots,\param_\cmind^{(0)},
\fvar_1^{(0)},\ldots,\fvar_\cmind^{(0)},\svar_1^{(0)},\ldots,\svar_\cmind^{(0)}\}$;
\item \textit{E step}: compute $\condexp(\params,\paramsq)$ using equations (\ref{eq:posts1}) to (\ref{eq:condexpgauss2});
\item \textit{M step}: estimate $\paramsqn$ using the following sub-steps:
    \begin{enumerate}
    \item \textit{The priors.} Compute
      $\fpr_1^{\itern},\ldots,\fpr_\outind^{\itern}$ and
      $\spr_1^{\itern},\ldots,\spr_\outind^{\itern}$
      using~(\ref{eq:msteppriors}) and (\ref{eq:msteppriors-more});
    \item \textit{The tying parameters.} For each $\cind=1\ldots\cmind$:
      \begin{itemize}
        \item Apply procedure {\it Choose} to determine an
        initial value, denoted by  $\tilde{\param}_n^{(0)}$, as
        proposed in
    Section~\ref{sec:alganal_gs};
    \item Apply procedure {\it Local Search} to each
      $Q_{\cind}^{\iter}(\param)$ as defined in (\ref{eq:mstepmax}) starting from
    $\tilde{\param}_n^{(0)}$ and set the result to
    $\param_n^{\itern}$ using the eq.~(\ref{eq:mstepmeans}) specified below;
     \end{itemize}
    \item \textit{The covariance matrices.} For every
      $\cind=1\ldots\cmind$, use (\ref{eq:optvars}) and (\ref{eq:optvars-more}) to compute
      $\fvar_n^{\itern}$ and $\svar_n^{\itern}$;
    \end{enumerate}
\item \textit{Check for convergence}: Terminate, otherwise go to Step 2;
\item  Apply procedure \textit{Select}, use~(\ref{eq:bic-criterion}) specified below to determine the best $\cmind$;
\end{enumerate}

This algorithm uses the following procedures:

\begin{itemize}

\item {\it Initialize}: this procedure aims at providing the
initial parameter values $\params^{(0)}$. Its performance has
a strong impact on the time required for the algorithm to converge.
In Section~\ref{sec:init} we propose different initialization strategies based on
single-space cluster detection.

\item {\it Select}: this procedure applies the BIC-like criterion
to determine the number of objects $\cmind$.
In Section~\ref{sec:bic} propose the consistent criterion
for the case of conjugate mixture models.

\item {\it Choose}: the goal of this procedure
is to provide at each M step initial values
$\tilde{\param}_1^{(0)},\ldots,\tilde{\param}_\cmind^{(0)}$ which
are likely to be close to the global maxima of the functions
$Q_{\cind}^{\iter}(\param)$ in (\ref{eq:mstepmax}).  The exact
form of this procedure is important to ensure the ability of the
subsequent {\it Local Search} procedure to find these global
maxima. We will use results on global search algorithms~\citep{zhigljavsky08stochastic}
and propose different variants in Section~\ref{sec:alganal_gs}.

\item {\it Local Search}: an important
requirement of this procedure is that it finds a local maximum of the
$Q_{\cind}^{\iter}(\param)$'s starting from any arbitrary point
in $\pars$. In this work, we will consider procedures that consist
in iterating a local update of the
form ($\nu$ is the iteration index):
\begin{equation}
\tilde{\paramind}^{(\nu+1)}=\tilde{\paramind}^{(\nu)}+\Hmat_{\cind}^{(q,
\nu) }\nablavect Q_{\cind}^{\iter}(\tilde{\paramind}^{(\nu)}),
\label{eq:mstepmeans}
\end{equation}
with $\Hmat_{\cind}^{(q, \nu)}$ being a positive definite matrix that
may vary with $\nu$. When the gradient $\nablavect
Q_{\cind}^{\iter}(\param)$ is Lipschitz continuous with some
constant $L_{\cind}^{\iter}$, an appropriate choice  that
guarantees the increase of
$Q_{\cind}^{\iter}(\tilde{\param}^{(\nu)})$ at each iteration
$\nu$, is to choose $\Hmat_{\cind}^{(q, \nu)}$ such that it verifies
$\|\Hmat_{\cind}^{\,(q, \nu)}\|\leq 2/L_{\cind}^{\iter}$.

Different choices for $\Hmat_{\cind}^{(q, \nu)}$ are possible and they correspond
to different optimization methods that belong, in general, to the variable metric class.
For example $\Hmat_{\cind}^{(q, \nu)}=\frac{2}{L_{\cind}^{\iter}}\Imat$
leads to gradient ascent, while taking $\Hmat_{\cind}^{(q, \nu)}$ as a
scaled inverse of the Hessian matrix
would lead to a Newton-Raphson optimization step. Other
possibilities include Levenberg-Marquardt and quasi-Newton
methods.
\end{itemize}


\section{Analysis of the {\it Local Search} Procedure}
\label{sec:alganal_ls}

Each instance of~(\ref{eq:mstepmax}) for $\cind=1,\ldots,\cmind$
can be solved independently. In this section
we focus on providing a set of conditions under which each iteration
of our algorithm guarantees that the objective function
 $Q_{\cind}^{\iter}(\param)$ in (\ref{eq:mstepmax}) is increased.
We start by
rewriting~(\ref{eq:mstepmax}) more conveniently in order to perform
the optimization with respect to $\param \in \pars$. To simplify
the notation, the iteration index $\iterind$ is sometimes omitted.
We simply write $Q_\cind(\param)$ for $Q_\cind^{\iter}(\param)$.

Let $\fpprobsumind=\sum_{\find=1}^\fmind \fpprobindq$
and $\spprobsumind=\sum_{\sind=1}^\smind \spprobindq$
denote the average object weights in each one of the two modalities.
We introduce
$\fpprob_{\cind} =\fpprobsumind^{-1} (\fpprob^{(q)}_{1\cind}, \ldots, \fpprob^{(q)}_{\fmind\cind})$
and
$\spprob_{\cind} =\spprobsumind^{-1} (\spprob^{(q)}_{1\cind}, \ldots, \spprob^{(q)}_{\smind\cind})$
the discrete probability distributions obtained by normalizing the object weights. We denote by
$\Fobs$ and $\Sobs$ the random variables that take their values in the discrete sets
$\{\fobs_1,\ldots,\fobsind,\ldots,\fobsmind\}$ and
$\{\sobs_1,\ldots,\sobsind,\ldots,\sobsmind\}$. It follows that
the expressions for the optimal variances (\ref{eq:optvars}) and
(\ref{eq:optvars-more}) as
functions of $\param$, can be rewritten as:
\begin{align}
\label{eq:optvarsbar0}
\fvarindqn(\param) =& \EE_{\fpprob_{\cind}}[
\big(\Fobs-\fosm(\param)\big)
\big(\Fobs-\fosm(\param)\big)\tp], \\
\label{eq:optvarsbar0-more}
\svarindqn(\param) = & \EE_{\spprob_{\cind}}[
\big(\Sobs-\sosm(\param)\big)
\big(\Sobs-\sosm(\param)\big)\tp],
\end{align}
where $\EE_{\fpprob_{\cind}}$ and $\EE_{\spprob_{\cind}}$ denote the expectations
with respect to the distributions $\fpprob_{\cind}$ and $\spprob_{\cind}$. Using some
standard projection formula, it follows that the covariances are:
\begin{align}
\label{eq:optvarsbar}
\fvarindqn(\param) = & \fcorrmat+ \vvect_f \vvect_f\tp, \\
\label{eq:optvarsbar-more}
\svar^{\itern}(\param) = & \scorrmat+ \vvect_g \vvect_g\tp,
\end{align}
where $\fcorrmat$ and $\scorrmat$ are the covariance matrices of
$\Fobs$ and $\Sobs$ respectively under distributions $\fpprob_{\cind}$
and $\spprob_{\cind}$, and $\vvect_f$ and $\vvect_g$ are vectors defined by:
\begin{eqnarray}
\vvect_f = \EE_{\fpprob_{\cind}}[\Fobs]-\fosm(\param), \\
\vvect_g = \EE_{\spprob_{\cind}}[\Sobs]-\sosm(\param).
\end{eqnarray}
For convenience we omit the index $\cind$ for
$\fcorrmat$, $\scorrmat$, $\vvect_f$ and $\vvect_g$.
Let $\fobsavgind=\EE_{\fpprob_{\cind}}[\Fobs]$ and
$\sobsavgind=\EE_{\spprob_{\cind}}[\Sobs]$. This yields:
\begin{align}
\fobsavgind&=\fpprobsumind^{-1}\sum\limits_{\find=1}^\fmind\fpprobindq\fobsind, \\
\sobsavgind&=\spprobsumind^{-1}\sum\limits_{\sind=1}^\smind\spprobindq\sobsind, \\
\fcorrmat&=\fpprobsumind^{-1}\sum\limits_{\find=1}^\fmind\fpprobindq\fobsind\fobsind\tp - \fobsavgind\fobsavgind\tp,\\
\scorrmat&=\spprobsumind^{-1}\sum\limits_{\sind=1}^\smind\spprobindq\sobsind\sobsind\tp - \sobsavgind\sobsavgind\tp.
\end{align}
Next we derive  a simplified expression for
$Q_{\cind}(\param)$ in (\ref{eq:mstepmax}) in order to investigate its
properties. Notice that one can write (\ref{eq:mstepmax}) as the sum
$Q_{\cind}(\param)=Q_{\cind,\fosm}(\param)+Q_{\cind,\sosm}(\param)$, with:
\begin{equation}
\label{eq:firstpart-of-Q}
Q_{\cind,\fosm}(\param)=-\sum\limits_{\find=1}^{\fmind}\fpprobindq
(\|\fobsind-\fosm(\param)\|^2_{\fvarindqn(\param)} + \log|\fvarindqn(\param)|),
\end{equation}
and a similar expression for
$Q_{\cind,\sosm}(\param)$. Eq.~(\ref{eq:firstpart-of-Q}) can be written:
\begin{equation}
Q_{\cind,\fosm}(\param)=-\fpprobsumind (\EE_{\fpprob_{\cind}}[
(\Fobs-\fosm(\param))\tp \fvarindqn (\param)^{-1}(\Fobs-\fosm(\param))]
+ \log|\fvarindqn(\param)|). \label{QF}
\end{equation}
The first term of (\ref{QF}) can be further divided into two terms:
\begin{align}
&\EE_{\fpprob_{\cind}}[(\Fobs-\fosm(\param))\tp
\fvarindqn(\param)^{-1}(\Fobs-\fosm(\param))] = \nonumber\\
=&\EE_{\fpprob_{\cind}}[(\Fobs-\fobsavgind)\tp
\fvarindqn(\param)^{-1}(\Fobs-\fobsavgind)] + \vvect_f\tp
{\fvarindqn(\param)}^{-1} \vvect_f.
\label{eq:QF-first}
\end{align}
The Sherman-Morrison formula applied to (\ref{eq:optvarsbar})
leads to
\begin{equation}
\fvarindqn(\param)^{-1}=\fcorrmat^{-1}-\fcorrmat^{-1}\vvect_f\vvect_f\tp\fcorrmat^{-1}/(1+D_{\cind,\fosm}(\param)),
\end{equation}
with:
\begin{equation}
D_{\cind,\fosm}(\param)=\|\fosm(\param)-\fobsavgind\|_{\fcorrmat}^2.
\end{equation}
It follows that (\ref{eq:QF-first}) can be written as the sum of:
\begin{equation}
\EE_{\fpprob_{\cind}}[(\Fobs-\fobsavgind)\tp
\fvarindqn(\param)^{-1}(\Fobs-\fobsavgind)] = C_f -
\frac{D_{\cind,\fosm}(\param)}{1+ D_{\cind,\fosm}(\param)},
\end{equation}
and of
\begin{equation}
\vvect_f\tp
{\fvarindqn(\param)}^{-1} \vvect_f =
\frac{D_{\cind,\fosm}(\param)}{1+ D_{\cind,\fosm}(\param)}.
\end{equation}
Hence the first term of~(\ref{QF}), namely~(\ref{eq:QF-first}) is equal to $C_f$ which
is constant with respect to $\param$.
Moreover, applying the matrix determinant
lemma to the second term of~(\ref{QF}) we successively obtain:
\begin{align}
\log|\fvarindqn(\param)|&=\log|\fcorrmat + \vvect_f\vvect_f\tp| =
\log|\fcorrmat| + \log(1 + \vvect_f\tp\fcorrmat^{-1}\vvect_f) = \nonumber\\
&=\log|\fcorrmat| + \log(1 + D_{\cind,\fosm}(\param)).
\end{align}
It follows that there is only one term depending on $\param$ in~(\ref{QF}):
\begin{equation}
Q_{\cind,\fosm}(\param)=-\fpprobsumind \left( C_f + \log|\fcorrmat| + \log(1 + D_{\cind,\fosm}(\param)) \right).
\end{equation}
Repeating the same derivation for the second sensorial modality
we obtain the following equivalent form of
(\ref{eq:mstepmax}):
\begin{equation}
Q_\cind(\param)=-\fpprobsumind \log(1 + D_{\cind,\fosm}(\param))- \spprobsumind
\log(1 + D_{\cind,\sosm}(\param)) + C, \label{eq:mstepmaxequiv}
\end{equation}
where $C$ is some constant not depending on $\param$.

Using this form of $Q_\cind(\param)$, we can now investigate the
properties of its gradient $\nablavect Q_\cind(\param)$.
It appears that under some regularity assumptions on $\fosm$ and $\sosm$,
the gradient $\nablavect Q_\cind(\param)$ is bounded and Lipschitz continuous.
The corresponding theorem is formulated and proved.
First we establish as a lemma some technical results, required to prove the theorem.
In what follows, for any matrix $\Vmat$, the matrix norm used is the operator norm
$\|\Vmat\|=\sup\limits_{\|\vvect\|=1}\|\Vmat \vvect\|$.
For simplicity, we further omit the index $\cind$.

\begin{lemma} \label{lemma:matrixineq}
Let $\Vmat$ be a symmetric positive definite matrix. Then the
function
$$\varphi(\vvect)=\|\Vmat\vvect\|/(1+\vvect\tp\Vmat\vvect)$$  is
bounded by
$\varphi(\vvect)\leq C_{\varphi}(\Vmat)$
with $C_{\varphi}(\Vmat)= \sqrt{\|\Vmat\|}/2$
and is Lipschitz continuous:
$$\forall\vvect,\tilde{\vvect}\quad\|\varphi(\vvect)-\varphi(\tilde{\vvect})\|\leq
L_{\varphi}(\Vmat)\|\vvect-\tilde{\vvect}\|,$$
where $L_{\varphi}(\Vmat)=\|\Vmat\|(1 +\mu(\Vmat)/2)$ is the Lipschitz
constant and
$\mu(\Vmat)=\|\Vmat\|\|\Vmat^{-1}\|$ is the condition number of $\Vmat$.
\end{lemma}
\begin{proof}
We start by introducing $\wvect=\Vmat\vvect$ so that
$\varphi(\vvect)=\tilde{\varphi}(\wvect)=\|\wvect\|/(1+\wvect\tp\Vmat^{-1}\wvect).$
As soon as
$\wvect\tp\Vmat^{-1}\wvect\geq\lambda_{\mathrm{min}}\|\wvect\|^2$
(where we denoted by $\lambda_{\mathrm{min}}$ the smallest
eigenvalue of $\Vmat^{-1}$, so that in fact
$\lambda_{\mathrm{min}}=\|\Vmat\|^{-1}$), to find the maximum of
$\tilde{\varphi}(\wvect)$ we should maximize the expression $t /
(1 + \lambda_{\mathrm{min}}t^2)$ for $t=\|\wvect\|\geq0$. It is
reached at the point $t^*=\lambda_{\mathrm{min}}^{-1/2}$.
Substituting this value into the original expressions gives
$\varphi(\vvect)\leq \sqrt{\|\Vmat\|}/2$.

To compute the Lipschitz constant $L_{\varphi}$ we consider the
derivative:
\begin{align}
\|\nablavect\tilde{\varphi}'(\wvect)\|&=\frac{\left\|(1 +
\wvect\tp\Vmat^{-1}\wvect)\wvect -
2\|\wvect\|^2\Vmat^{-1}\wvect\right\|}{\|\wvect\|(1+\wvect\tp\Vmat^{-1}\wvect)^2}\leq
1 +
\frac{2\|\Vmat^{-1}\|\|\wvect\|^2}{(1+\wvect\tp\Vmat^{-1}\wvect)^2},\nonumber
\end{align}
from where we find that $\|\nablavect
\tilde{\varphi}'(\wvect)\|\leq 1 + \mu(\Vmat)/2$, and so
$L_{\varphi}=\|\Vmat\|(1 + \mu(\Vmat)/2).\quad\blacksquare$
\end{proof}

This lemma yields the following main result for the gradient $\nablavect Q$:

\begin{theorem} \label{lemma:lipschitz}
Assume functions $\fosm$ and $\sosm$ and their derivatives
$\fosm'$ and $\sosm'$ are Lipschitz continuous with constants
$L_{\fosm}$, $L_{\sosm}$, $L'_{\fosm}$ and $L'_{\sosm}$
respectively. Then the gradient $\nablavect Q$ is bounded and
Lipschitz continuous with some constant $L$.
\end{theorem}
\begin{proof}
From (\ref{eq:mstepmaxequiv}) the gradient $\nablavect Q$  can be written as:
\begin{eqnarray}
\nablavect Q(\param)&=&\nablavect Q_{\fosm}(\param) + \nablavect
Q_{\sosm}(\param)= \nonumber \\
 \label{eq:optimfuncgrad}
&=& \frac{2\fpprobsum
{\fosm'}\tp(\param)\fcorrmat^{-1}(\fobsavg-\fosm(\param)) }{1 +
D_{\fosm}(\param)} +\frac{2\spprobsum
{\sosm'}\tp(\param)\scorrmat^{-1}(\sobsavg-\sosm(\param)) }{1 +
D_{\sosm}(\param)}.
\end{eqnarray}
It follows from Lemma~\ref{lemma:matrixineq} that
\scalebox{0.9}{$\left\|\nablavect Q_{\fosm}(\param)\right\| \leq
2L_{\fosm}\fpprobsum C_{\varphi}(\fcorrmat^{-1})\;$} and
\scalebox{0.9}{$\;\left\|\nablavect Q_{\sosm}(\param)\right\| \leq
2L_{\sosm}\spprobsum C_{\varphi}(\scorrmat^{-1})$}. The norm of the
gradient is then bounded by:
\begin{equation}
\left\|\nablavect Q(\param)\right\| \leq 2L_{\fosm}\fpprobsum
C_{\varphi}(\fcorrmat^{-1}) + 2L_{\sosm}\spprobsum
C_{\varphi}(\scorrmat^{-1}). \label{eq:optimfuncgradbound}
\end{equation}
Considering the norm $\left\|\nablavect Q_{\fosm}(\param) -
\nablavect Q_{\fosm}(\tilde{\param})\right\|$, we introduce
$\vvect_1=\fobsavg - \fosm(\param)$ and $\vvect_2=\fobsavg -
\fosm(\tilde{\param})$. Then we have:
\begin{align}
\left\|\nablavect Q_{\fosm}(\param) - \nablavect
Q_{\fosm}(\tilde{\param})\right\|
&\leq 2 \fpprobsum\left(\left\|\frac{(\fosm'(\param)-\fosm'(\tilde{\param}))\tp\fcorrmat^{-1}\vvect_1}
{1+\|\vvect_1\|^2_{\fcorrmat}}\right\|\right.+\nonumber\\
&+\left.\left\|\frac{{\fosm'}\tp(\tilde{\param})\fcorrmat^{-1}\vvect_2}{1+\|\vvect_2\|^2_{\fcorrmat}}
-\frac{{\fosm'}\tp(\tilde{\param})\fcorrmat^{-1}\vvect_1}{1+\|\vvect_1\|^2_{\fcorrmat}}
\right\| \right).
\end{align}
Using Lemma \ref{lemma:matrixineq} with $\fcorrmat^{-1}$ we have:
$$\left\|\nablavect Q_{\fosm}(\param) - \nablavect Q_{\fosm}(\tilde{\param})\right\| \leq
2\fpprobsum\big(L'_{\fosm} C_{\varphi}(\fcorrmat^{-1}) +
L^2_{\fosm} L_{\varphi}(\fcorrmat^{-1})\big)
\|\param-\tilde{\param}\|.$$ The same derivations can be performed
for $\nablavect Q_{\sosm}(\param)$, so that finally we get:
\begin{equation}
\left\|\nablavect Q_{\sosm}(\param) - \nablavect
Q_{\sosm}(\tilde{\param})\right\|\leq L \|\param-\tilde{\param}\|,
\end{equation}
where the Lipschitz constant is given by:
\begin{equation}
L=2\fpprobsum \left(L'_{\fosm} C_{\varphi}(\fcorrmat^{-1}) + L^2_{\fosm} L_{\varphi}(\fcorrmat^{-1})\right)+
2\spprobsum\left(L'_{\sosm} C_{\varphi}(\scorrmat^{-1}) + L^2_{\sosm} L_{\varphi}(\scorrmat^{-1})\right).
\label{eq:optimfunclipconst}
\end{equation}
$\blacksquare$
\end{proof}

To actually construct the non-decreasing sequence in
(\ref{eq:mstepmeans}), we make use of the following fundamental
result on variable metric gradient ascent algorithms.
\begin{theorem}[\citep{polyak87introduction}]\label{theorem:nondec} Let the function $Q:\mathbb{R}^\pdim\to\mathbb{R}$ be differentiable on $\mathbb{R}^\pdim$ and its gradient
$\nablavect Q$ be Lipschitz continuous with constant $L$. Let the
matrix $\Hmat$ be positive definite, such that
$\|\Hmat\|\leq\frac{2}{L}$. Then the sequence
$Q(\tilde{\param}^{(\nu)})$, defined by $\tilde{\param}^{(\nu+1)}
= \tilde{\param}^{(\nu)} + \Hmat \nablavect
Q(\tilde{\param}^{(\nu)})$ is non-decreasing.
\end{theorem}

This result shows that for any functions $\fosm$ and $\sosm$ that
verify the conditions of Theorem~\ref{lemma:lipschitz}, using
(\ref{eq:mstepmeans}) with $\Hmat=\frac{2}{L}{\bf I}$, we are able
to construct a non-decreasing sequence and an appropriate \textit{Local
Search} procedure. Notice however, that its guaranteed theoretical convergence speed is
linear. It can be improved in several ways.

First,  the
optimization {\em direction} can be adjusted. For certain problems,
the matrix $\Hmat$ can be chosen as in variable metric algorithms,
such as Newton-Raphson method, quasi-Newton methods or
Levenberg-Marquardt method, provided that it satisfies the
conditions of Theorem~\ref{theorem:nondec}. Second, the
optimization {\em step size} can be increased based on local
properties of the target function. For example, at iteration
$\nu$, if when considering the functions $\fosm$ and $\sosm$ on
some restricted domain $\pars^{(\nu)}$ there exist smaller local
Lipschitz constants $L^{(\nu)}_{\fosm}$, $L^{(\nu)}_{\sosm}$,
$L^{\prime (\nu)}_{\fosm}$ and $L^{\prime (\nu)}_{\sosm}$, $\Hmat$
can be set to $\Hmat=\frac{2}{L^{(\nu)}}{\bf I}$ with $L^{(\nu)}$
smaller than $L$. It follows that
$\|\tilde{\param}^{(\nu+1)}-\tilde{\param}^{(\nu)}\| \leq
\frac{2}{L^{(\nu)}}\|\nablavect Q(\tilde{\param}^{(\nu)})\|$,
which means that one can take the local constants,
$L^{(\nu)}_{\fosm}$, $L^{(\nu)}_{\sosm}$, $L^{\prime
(\nu)}_{\fosm}$ and $L^{\prime (\nu)}_{\sosm}$ if they are valid
in the ball $\mathbb{B}_{\rho^{(\nu)}}(\tilde{\param}^{(\nu)})$
with
\begin{equation}
\rho^{(\nu)}=\frac{2}{L^{(\nu)}} \left(
2L^{(\nu)}_{\fosm}\fpprobsum C_{\varphi}(\fcorrmat^{-1}) +
2L^{(\nu)}_{\sosm}\spprobsum C_{\varphi}(\scorrmat^{-1}) \right).
\label{eq:radius}
\end{equation}

\section{Global Search and the {\it Choose} Procedure}
\label{sec:alganal_gs}

Theorem~\ref{lemma:lipschitz} allows us to use the improved global
random search techniques for Lipschitz continuous
functions~\citep{zhigljavsky91theory}. These algorithms are known
to converge, in the sense that generated point sequences fall
infinitely often into an arbitrarily small neighbourhood of the
optimal points set. For more details and convergence conditions
see Theorem~3.2.1 and the discussion that follows
in~\citep{zhigljavsky91theory}. A proper choice of the initial
value $\tilde{\param}^{(0)}$ not only guarantees to find the
global maximum, but can also be used to increase the convergence
speed. A basic strategy is to draw samples in $\pars$, according
to some sequence of distributions over $\pars$, that verifies the
convergence conditions of global random search methods. However,
the speed of convergence of such an algorithm is quite low.

Global random search methods can also be significantly improved by
taking into account some specificities of the target function.
Indeed, in our case, function~(\ref{eq:mstepmaxequiv}) is made of
two parts for which the optimal points are known and are
respectively  $\fobsavg$ and $\sobsavg$. If there exists
$\tilde{\param}^{(0)}$ such that
$\tilde{\param}^{(0)}\in\fosm^{-1}(\fobsavg)\cap\sosm^{-1}(\sobsavg)$,
then it is the global maximum and the M step solution is found.
Otherwise, one can sample $\pars$ in the vicinity of the set
$\fosm^{-1}(\fobsavg)\cup\sosm^{-1}(\sobsavg)$ to focus on a
subspace that is likely to contain the global maximum. This set
is, generally speaking, a union of two manifolds. For sampling
methods on manifolds we refer to~\citep{zhigljavsky91theory}. An
illustration of this technique is given in
Section~\ref{sec:results}.

Another possibility is to use  a heuristic that
function~(\ref{eq:mstepmaxequiv}) does not change much after one
iteration of the EM algorithm. Then, the initial point
$\tilde{\param}^{(0)}$ for the current iteration can be set to the
optimal value computed at the previous iteration. However, in
general, this simple strategy does not yield the global maximum,
as can be seen from the results in Section~\ref{sec:results1}.

\section{Algorithm Initialization and the {\it Initialize} Procedure}
\label{sec:init}

In this section we focus on the problem of selecting the initial values
$\params^{(0)}$ for the model parameters. As it is often the case with
iterative optimization algorithms, the closer $\params^{(0)}$ is to the
optimal parameter values, the less time the algorithm would require to
converge. Within the framework of conjugate mixture models we formulate
two initialization strategies, namely the {\it Observation Space
  Candidates} (OSC) strategy and the
{\it Parameter Space Candidates} (PSC) strategy, that attempt to find a good
initialization.

The {\it Observation Space Candidates} strategy consists in searching for
cluster centers in single modality spaces $\fos$ and $\sos$ to further
map them into the parameter space $\pars$, and select the best candidates.
More specifically, we randomly select an observation $\fobsind$ (or
$\sobsind$) and run the mean shift algorithm~\citep{comaniciu02meanshift}
in the corresponding space to find local modes of the distribution,
which are called {\it candidates}.
The sets of candidate points $\{\hat{\fobs}_i\}_{i\in I}$ and $\{\hat{\sobs}_j\}_{j\in J}$
are further rarefied, that is if $\|\hat{\fobs}_{i_1} - \hat{\fobs}_{i_2}\|\leq \varepsilon_{\mathrm{f}}$
for some $i_1 \neq i_2$ and for some threshold $\varepsilon>0$, we eliminate one of these points.
These rarefied sets are then mapped to $\pars$.
If one of the observation space mappings, for example $\fosm$, is non-injective,
for each $\hat{\fobs}_i$ we need to select a point $\param_i\in\fosm^{-1}(\hat{\fobs}_i)$
that is the best in some sense. We consider observations density
in the other observation spaces around an image of $\param_i$
as the optimality measure of $\param_i$. This can be estimated through
calculation of the k-th nearest neighbour distance (k-NN) in the corresponding observation space.
The final step is to choose $\cmind$ points out of these candidates to initialize
the cluster centers $\{\param_1,\ldots,\param_\cmind\}$,
so that the inter-cluster distances are maximized. This can be done using,
for example, hierarchical clustering.
The variances $\fvar_1,\ldots,\fvar_\cmind$ and $\svar_1,\ldots,\svar_\cmind$
are then calculated by standard empirical variance formulas based on
observations, that are closest to the corresponding class center.
The priors $\fpr_1,\ldots,\fpr_{\outind}$ and $\spr_1,\ldots,\spr_{\outind}$ are set to be equal.

The {\it Parameter Space Candidates} strategy consists in mapping all the observations
to the parameter space $\pars$, and performing subsequent clustering in that space.
More specifically, for every observation $\fobsind$ and $\sobsind$ we find an optimal
point from the corresponding preimage $\fosm^{-1}(\fobsind)$ and $\sosm^{-1}(\sobsind)$.
The optimality condition is the same as in the previous strategy,
that is we compare the local observation densities using k-NN distances.
Then one proceeds with selecting local modes in space $\pars$ using the mean-shift algorithm,
and initializing $\cmind$ cluster centers $\{\param_1,\ldots,\param_\cmind\}$
from all the candidates thus calculated. The estimation of variances and priors
is exactly the same as in the previous strategy.

The second strategy proved to be better when performing simulations
(see Section~\ref{section:experimental-validation}).
This can be explained by possible errors
in finding the preimage of an observation space point in the parameter space.
Thus mapping a rarefied set of candidates to the parameter space
is less likely to make a good guess in that space
than mapping all the observations and finding the candidates directly in the parameter space.

\section{Estimating the Number of Components and the {\it Select} Procedure}
\label{sec:bic}

To choose the $\cmind$ that best corresponds to the data, we perform model
selection based on a criterion that resembles
the BIC criterion~\citep{Schwarz78}.
We consider the score function of the form
\begin{equation}
\mathrm{BIC}_\cmind=-2\lhood(\fobss, \sobss, \hat{\params}_\cmind) + D_\cmind \log(\fmind + \smind),
\label{eq:bic-criterion}
\end{equation}
where $\hat{\params}_\cmind$ is the ML estimate obtained by the proposed EM algorithm,
$\lhood(\fobss, \sobss, \params)$ is given by~(\ref{eq:log-lhood}) and
$D_\cmind = \cmind \left(\pdim + 2 + \frac{1}{2}(\fdim^2+\sdim^2+\fdim+\sdim)\right)$
is the dimensionality of the model.

As in the case of (non-conjugate) Gaussian mixture models,
we cannot derive the criterion from the Laplace approximation
of the probability $P(\fobss, \sobss | \cmind = \cmind_0)$
because of the Hessian matrix of $\lhood(\fobss, \sobss, \params)$
that is not necessarily positive definite~\citep{aitkin85estimation,quinn87anote}.
Nevertheless, we can use the same arguments as those used in~\citep{keribin00consistent}
for Gaussian mixture models to show that the criterion is consistent, i.e.
if $\cmind_*$ is the number of components in the real model that generated $\fobss$ and $\sobss$, then
\begin{equation}
\cmind_\mathrm{BIC}\to \cmind_* \quad\mbox{a.s.}, \quad\mbox{when}\quad\fmind,\smind\to\infty,
\label{eq:bic-consistency}
\end{equation}
provided variances $\fvar_1,\ldots,\fvar_\cmind, \svar_1,\ldots,\svar_\cmind$ are non-degenerate
and the sequence $\frac{\fmind}{\fmind+\smind}$ has only one accumulation point (i.e. has a limit).

The BIC-like criterion~(\ref{eq:bic-criterion}) shows good performance on both
simulated and real data (see Section~\ref{section:experimental-validation}),
choosing correctly the number of objects in all the cases.

\section{Clustering Using Auditory and Visual Data}
\label{sec:results}

We illustrate the method in the case of audiovisual (AV) objects.
Objects could be characterized both by their locations in space
and by their auditory status, i.e., whether they are emitting
sounds or not. These object characteristics are not directly
observable and hence they need to be inferred from sensor data,
e.g., cameras and microphones. These sensors are based on
different physical principles, they operate with different
bandwidths and sampling rates, and they provide different types of
information. On one side, light waves convey useful visual
information only indirectly, on the premise that they reflect onto
the objects' surfaces. A natural scene is composed of many
objects/surfaces and hence the task of associating visual data
with objects is a difficult one. On the other side, acoustic waves
convey auditory information directly from the emitter to the
receiver but the observed data is perturbed by the presence of
reverberations, of other sound sources, and of background noise.
Moreover, very different methods are used to extract information
from these two sensor types. A wide variety of computer vision
principles exist for extracting 3D points from a single image or
from a pair of stereoscopic cameras~\citep{ForsythPonce2003} but
practical methods are strongly dependent on the lighting
conditions and on the properties of the objects' surfaces
(presence or absence of texture, color, shape, reflectance, etc.).
Similarly, various algorithms were developed to locate sound
sources using a microphone pair based on interaural time
differences (ITD) and on interaural level differences
(ILD)~\citep{brown06casabook,christensen07integrating}, but these
cues are difficult to interpret in natural settings due to the
presence of background noise and of other reverberant objects. A
notable improvement consists in the use a larger number of
microphones~\citep{dibiase01robust}. Nevertheless, the extraction
of 3D sound source positions from several microphone observations
results in inaccurate estimates. We show below that our method can
be used to combine visual and auditory observations to detect and
localize objects. A typical example where the conjugate mixture
models framework may help is the task of locating several speaking
persons.

Using the same notations
as above, we consider two sensor spaces. The multimodal data consists of
$M$
visual observations $\vobss$ and of
$K$ auditory observations $\aobss$. We consider
data that are recorded over a short time interval $[t_1,t_2]$, such that
one can reasonably assume that the AV objects have a stationary
spatial location. Nevertheless, it is not assumed here that the AV
objects, e.g., speakers, are static: lip movements, head and hand
gestures are tolerated.
We address the problem of estimating the spatial locations of all the
objects that are both seen and heard. Let $N$ be the number of objects
and in this case each object is described by a three dimensional
parameter vector
$\paramind=(\rcx_\cind,\rcy_\cind,\rcz_\cind)\tp$.

The AV data are gathered using a pair of stereoscopic cameras
and a pair of omnidirectional microphones, i.e., binocular vision and
binaural hearing. A visual observation vector
$\vobsind=(\vcu_\vind,\vcv_\vind,\vcd_\vind)\tp$ corresponds to a 2D
image location $(\vcu_\vind,\vcv_\vind)$ and to an associated binocular
disparity $\vcd_\vind$. Considering a
projective camera model~\citep{Faugeras93} it is straightforward to
define an invertible function $\vfunc:\mathbb{R}^3\to\mathbb{R}^3$ that maps
$\svect=(\rcx,\rcy,\rcz)\tp$ onto $\vv{f}=(\vcu,\vcv,\vcd)\tp$:
\begin{equation}
\vfunc(\spkr) =  \left(\frac{\rcx}{\rcz}, \frac{\rcy}{\rcz}, \frac{1}{\rcz}\right)\tp \quad
\mbox{and } \quad \vfunc^{-1}(\vobs) =
\left(\frac{ \vcu}{\vcd}, \frac{ \vcv}{\vcd}, \frac{1}{\vcd} \right)\tp. \label{eq:vfuncdef}
\end{equation}
This model corresponds to a rectified camera pair~\citep{HartleyZisserman00}
and it can be easily generalized to more
complex binocular geometries~\citep{HH08,hansard07patterns}. Without
loss of generality one can use a sensor-centered coordinate system to
represent the object locations.

Similarly one can use the auditory equivalent of disparity, namely the
{\em interaural time difference} (ITD) widely used by auditory scene
analysis methods~\citep{brown06casabook}. The function
$\afunc:\mathbb{R}^3\to\mathbb{R}$ maps $\svect=(\rcx,\rcy,\rcz)\tp$
onto a 1D audio observation:
\begin{equation}
g=\afunc(\spkr)=\frac{1}{c}\Bigl(\|\spkr-\spkr_{M_1}\|-\|\spkr-\spkr_{M_2}\|\Bigr).
\label{eq:afuncdef}
\end{equation}
Here $c$ is the sound speed and
$\spkr_{M_1}$ and $\spkr_{M_2}$ are the 3D locations of the two
microphones in the sensor-centered coordinate system. Each isosurface
defined by~(\ref{eq:afuncdef}) is
represented by one sheet of a two-sheet hyperboloid in 3D. Hence, each
audio observation $g$ constrains the location of the auditory source
to lie onto a 2D manifold.

In order to perform audiovisual clustering based on the conjugate EM
algorithm, Theorem~\ref{lemma:lipschitz} (Section~\ref{sec:alganal_ls}) must hold for both
(\ref{eq:vfuncdef}) and (\ref{eq:afuncdef}), namely the functions
$\vfunc$ and $\afunc$ and their derivatives are Lipschitz
continuous. We prove the following theorem:

\begin{theorem} \label{theorem:AVmapping}
The functions $\fosm$, $\fosm'$, $\sosm$ and $\sosm'$ are
Lipschitz continuous with constants
$L_\fosm=z_{\mathrm{min}}^{-1}\sqrt{3}$,
$L'_\fosm=z_{\mathrm{min}}^{-2}$,
$L_\sosm=\|\param_{\mathrm{M_1}}-\param_{\mathrm{M_2}}\|(cR)^{-1}$
and $L'_\sosm=3(cR)^{-1}$ in the domain $\pars= \{
    |z|>z_{\mathrm{min}}>1\}
\cap
    \Big\{ \min \{
        \|\param-\param_{\mathrm{M_1}}\|,
        \|\param-\param_{\mathrm{M_2}}\|
    \}>R>1\Big\}.$
\end{theorem}

\begin{proof}
The derivatives of $\fosm$ and $\sosm$ are given by:
\begin{eqnarray}
\fosm'(\param)&=&\frac{1}{\rcz}\left[
\begin{array}{lll}1 & 0 & -\rcx/\rcz \\ 0 & 1 & -\rcy/\rcz \\ 0 & 0 & -1/\rcz
\end{array}\right]
\\
\sosm'(\param)&=&\frac{1}{c}\left(\frac{\param-\param_{\mathrm{M_1}}}{\|\param-\param_{\mathrm{M_1}}\|} - \frac{\param-\param_{\mathrm{M_2}}}{\|\param-\param_{\mathrm{M_2}}\|}\right).
\end{eqnarray}

The eigenvalues of $\fosm'(\param)$ are $1/\rcz$ and $-1/\rcz^2$,
so $\|\fosm'(\param)\|\leq\max\{z^{-1},z^{-2}\}\leq
z_{\mathrm{min}}^{-1}$, from which it follows that $L_\fosm$ can
be taken as $L_\fosm=z_{\mathrm{min}}^{-1}\sqrt{3}$. Also
$\|\fosm'(\param) - \fosm'(\tilde{\param})\|\leq
\max\{|z^{-1}-\tilde{z}^{-1}|,|z^{-2}-\tilde{z}^{-2}|\}\leq
z_{\mathrm{min}}^{-2}\|\param-\tilde{\param}\|$, so that
$L'_\fosm$ can be set to $L'_\fosm=z_{\mathrm{min}}^{-2}$.

Introducing $\evect_1 =
\frac{\param-\param_{\mathrm{M1}}}{\|\param-\param_{\mathrm{M_1}}\|}$
and $\evect_2 =
\frac{\param-\param_{\mathrm{M2}}}{\|\param-\param_{\mathrm{M_2}}\|}$,
it comes $\|\evect_1\| = \|\evect_2\| = 1$ and
$\sosm'(\param)=\frac{1}{c}(\evect_1 - \evect_2)$. Provided that
$\|\param-\param_{\mathrm{M_1}}\|$ and
$\|\param-\param_{\mathrm{M_2}}\|$ are both greater than $R$, it
follows  $\|\sosm'(\param)\|=\frac{1}{c}\|\evect_1-\evect_2\|\leq
\|\param_{\mathrm{M_1}}-\param_{\mathrm{M_2}}\|(cR)^{-1}$ and so
$L_\sosm=\|\param_{\mathrm{M_1}}-\param_{\mathrm{M_2}}\|(cR)^{-1}$.
Then, the second derivative of $\sosm$ is given by
$$
\sosm''(\param)=\frac{1}{c\|\param-\param_{\mathrm{M_1}}\|}(\Imat-\evect_1\evect_1\tp)-\frac{1}{c\|\param-\param_{\mathrm{M_2}}\|}(\Imat-\evect_2\evect_2\tp).
$$
so that
$\|\sosm''(\param)\|\leq\left|\frac{1}{c\|\param-\param_{\mathrm{M1}}\|}-\frac{1}{c\|\param-\param_{\mathrm{M2}}\|}\right|
+ \sup\limits_{\|\vvect\|=1} \frac{2\evect_1\evect_1\tp
\vvect}{c\min \{\|\param-\param_{\mathrm{M1}}\|,
\|\param-\param_{\mathrm{M2}}\|\}} \leq 3(cR)^{-1}$, and
$L'_\sosm$ can be set to  $L'_\sosm=3(cR)^{-1}.\quad\blacksquare$
\end{proof}

This result shows that under some natural conditions (The AV objects
should not be too close to the sensors)
the conjugate EM algorithm described in
Section~\ref{sec:indep:em:gem} can be applied. The constant $L$
given by Lemma~\ref{lemma:lipschitz} guarantees a certain
(worst-case) convergence speed. In practice, we can use the
techniques mentioned  in
Sections~\ref{sec:alganal_ls} and \ref{sec:alganal_gs} to accelerate
the algorithm. First, to speed up the local optimization step,
local Lipschitz constants can be computed based on the current
value of parameter $\tilde{\param}^{(\nu)}$.
Equation~(\ref{eq:radius}) gives the largest possible step size
$\rho^{(\nu)}$, so setting $z_{\mathrm{min}}^{(\nu)} = z^{(\nu)} -
\rho^{(\nu)}$ and $R^{(\nu)} =
\min\{\|\tilde{\param}^{(\nu)}-\param_{\mathrm{M_2}}\|,
\|\tilde{\param}^{(\nu)}-\param_{\mathrm{M_1}}\|\} - \rho^{(\nu)}$,
provides local Lipschitz constants that insure the update not to
quit $\pars^{(\nu)}= \{
    |z|>z^{(\nu)}_{\mathrm{min}}\}
\cap
    \Big\{ \min \{
        \|\param-\param_{\mathrm{M_1}}\|,
        \|\param-\param_{\mathrm{M_2}}\|
    \}>R^{(\nu)}\Big\}$.
Second, we propose four possibilities to set the initial object
parameter values $\tilde{\param}^{(0)}_\cind$:
(i)~it can be taken to be the previously estimated object position $\param^{(\iterind-1)}_\cind$,
(ii)~it can be set to $\fosm^{-1}(\fobsavg)$ (as soon
as $\fosm$ is injective in $\pars$), (iii)~it can be found through
sampling of the manifold $\sosm^{-1}(\sobsavg)$ by selecting the
sampled value which gives the largest $Q$ value, or (iv)~similarly
through sampling directly in $\pars$. Comparisons are reported in
the following sections.

\section{Experiments with Simulated Data}
\label{sec:results1}

Our algorithm is first illustrated on simulated data. For simplicity
we consider $(u, d)$ and $(x, z)$
coordinates so that $\fos\subseteq\mathbb{R}^2$ and
$\pars\subseteq\mathbb{R}^2$. Notice however that this preserves the projective nature
of the mapping $\fosm$, it does not qualitatively affect the
results and allows to better understand the algorithm performance. We
consider three objects defined in $\pars$ by $\paramind$,
$\cind=1,2,3$. We simulated three cases: well-separated objects (GoodSep), partially occluded objects
(PoorSep) and poor precision in visual observations for
well-separated objects (PoorPrec). The
ground-truth object
locations $(x, z)$ for the GoodSep and PoorPrec cases are the same,
namely $\vv{s}_1=(-300, 1000)$, $\vv{s}_2=(10, 800)$ and $\vv{s}_3=(500, 1500)$. In the
PoorSep case, the coordinates are respectively $\vv{s}_1=(-300, 1000)$,
$\vv{s}_2=(10, 800)$ and $\vv{s}_3=(100, 1500)$.
The data in both observation spaces $\fos$
and $\sos$ was simulated from a mixture model with three Gaussian
components and a uniform component that models the outliers. The means of the
Gaussian components are computed using $\fosm(\paramind)$ and
$\sosm(\paramind)$, $\cind=1,2,3$. An example of simulated data for the three
mentioned configurations is shown in
Figure~\ref{fig:simudata1}, i.e., $(u,d)$ locations of the
visual observations and ITD values of the auditory observations. 

\begin{figure*}[tb]
\begin{tabular}{ccc}
\hspace{-0.5in}
\includegraphics[width=0.4\textwidth, type=pdf, ext=.pdf, read=.pdf]{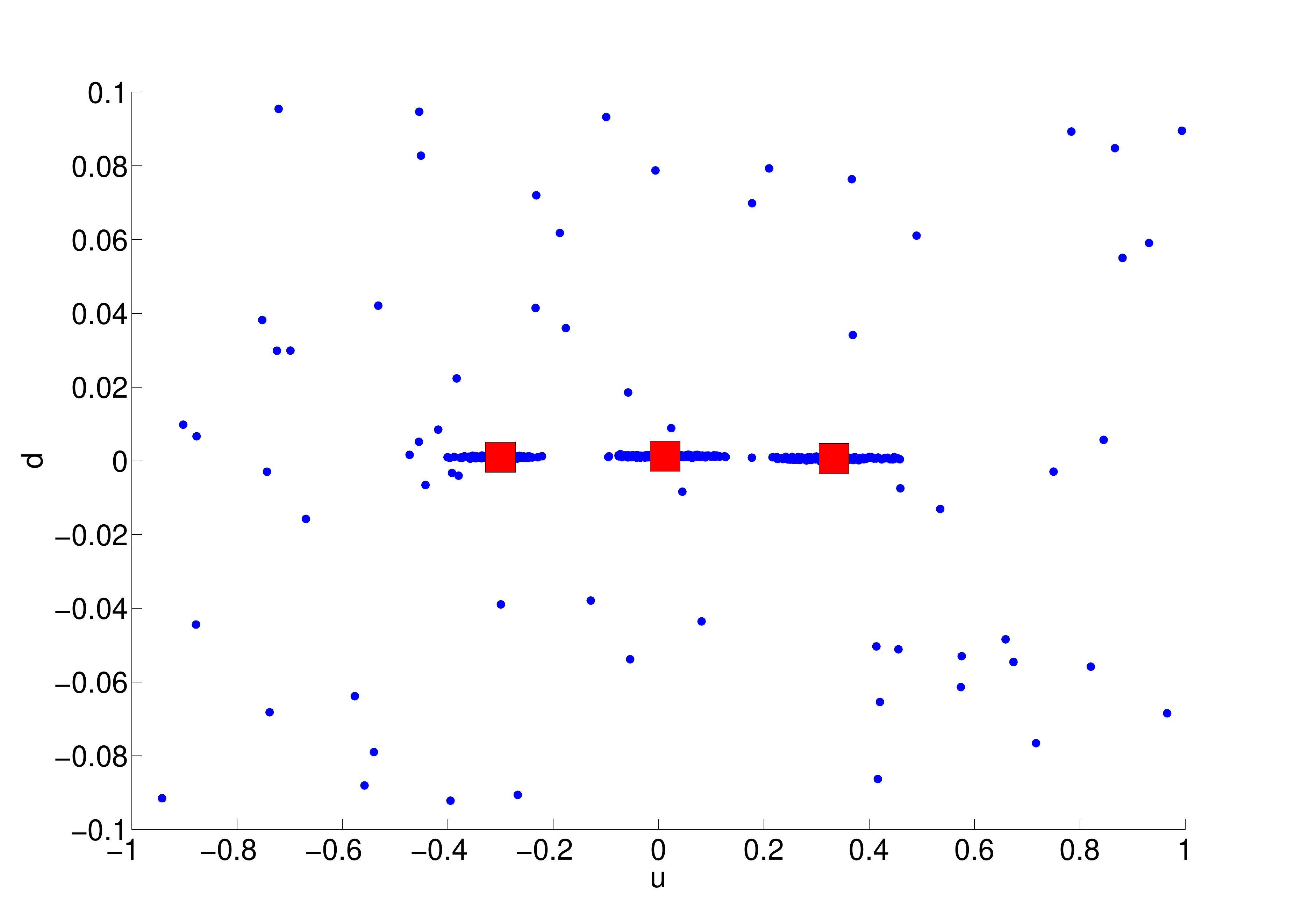} &
\hspace{-0.45in} 
\includegraphics[width=0.4\textwidth, type=pdf, ext=.pdf, read=.pdf]{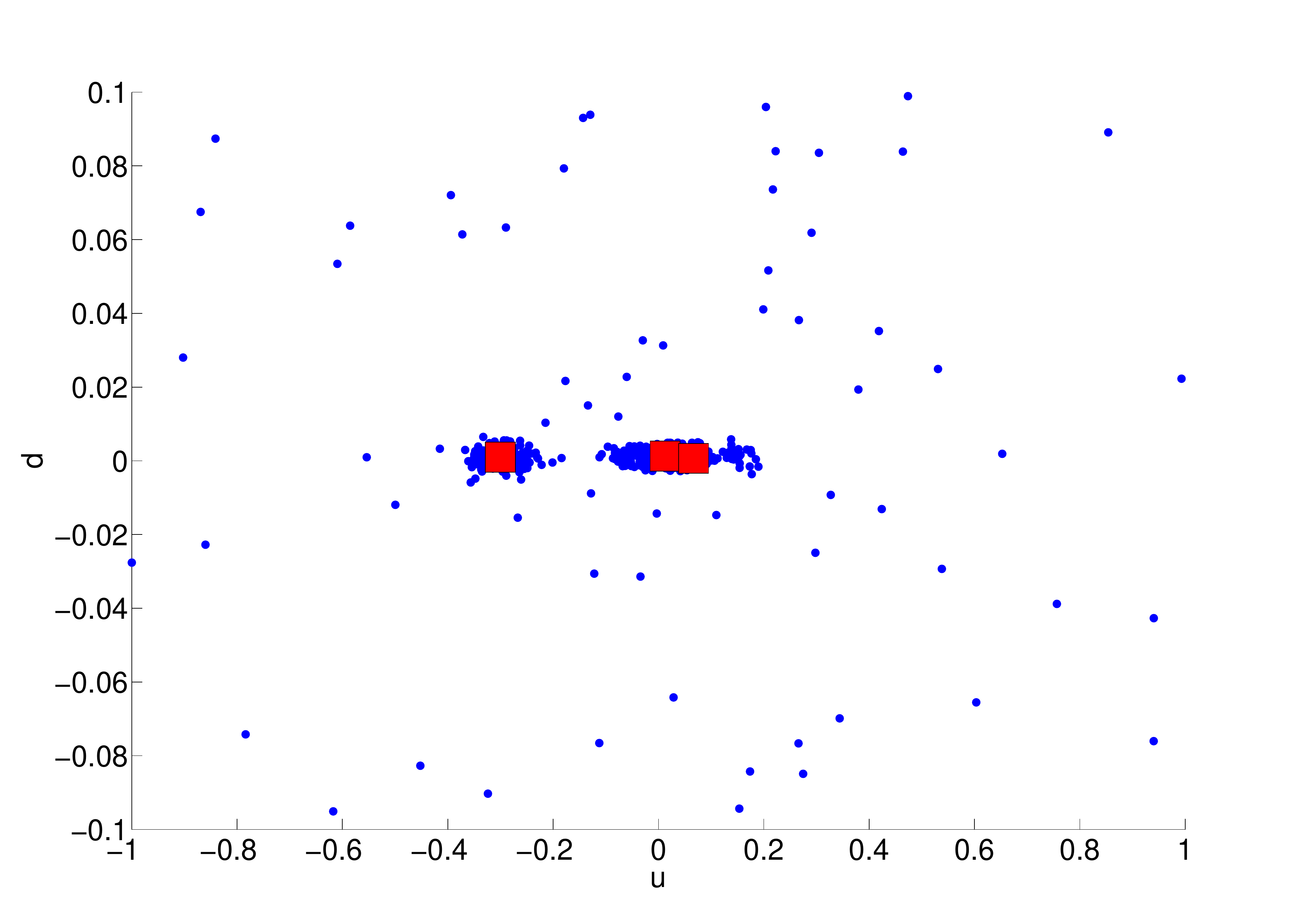}  &
\hspace{-0.45in}
\includegraphics[width=0.4\textwidth, type=pdf, ext=.pdf, read=.pdf]{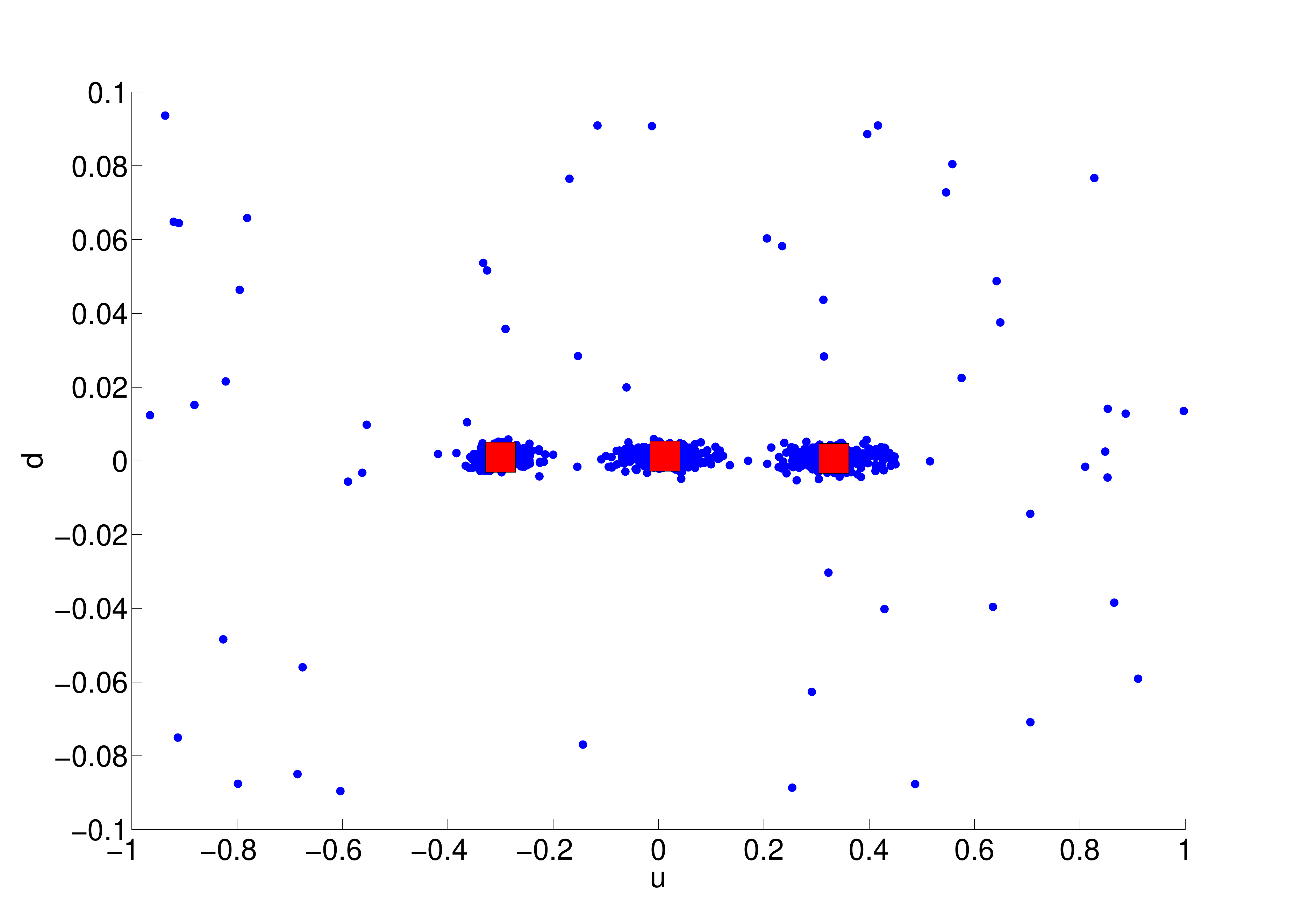}  \\
\hspace{-0.5in}
\includegraphics[width=0.4\textwidth, type=pdf, ext=.pdf, read=.pdf]{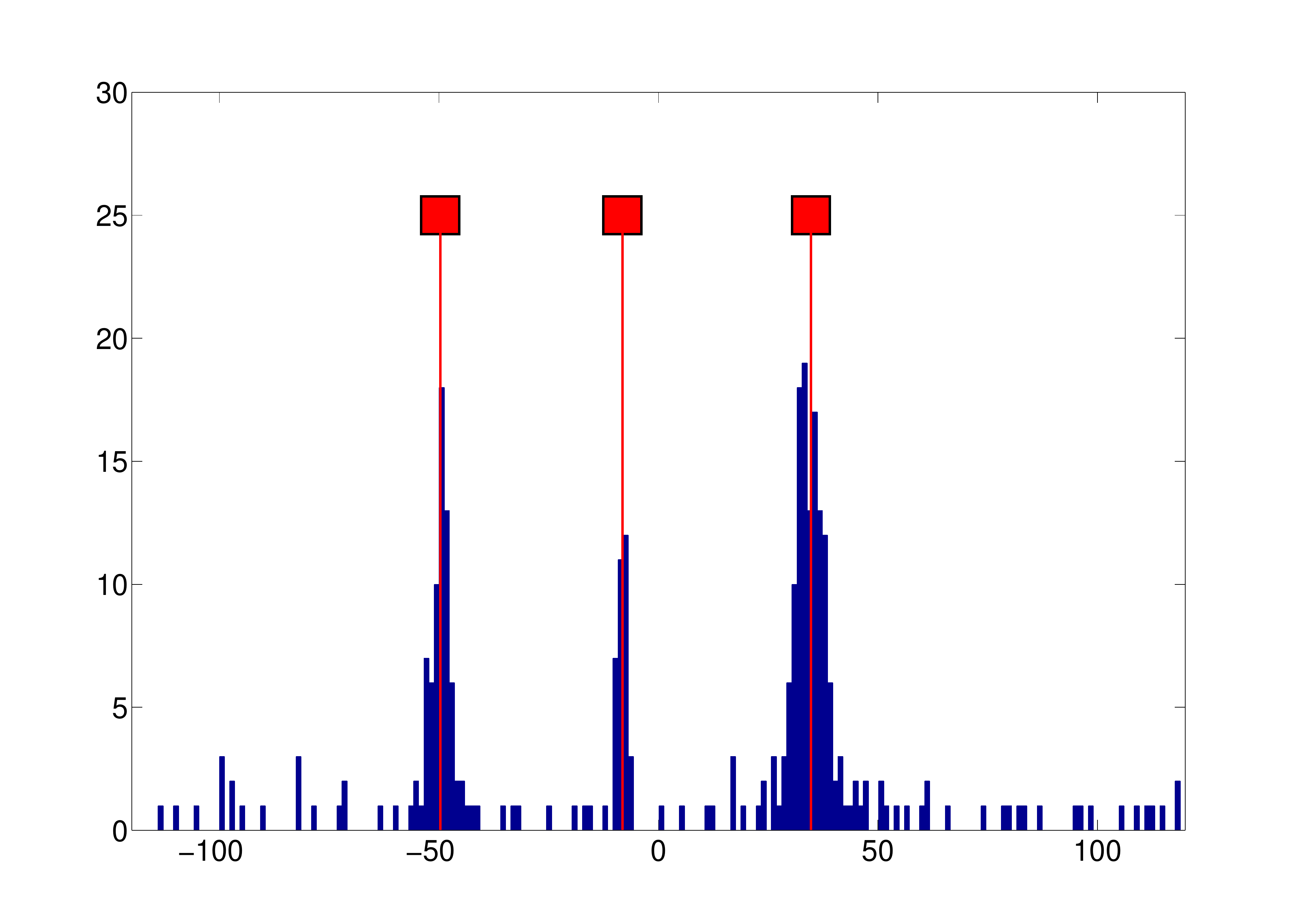} &
\hspace{-0.45in} 
\includegraphics[width=0.4\textwidth, type=pdf, ext=.pdf, read=.pdf]{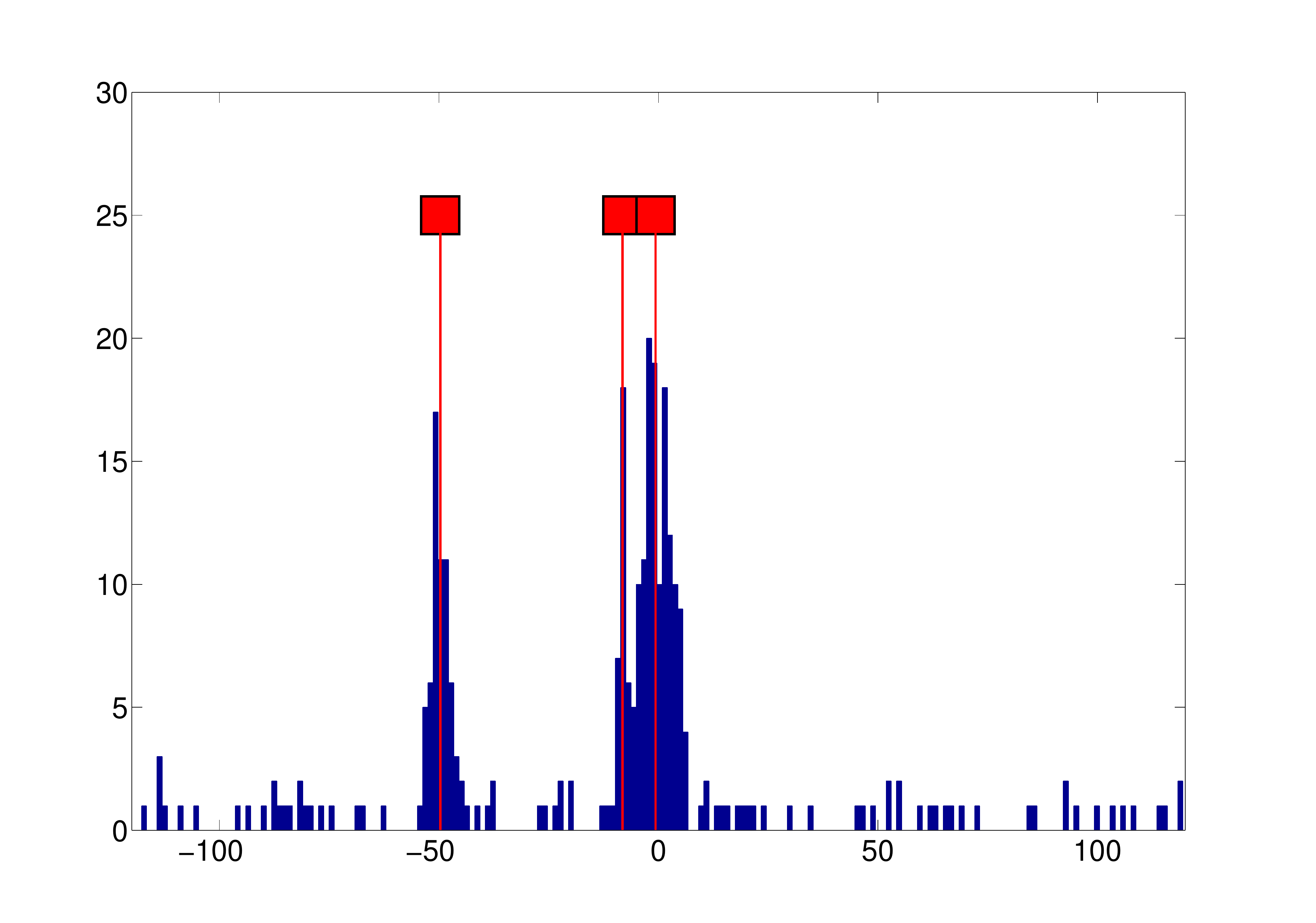}  &
\hspace{-0.45in}
\includegraphics[width=0.4\textwidth, type=pdf, ext=.pdf, read=.pdf]{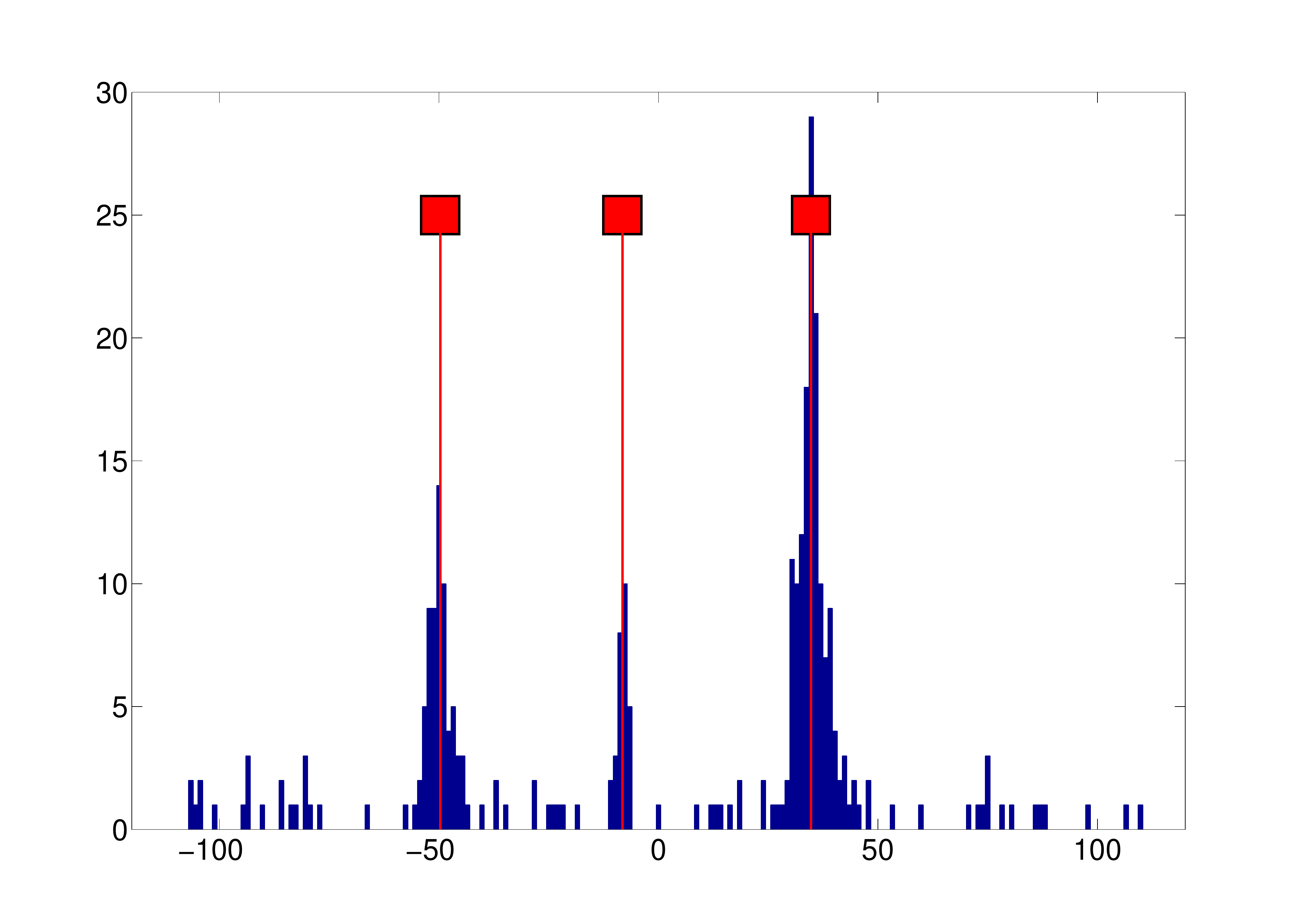}  \\
     (a) GoodSep & (b) PoorSep & (c) PoorPrec
\end{tabular}
\caption{Simulated data in visual (top) and audio (bottom) observation
spaces for three cases:  (a)~well-separated objects, (b)~partially
occluded objects, and (c)~poor precision of visual observations.
The small squares correspond to the ground-truth parameter values. 
Each one of the two mixtures models (associated with each sensorial
modality) contains four
components: three objects and one outlier class. }
\label{fig:simudata1}
\end{figure*}

\paragraph{Initialization.}
We compared two strategies, {\it Observation Space Candidates} (OSC) and
{\it Parameter Space Candidates} (PSC) that are proposed in Section~\ref{sec:init}.
Their performance is summarized in Figure~\ref{fig:initcompare}.
It shows the mean and variance of the likelihood value $\lhood(\fobss, \sobss, \params)$
for initial parameters $\params^{(0)}_{\mathrm{OSC}}$ and $\params^{(0)}_{\mathrm{PSC}}$
chosen by OSC and PSC strategies respectively.
For the total number of clusters $N=1, \ldots, 5$ and different object configurations,
we calculate the statistics based on 10 initializations.
The analysis shows that the PSC strategy performs at least as well as the OSC strategy,
or even better in some cases.
Our explanation is that mappings from observation spaces to parameter space are subject
to absolute (and in our case bounded) noise.
Mapping all the observations and calculating a candidate point in the parameter space
has an averaging effect and reduces the absolute error,
compared to the strategy with candidate calculation being performed in an observation space
with subsequent mapping to the parameter space.
Therefore in what follows, all the results are obtained based on the PSC initialization strategy.
\begin{figure*}[tb]
\begin{tabular}{ccc}
\hspace{-0.5in}
\includegraphics[width=0.4\textwidth, type=pdf, ext=.pdf, read=.pdf]{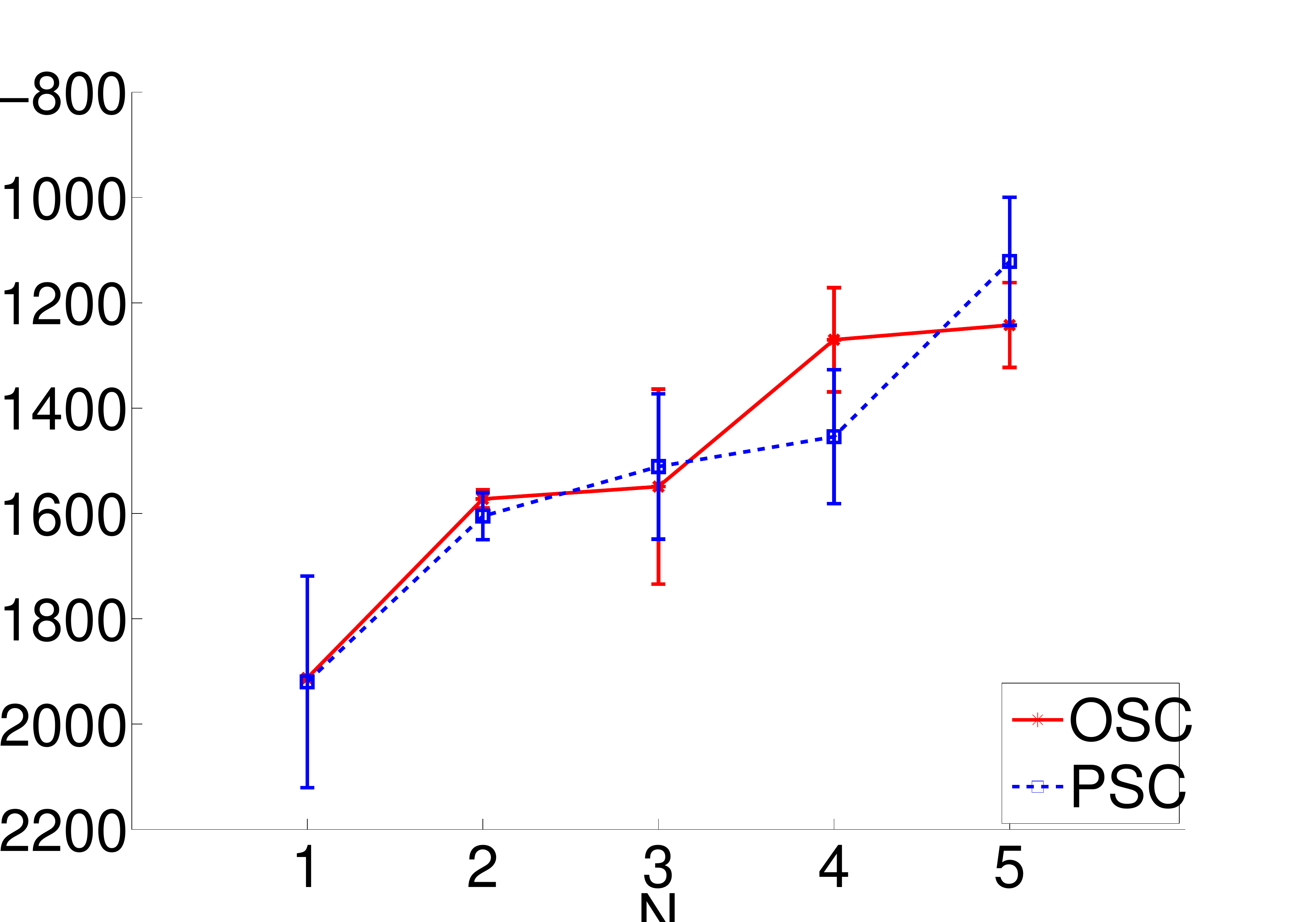} &
\hspace{-0.45in} 
\includegraphics[width=0.4\textwidth, type=pdf, ext=.pdf, read=.pdf]{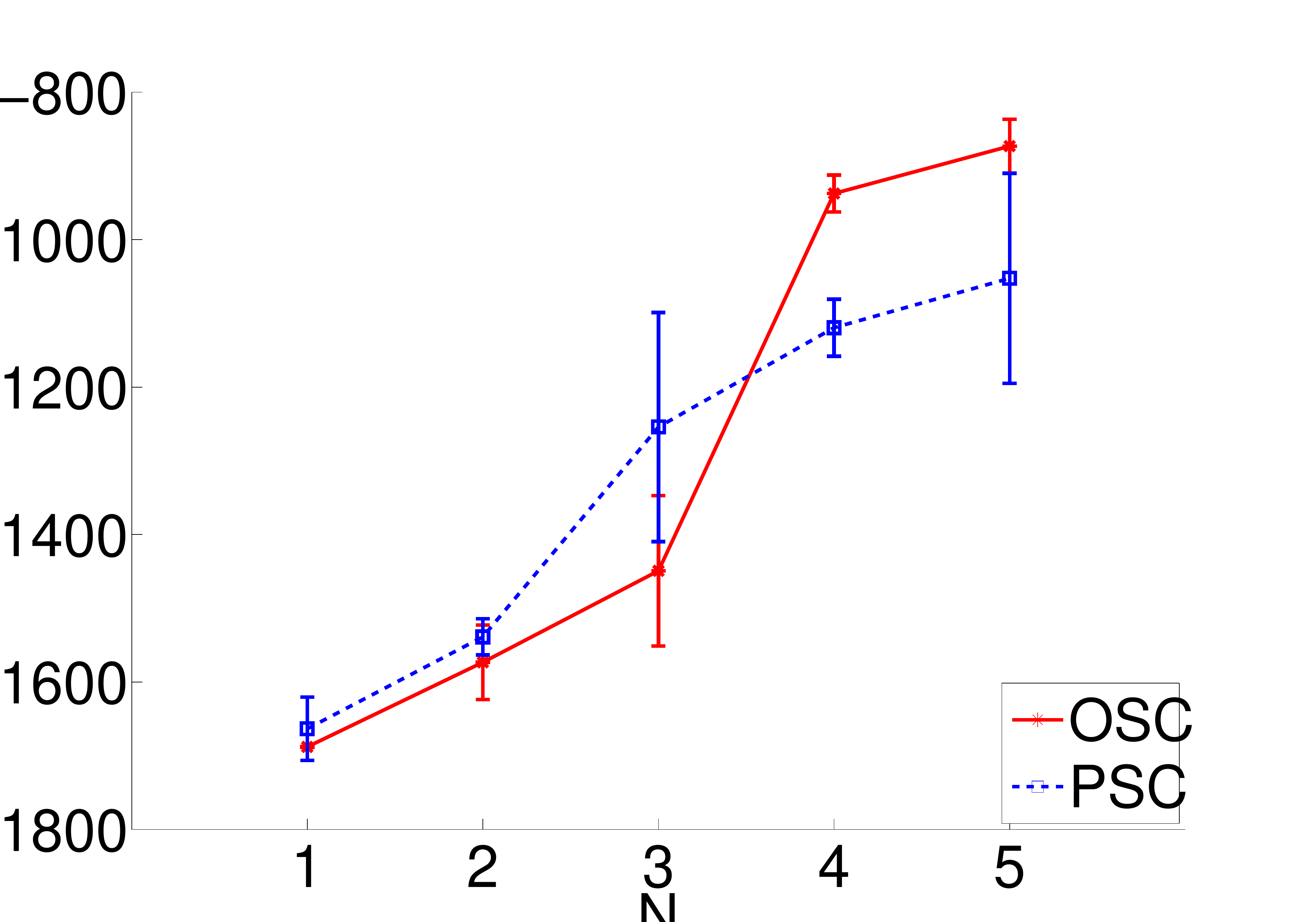}  &
\hspace{-0.45in}
\includegraphics[width=0.4\textwidth, type=pdf, ext=.pdf, read=.pdf]{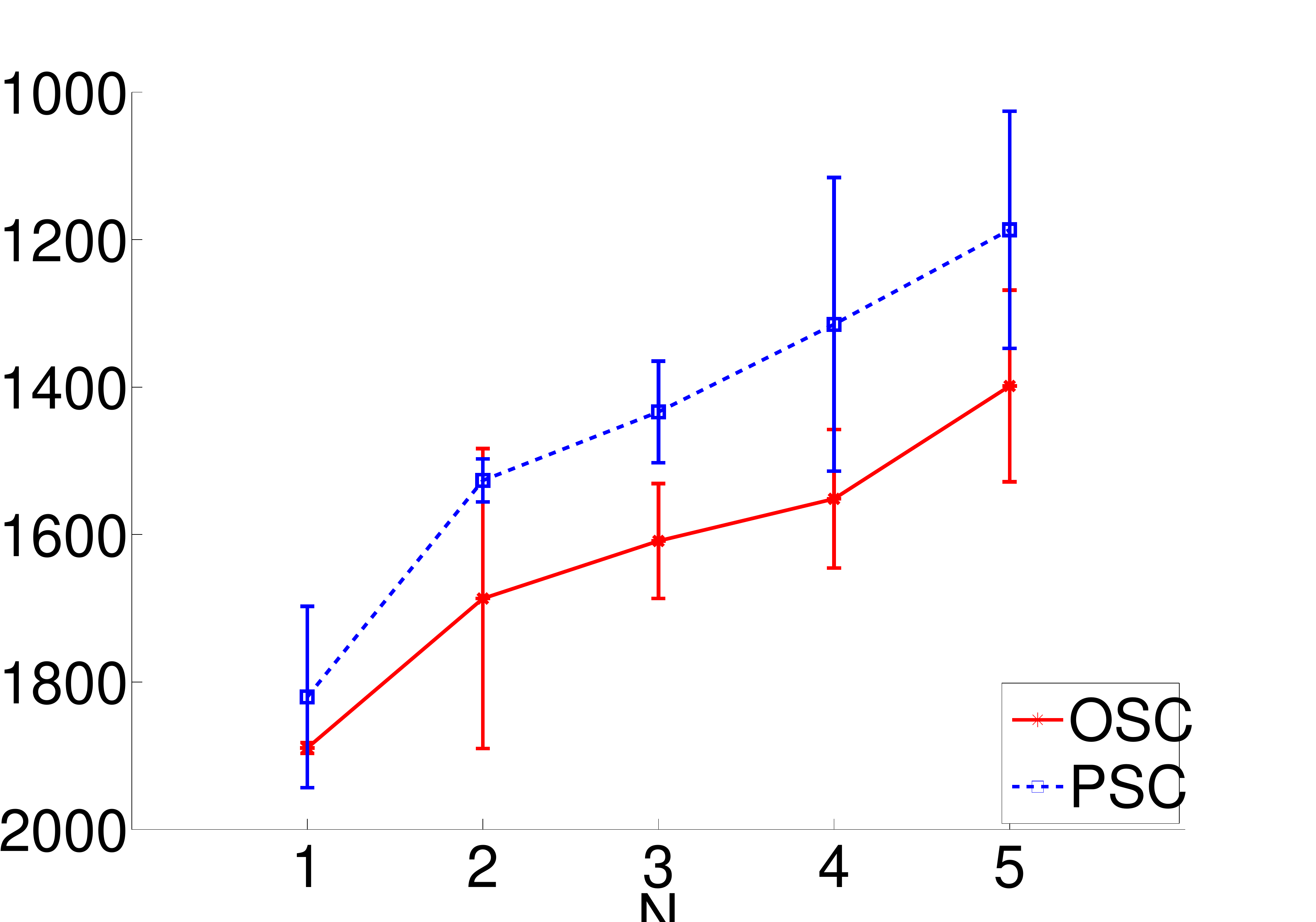}  \\
     (a) GoodSep & (b) PoorSep & (c) PoorPrec
\end{tabular}
\caption{Means and variances of log-likelihood values $\lhood(\fobss, \sobss, \params)$
for initial parameters $\params^{(0)}_{\mathrm{OSC}}$ and $\params^{(0)}_{\mathrm{PSC}}$
chosen by {\it Observation Space Candidates} (OSC, red) and
{\it Parameter Space Candidates} (PSC, blue) strategies respectively,
for different numbers of clusters $\cmind$ and different data configurations.}
\label{fig:initcompare}
\end{figure*}%

\paragraph{Optimization.} We compared several versions of the algorithm based on various
{\it Choose} and {\it Local Search} strategies.
For the initial values $\tilde{\paramind}^{(0)}$, we considered the
following possibilities: the optimal value computed at a
previous run of the algorithm (IP), the value predicted from
visual data (IV), the value predicted from audio data (IA) and the
value obtained by global random search (IG). More specifically:
\begin{itemize}
\item
When initializing from visual data (IV), the average value
$\fobsavg_\cind$, calculated in the current E-step of the
algorithm for every $\cind$, was mapped to the parameter space and
$\tilde{\paramind}^{(0)}$ set to
$\tilde{\paramind}^{(0)}=\fosm^{-1}(\fobsavg_\cind)$ using the
injectivity of $\fosm$. 
\item
When initializing from audio data (IA),
$\sosm^{-1}(\sobsavg_\cind)$ defines a manifold. The general
strategy here would be to find the optimal point that lies on this
surface. We achieved this through random search based on a uniform
sampling on the corresponding part of the hyperboloid
(see~\citep{zhigljavsky91theory} for details on sampling from an
arbitrary distribution on a manifold); in our experiments we used
50 samples to select the one providing the largest $Q$ (likelihood)
value. 
\item
The
most general initialization scheme (IG) was implemented using
global random search in the whole parameter space $\pars$; 200
samples were used in this case.
\end{itemize}

Local optimization was performed either using basic gradient
ascent (BA) or the locally accelerated gradient ascent (AA). The
latter used the local Lipschitz constants to augment the step
size, as described in Section~\ref{sec:alganal_ls}.

Each algorithm run consisted of 70 iterations of the EM algorithm
with 10 non-decreasing iterations during the M step.

To check the convergence speed of different versions of the algorithm
for the three object configurations we compared the likelihood evolution
graphs that are presented in Figure~\ref{fig:lhoods}.
Each graph contains several
curves that correspond to five different versions of the
algorithm. The acronyms  we use to refer to the different versions
(for example, IPAA) consist of two parts encoding the
initialization (IP) and the local optimization (AA) types. The
black dashed line on each graph shows the `ground truth'
likelihood level, that is the likelihood value for the parameters
used to generate the data. The meaning of the acronyms is recalled
in Table~\ref{acro}.

\begin{table*}[tb]
\caption{\label{acro} Acronyms used for five variants of the
conjugate EM algorithm. Variants correspond to different choices
for the {\it Choose} and {\it Local search} procedures.}
\begin{center}
\begin{tabular}{|c|c|c|}
\hline Acronym& $\tilde{\param}^{(0)}$ initialization  ({\it
Choose}) & Local optimization ({\it Search})\\
\hline
IPBA & previous iteration value & basic gradient ascent\\
IGAA & global random search & accelerated gradient ascent\\
IVAA & predicted value from visual data & accelerated gradient ascent\\
IPAA & previous iteration value & accelerated gradient ascent\\
IAAA & audio predicted manifold sampling &
accelerated gradient ascent\\
 \hline
\end{tabular}
\end{center}
\end{table*}

\begin{figure*}[tb]
\begin{center}
     \includegraphics[width=0.48\textwidth, type=pdf, ext=.pdf, read=.pdf]{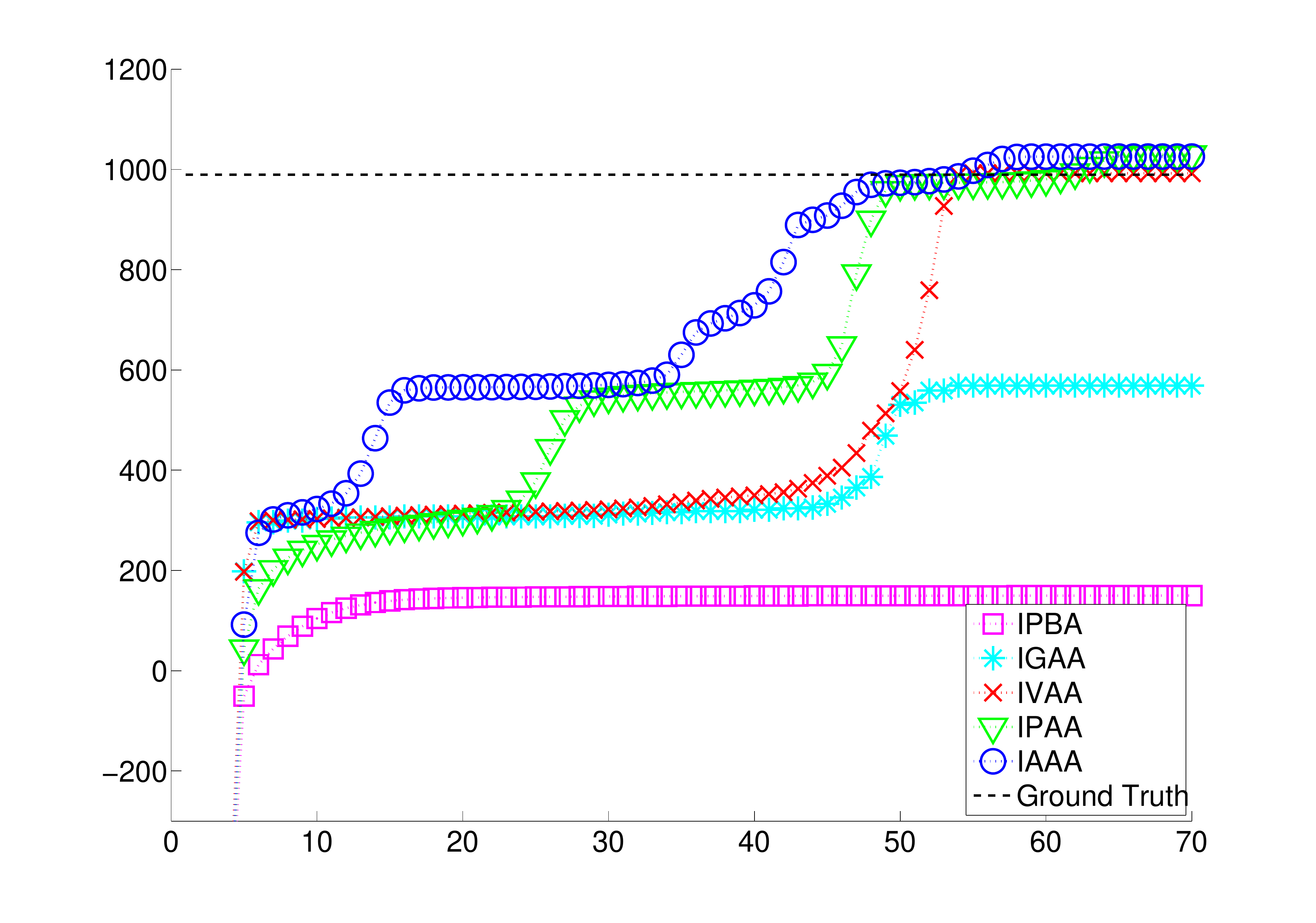}
     \includegraphics[width=0.48\textwidth, type=pdf, ext=.pdf, read=.pdf]{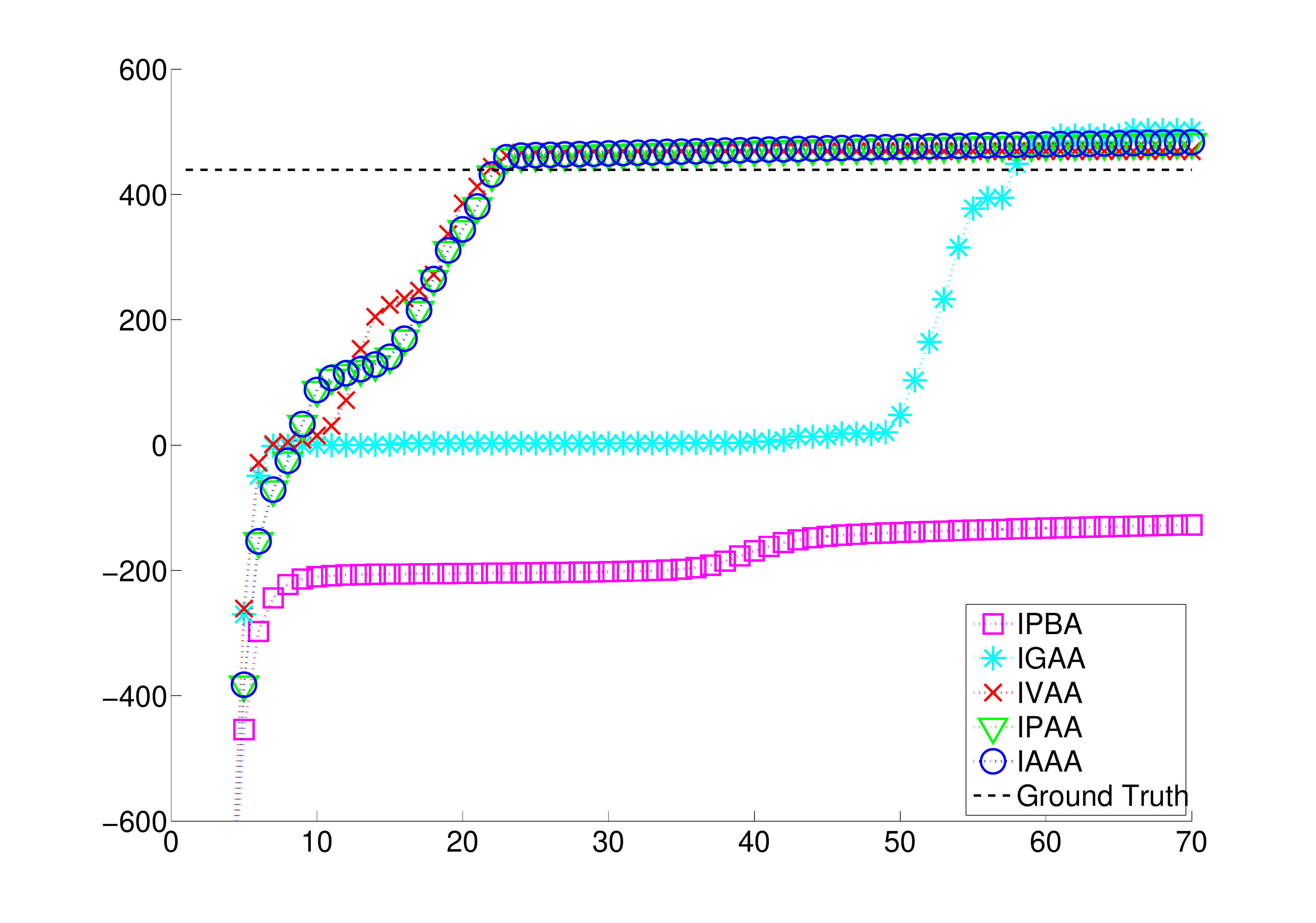}
     \includegraphics[width=0.48\textwidth, type=pdf, ext=.pdf, read=.pdf]{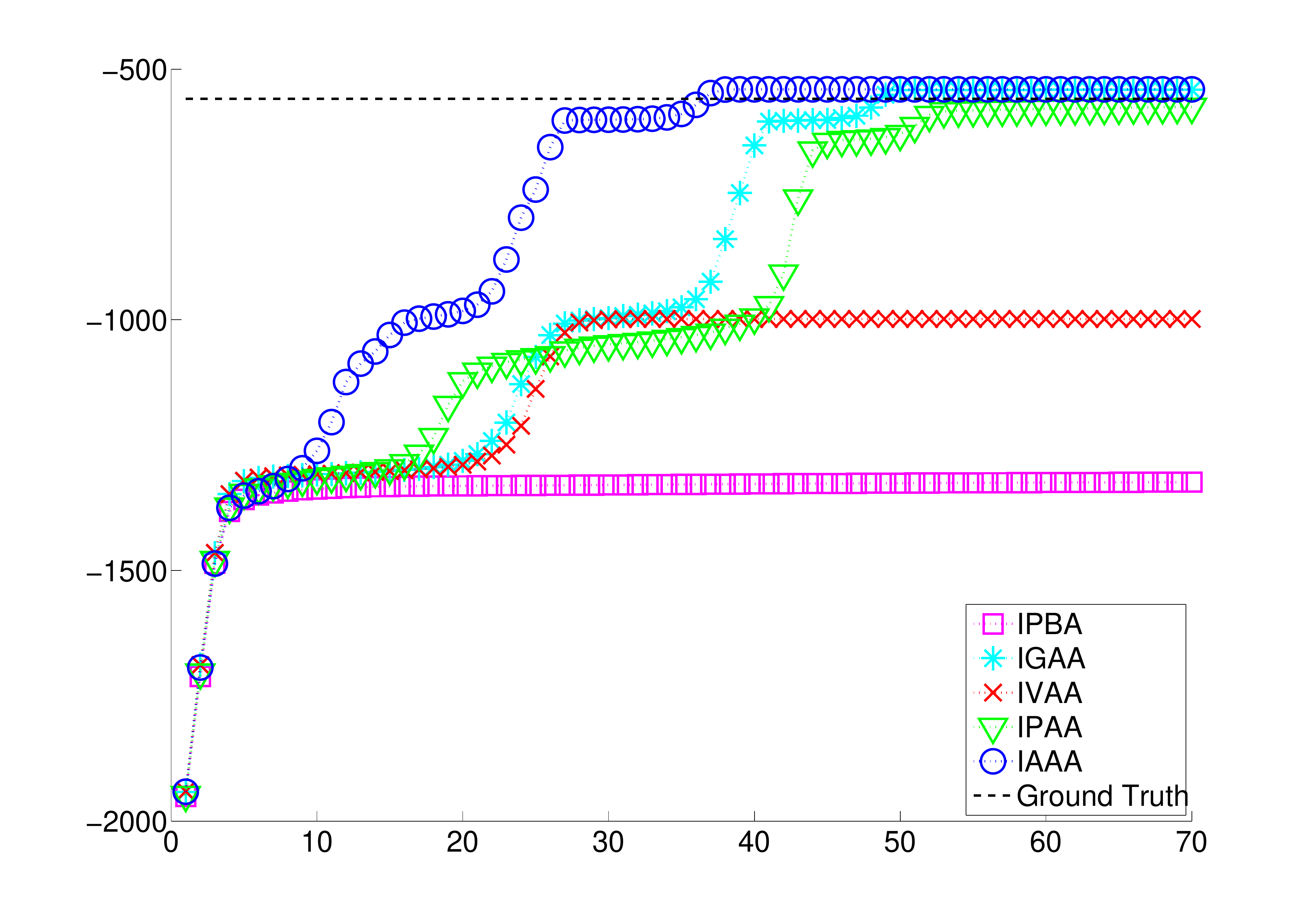}
\end{center}
\caption{Likelihood function evolution for five variants of the
algorithm in three cases. Top-left: well-separated objects;
top-right: poorly separated objects; bottom: well-separated object
but poor observation precision.} \label{fig:lhoods}
\end{figure*}

As expected, the simplest version IPBA that uses none of the
proposed acceleration techniques appears to be the slowest. The
other variants using basic gradient ascent are then not reported.
Predicting a single object parameter value from visual observations
(IVAA) does not give any improvement over IPAA, where
$\tilde{\param}^{(0)}$ is taken from the previous EM iteration.
When $\tilde{\param}^{(0)}$  is obtained by sampling the
hyperboloid predicted from audio observations (IAAA), a
significant impact on the  convergence speed is observed,
especially on early stages of the algorithm, where the predicted
value can be quite far from the optimal one. However, `blind'
sampling of the whole parameter space does not bring any
advantage: it is much less efficient regarding the number of
samples required for the same precision. This suggests that in the
general case,  the best strategy would be to sample the manifolds
$\fosm^{-1}(\fobsavg_\cind)$ and $\sosm^{-1}(\sobsavg_\cind)$ with
possible small perturbations to find the best
$\tilde{\param}^{(0)}$  estimate and to perform an accelerated
gradient ascent afterwards (IAAA).
We note that IAAA succeeds in all the cases
to find parameter values that are well-fitted to the model in
terms of likelihood function (likelihood is greater or equal
than that of real parameter values).

\begin{figure*}[tpb]
\begin{center}
     \includegraphics[width=0.8\textwidth, type=pdf, ext=.pdf, read=.pdf]{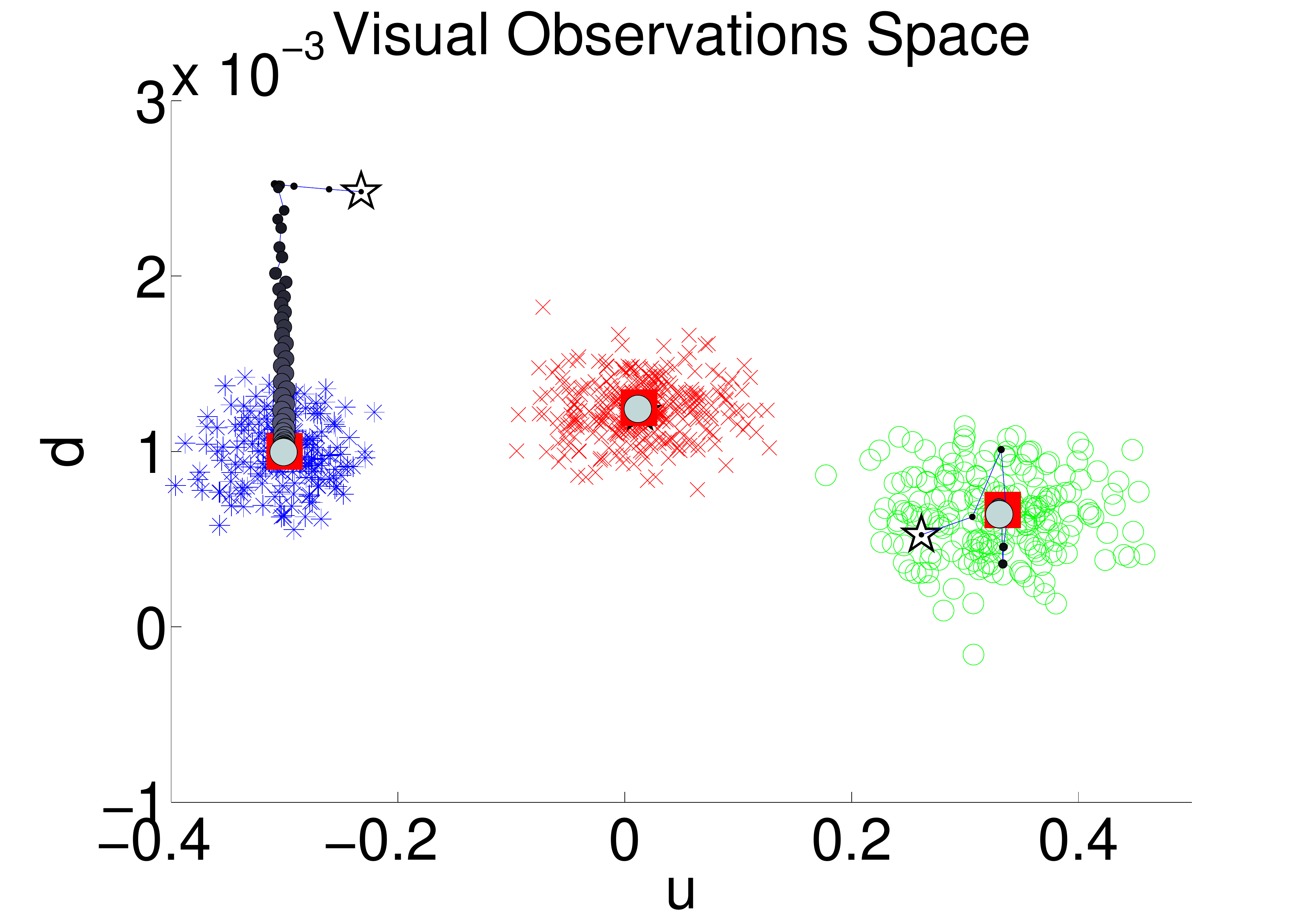}
     \includegraphics[width=0.8\textwidth, type=pdf, ext=.pdf, read=.pdf]{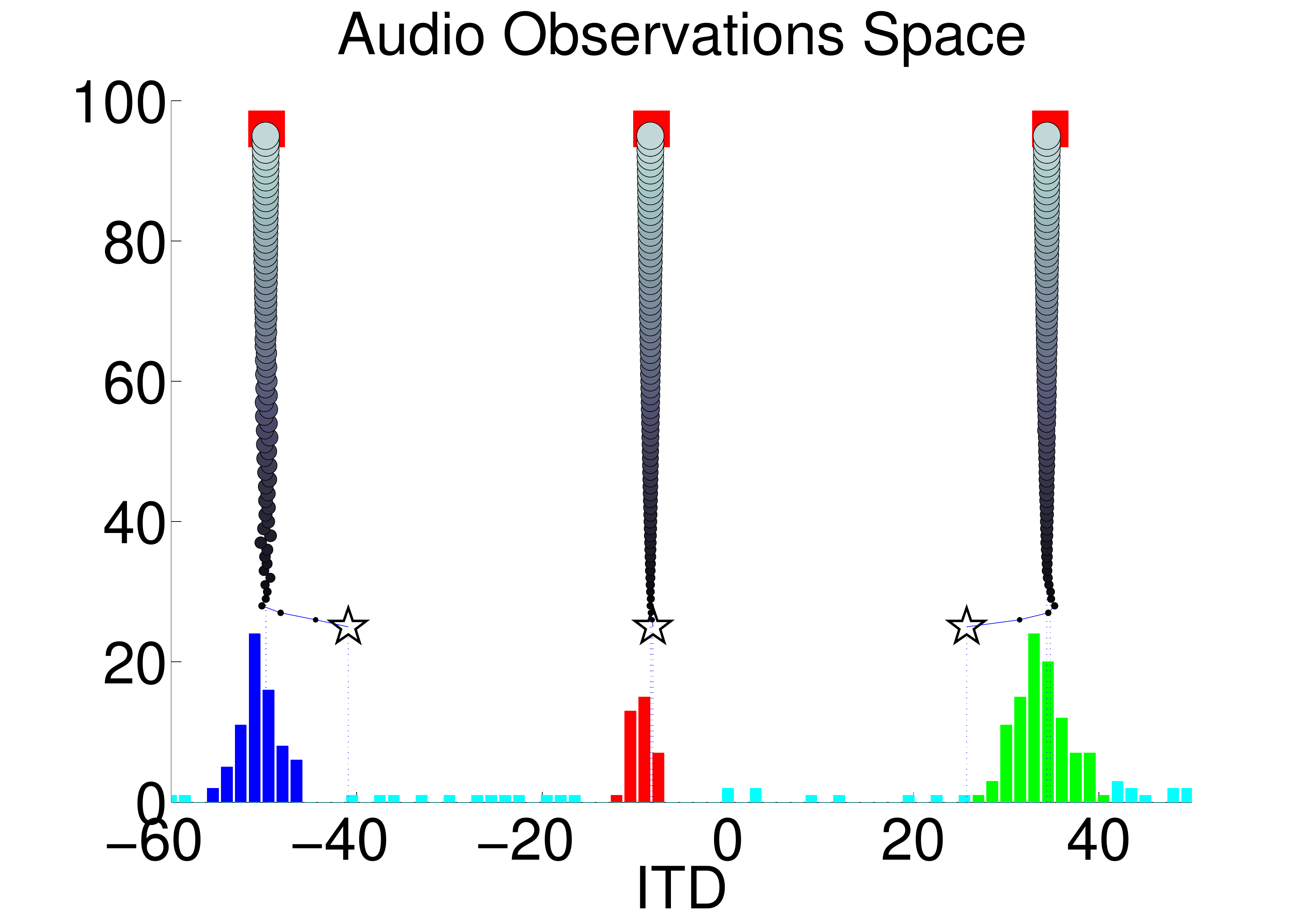}
\end{center}
\caption{IAAA algorithm: parameter evolution and assignment
results for the GoodSep case in audio and visual spaces (note the
scale change which corresponds to a zoom on the cluster centers).
The initialization (white stars) is based on the PSC strategy.
Ground truth means are marked
with squares. The evolution is shown by circles from smaller to
bigger, from darker to brighter. Observations assignments are
depicted by different markers ($\circ$, $*$ and $\times$ for the
three object classes) in visual space and are colour-coded in audio
space. Due to the zoom, outliers are not visible on these
figures.} \label{fig:simures_av}
\end{figure*}

\begin{figure*}[tpb]
\begin{center}
     \includegraphics[width=0.9\textwidth, type=pdf, ext=.pdf, read=.pdf]{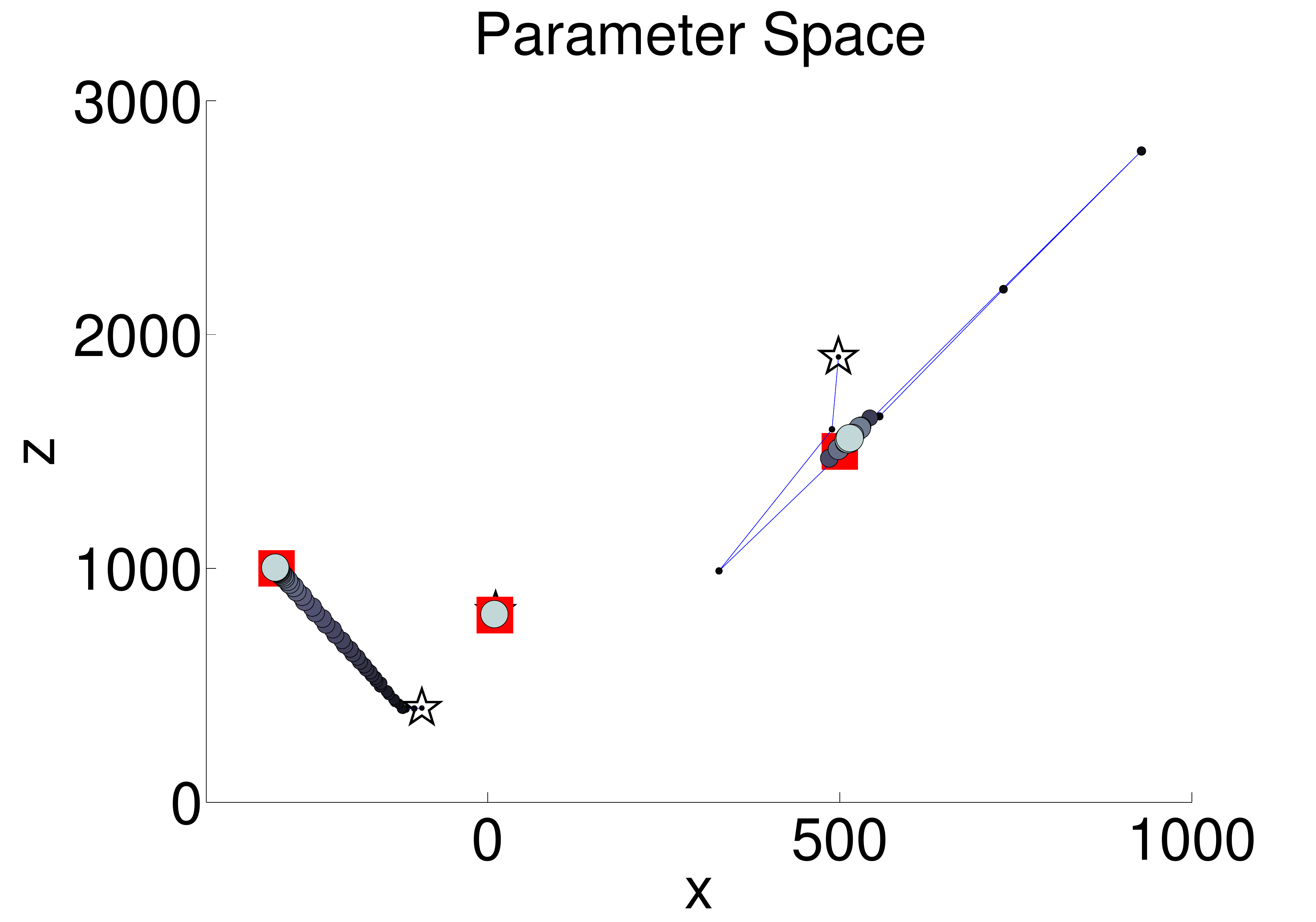}
\end{center}
\caption{IAAA algorithm: parameter evolution for the GoodSep case
in object space. The initialization (white stars) is based on the PSC strategy.
Ground truth means are marked with squares. The evolution is shown by circles
from smaller to bigger, from darker to brighter.}
\label{fig:simures_p}
\end{figure*}

\begin{figure*}[tpb]
\begin{tabular}{ccc}
\hspace{-0.5in}
\includegraphics[width=0.4\textwidth, type=pdf, ext=.pdf, read=.pdf]{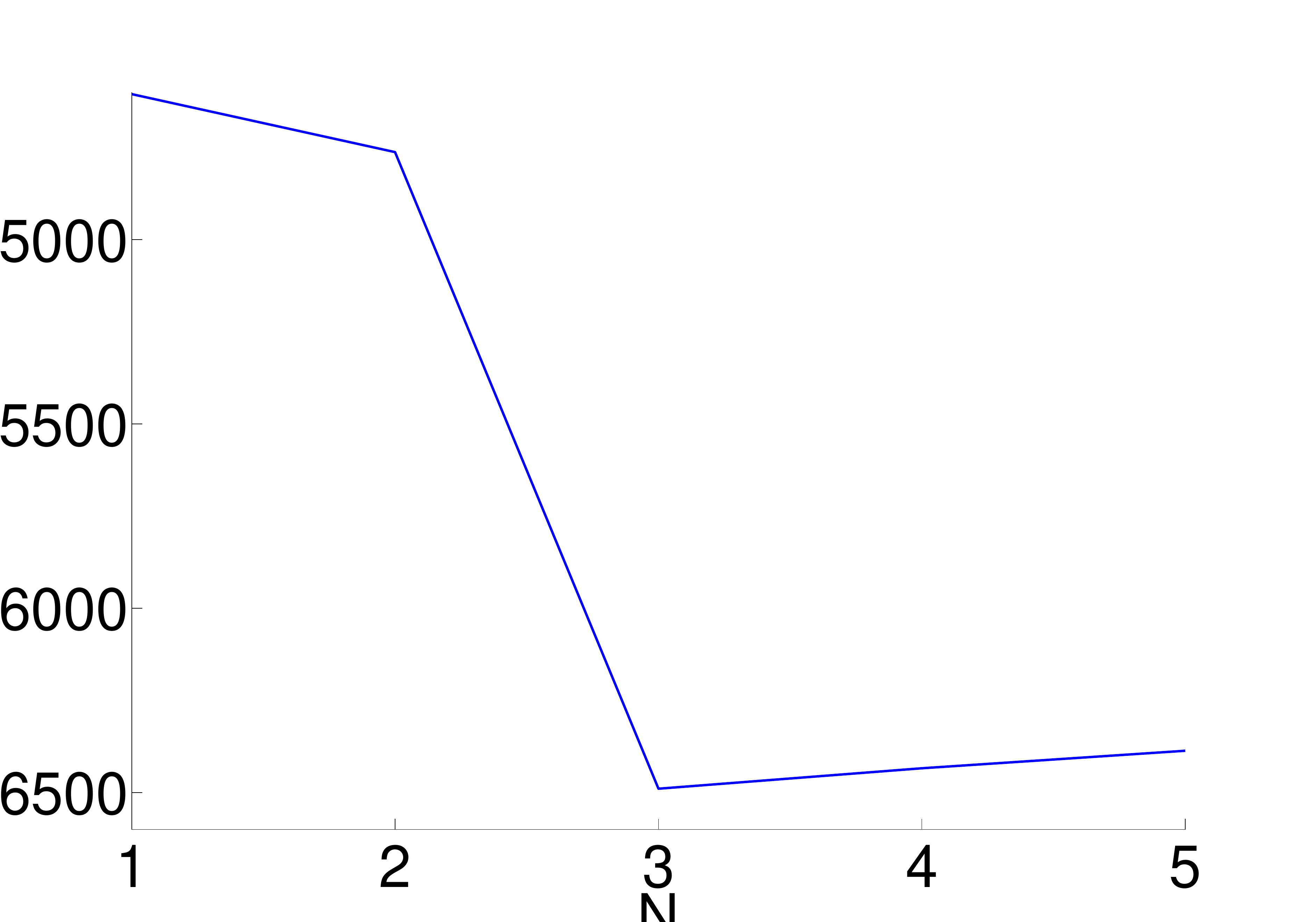} &
\hspace{-0.45in} 
\includegraphics[width=0.4\textwidth, type=pdf, ext=.pdf, read=.pdf]{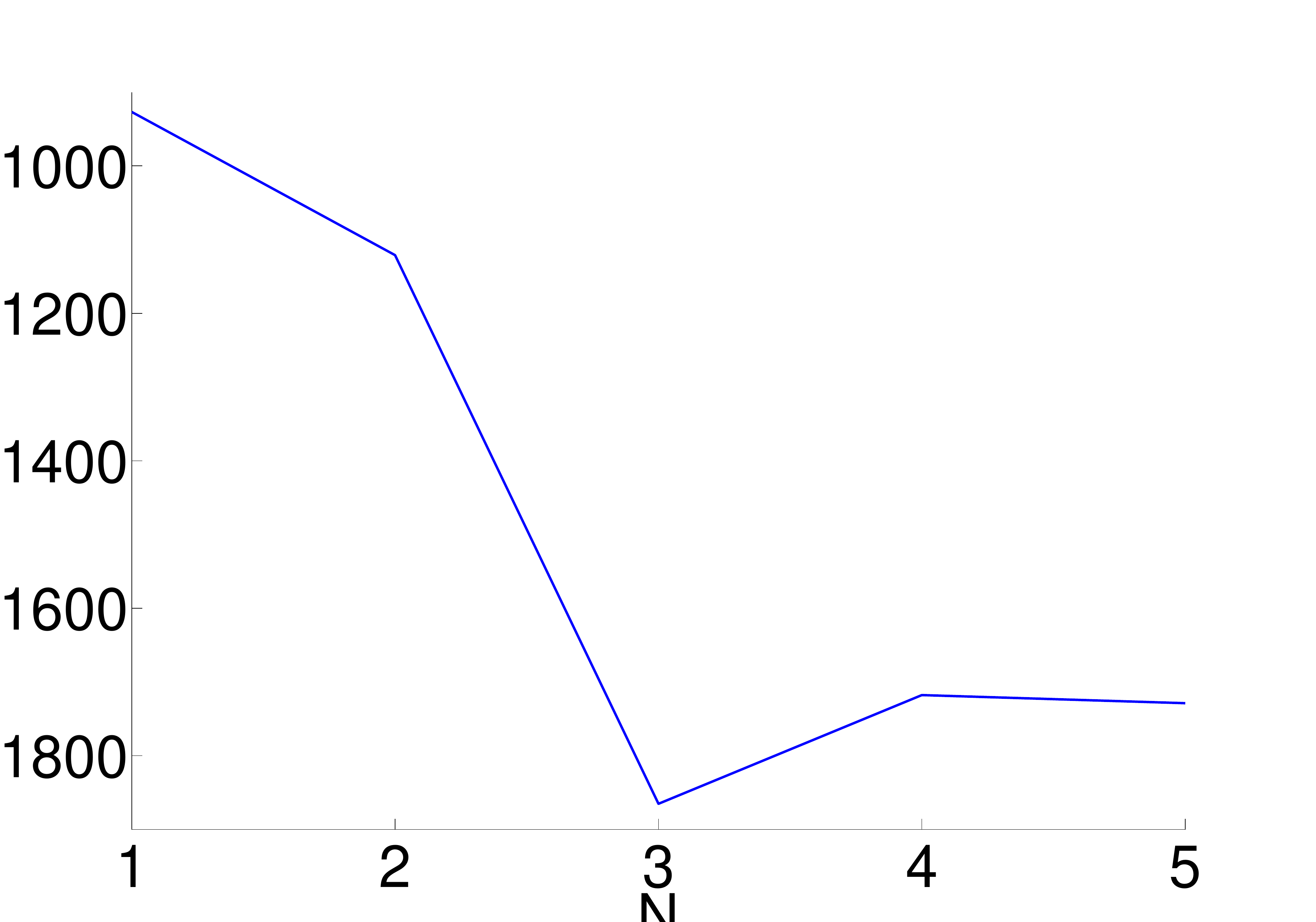}  &
\hspace{-0.45in}
\includegraphics[width=0.4\textwidth, type=pdf, ext=.pdf, read=.pdf]{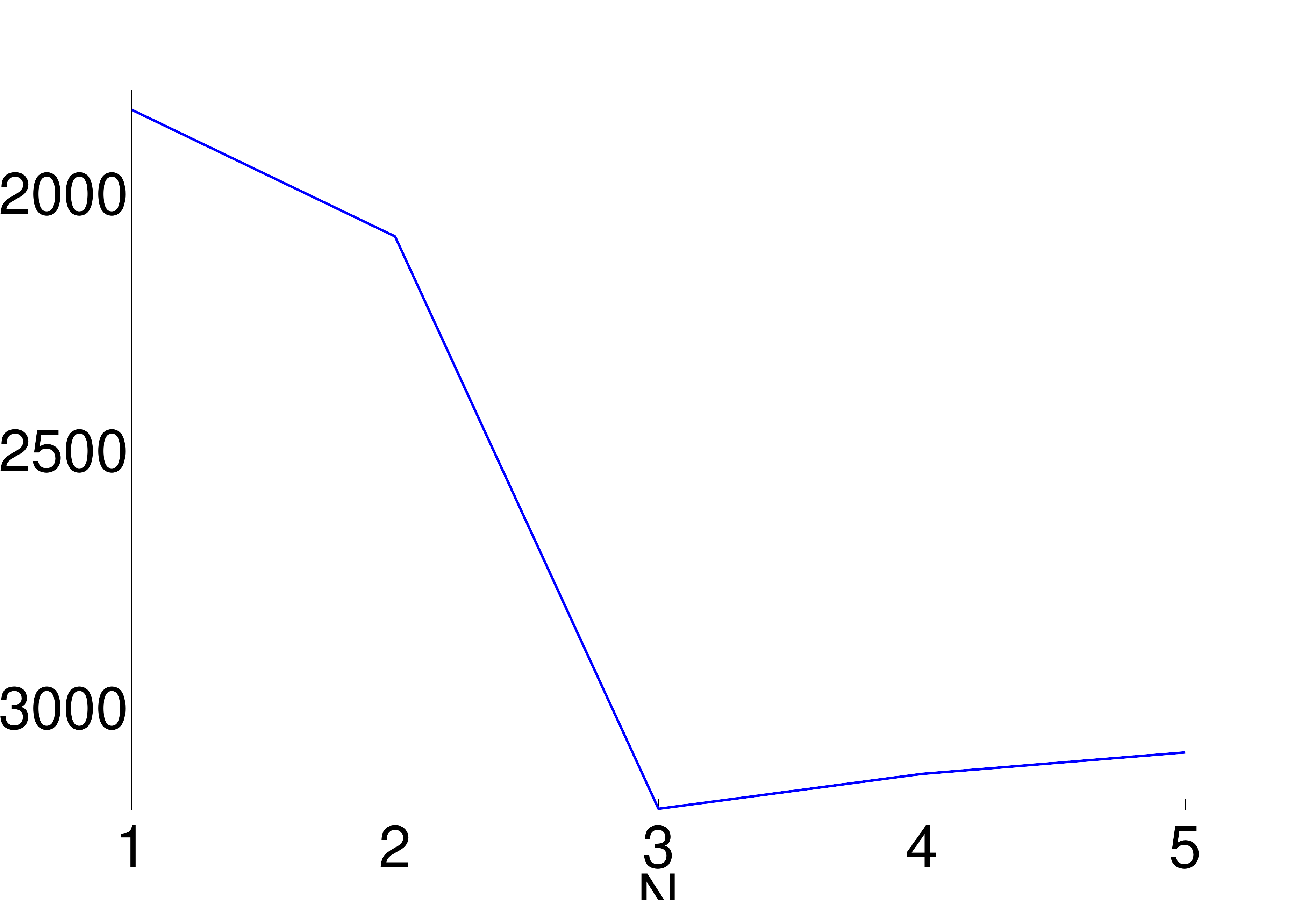}  \\
     (a) GoodSep & (b) PoorSep & (c) PoorPrec
\end{tabular}
\caption{BIC score graphs for the three object configurations, evaluated for models trained
for different total number of clusters $N$.}
\label{fig:simuBIC}
\end{figure*}%

Parameter evolution trajectories for the IAAA version of the
algorithm in the GoodSep case are shown in
Figures~\ref{fig:simures_av}-\ref{fig:simures_p}. The estimate
changes are reflected by the node sizes (from smaller to bigger)
and colours (from darker to lighter). The final values are very
close to the real cluster centers in all three audio, visual and
object spaces. The convergence speed is quite dependent on the
initialization. In the provided example the algorithm spent almost
a half of useful iterations to disentangle the estimates trying to
decide which one corresponds to which class. Another possibility
here would be to predict the initial values through sampling in
the audio domain. We demonstrate this strategy further when
working with real data.

\renewcommand{\arraystretch}{0.85}
\begin{table*}[tpb]
\caption{\label{tab:simures} IAAA algorithm: object location estimates in parameter, visual and audio spaces
for GoodSep, PoorSep and PoorPrec object configurations.
The estimates are calculated based on ten runs of the algorithm with PSC initializations.}
\begin{center}
\begin{tabular}{|c|c|c|c|c|c|}
\hline
\multicolumn{2}{|c|}{} & Ground Truth & Estimates Mean & Absolute Error & Relative Error \\
\hline
\multicolumn{2}{|c|}{Parameter Space}
& $\param=(\rcx,\rcz)$ & $\hat{\param}=(\hat{\rcx},\hat{\rcz})$ & $e_\mathrm{a}=\|\hat{ \param} - \param \|$ & $e_\mathrm{r}=\|\hat{ \param} - \param \| / \|\param\|$ \\[\smallskipamount]
\hline
\multirow{3}{*}{\rotatebox{90}{GoodSep}}
& Object 1 & \scalebox{0.8}{$(-300, 1000)$} & \scalebox{0.8}{$(-300.13, 997.81)$} & 2.2   & $2.1\cdot10^{-3}$  \\
& Object 2 & \scalebox{0.8}{$(10, 800)$}    & \scalebox{0.8}{$(9.28, 804.46)$}    & 4.52  & $5.7\cdot10^{-3}$  \\
& Object 3 & \scalebox{0.8}{$(500, 1500)$}  & \scalebox{0.8}{$(513.56, 1555.23)$} & 56.86 & $3.5\cdot10^{-2}$  \\
\hline
\multirow{3}{*}{\rotatebox{90}{PoorSep}}
& Object 1 & \scalebox{0.8}{$(-300, 1000)$} & \scalebox{0.8}{$(-307.47, 1028.38)$} & 29.35 & $2.8\cdot10^{-2}$  \\
& Object 2 & \scalebox{0.8}{$(10, 800)$}    & \scalebox{0.8}{$(14.19, 895.69)$}    & 95.79 & $1.2\cdot10^{-1}$  \\
& Object 3 & \scalebox{0.8}{$(100, 1500)$}  & \scalebox{0.8}{$(105.02, 1447.49)$}  & 52.75 & $3.5\cdot10^{-2}$  \\
\hline
\multirow{3}{*}{\rotatebox{90}{PoorPrec}}
& Object 1 & \scalebox{0.8}{$(-300, 1000)$} & \scalebox{0.8}{$(-208.86, 698.51)$} & 314.97 & $0.3$  \\
& Object 2 & \scalebox{0.8}{$(10, 800)$}    & \scalebox{0.8}{$(8.44, 703.97)$}    & 96.04  & $1.2\cdot10^{-1}$  \\
& Object 3 & \scalebox{0.8}{$(500, 1500)$}  & \scalebox{0.8}{$(507.65, 1533.8)$}  & 34.66  & $2.2\cdot10^{-2}$  \\
\hline
\multicolumn{2}{|c|}{Visual Space} & $\fobs=(\vcu,\vcd)$ & $\hat{\fobs}=(\hat{\vcu},\hat{\vcd})$ & $e_\mathrm{a}=\|\hat{\fobs} - \fobs\|$ & $e_\mathrm{r}=\|\hat{\fobs} - \fobs\| / \|\fobs\|$ \\
\hline
\multirow{3}{*}{\rotatebox{90}{GoodSep}}
& Object 1 & \scalebox{0.8}{$(-0.3, 0.001)$}     & \scalebox{0.8}{$(-0.3008, 0.001)$}  & $7.87\cdot10^{-4}$  & $2.6\cdot10^{-3}$  \\
& Object 2 & \scalebox{0.8}{$(0.0125, 0.00125)$} & \scalebox{0.8}{$(0.0115, 0.00124)$} & $9.59\cdot10^{-4}$  & $7.6\cdot10^{-2}$  \\
& Object 3 & \scalebox{0.8}{$(0.3333, 0.00067)$} & \scalebox{0.8}{$(0.3302, 0.00064)$} & $31.21\cdot10^{-4}$ & $9.3\cdot10^{-3}$  \\
\hline
\multirow{3}{*}{\rotatebox{90}{PoorSep}}
& Object 1 & \scalebox{0.8}{$(-0.3, 0.001)$}     & \scalebox{0.8}{$(-0.299, 0.001)$}   & $1.02\cdot10^{-3}$  & $3.4\cdot10^{-3}$  \\
& Object 2 & \scalebox{0.8}{$(0.0125, 0.00125)$} & \scalebox{0.8}{$(0.0159, 0.00112)$} & $3.36\cdot10^{-3}$  & $2.6\cdot10^{-1}$  \\
& Object 3 & \scalebox{0.8}{$(0.6667, 0.00067)$} & \scalebox{0.8}{$(0.7131, 0.00238)$} & $4.95\cdot10^{-3}$  & $7.4\cdot10^{-2}$  \\
\hline
\multirow{3}{*}{\rotatebox{90}{PoorPrec}}
& Object 1 & \scalebox{0.8}{$(-0.3, 0.001)$}     & \scalebox{0.8}{$(-0.299, 0.0014)$} & $10.8\cdot10^{-4}$  & $3.5\cdot10^{-3}$  \\
& Object 2 & \scalebox{0.8}{$(0.0125, 0.00125)$} & \scalebox{0.8}{$(0.012, 0.00142)$} & $5.38\cdot10^{-4}$  & $4.3\cdot10^{-2}$  \\
& Object 3 & \scalebox{0.8}{$(0.3333, 0.00067)$} & \scalebox{0.8}{$(0.331, 0.00065)$} & $23.56\cdot10^{-4}$ & $7.1\cdot10^{-3}$  \\
\hline
\multicolumn{2}{|c|}{Audio Space} & $\sobs$ & $\hat{\sobs}$ & $e_\mathrm{a}=|\hat{\sobs} - \sobs|$ & $e_\mathrm{r}=|\hat{\sobs} - \sobs| / |\sobs|$ \\
\hline
\multirow{3}{*}{\rotatebox{90}{GoodSep}}
& Object 1 & $-49.71$ & $-49.8$ & $0.09$ & $1.9\cdot10^{-3}$  \\
& Object 2 & $-8.22$ & $-8.35$ & $0.13$  & $1.6\cdot10^{-2}$  \\
& Object 3 & $34.75$ & $34.37$ & $0.38$  & $1.1\cdot10^{-2}$  \\
\hline
\multirow{3}{*}{\rotatebox{90}{PoorSep}}
& Object 1 & $-49.71$ & $-49.59$ & $0.12$ & $2.3\cdot10^{-3}$  \\
& Object 2 & $-8.22$ & $-7.76$  & $0.46$  & $5.6\cdot10^{-2}$  \\
& Object 3 & $-0.66$ & $-0.02$  & $0.65$  & $9.7\cdot10^{-1}$  \\
\hline
\multirow{3}{*}{\rotatebox{90}{PoorPrec}}
& Object 1 & $-49.71$ & $-49.49$ & $0.22$ & $4.4\cdot10^{-3}$  \\
& Object 2 & $-8.22$ & $-8.28$ & $0.06$   & $7.6\cdot10^{-3}$  \\
& Object 3 & $34.75$ & $34.47$ & $0.29$   & $8.3\cdot10^{-3}$  \\
\hline
\end{tabular}
\end{center}
\end{table*}

We compared the performance of our algorithm for the three object
configurations. For each of them, we computed absolute and relative
errors for the object parameter estimations in the different coordinate
systems (object, audio and visual spaces). The averages were taken
over 10 runs of the algorithm for different PSC
initializations, as described above. The results are reported in
Table~\ref{tab:simures}. We give object location estimates
$\hat{\param}=(\hat{\rcx}, \hat{\rcz})$,
$\hat{\fobs}=(\hat{\vcu},\hat{\vcd})$ and
$\hat{\sobs}$ in parameter, visual and audio spaces respectively.
It appears that the localization precision is quite high. In a
realistic setting such as that of Section~\ref{section:experimental-validation}, the
measurement unit can be set to a millimeter. In that case, the
observed precision, in a  well-separated objects configuration, it
is at worse about 6cm. However, precision in the  $\rcz$ coordinate is quite
sensible to the variance of the visual data and the object
configuration. To  get a better idea of the relationship between
the variance in object space and the variance in visual space,
$\fosm^{-1}$ can be replaced by its linear approximation  given by
a first order Taylor expansion. Assuming then that visual data are
distributed according to some probability distribution with mean
$\mu_\fosm$ and variance $\Sigma_\fosm$, it follows that through
the linear approximation of $\fosm^{-1}$, the variance in object
space is
$\frac{\partial\fosm^{-1}(\mu_\fosm)}{\partial\fobs}\fvar_\fosm\frac{\partial\fosm^{-1}(\mu_\fosm)}{\partial\fobs}\tp$.
Then, the  $\rcz$ coordinate covariance for an object $\cind$ is
approximately proportional to the $\vcd$ covariance for the
object multiplied by $\rcz_\cind^4$. For distant objects, a very
high precision in $\vcd$ is needed to get a satisfactory precision
in $\rcz$. At the same time we observe that the likelihood of the
estimate configuration often exceeds the likelihood for real
parameter values. This suggests that the model performs well for
the given data, but cannot get better precision than that imposed
by the data.

\paragraph{Selection.}
To select the optimal number of clusters $N$ we applied the BIC criterion~(\ref{eq:bic-criterion})
to the models, trained for that $N$. The BIC score graphs are shown on Figure~\ref{fig:simuBIC}.
The total number of objects $N$ is correctly determined in all the 3 cases of object
configurations, from which we conclude that the BIC criterion provides reliable
model selection in our case.



\section{Experiments with Real Data}
\label{section:experimental-validation}

In this section we evaluate
the effectiveness of our algorithms in
estimating the 3D locations of AV objects, i.e., a person localization
task. The examples used below are from a database of realistic AV scenarios
described in detail in~\citep{arnaud08cava}.

The experimental setup consists of a \textit{mannequin} equipped
with a pair of microphones fixed into its ears and a pair of
stereoscopic cameras mounted onto its forehead (this device was
developed within the
POP\footnote{\url{http://perception.inrialpes.fr/POP/}} project).
Each data set comprises two audio tracks, two image sequences, as
well as the calibration information. All the recordings were
performed in an ordinary room with no special adjustments to its
acoustics or appearance. Thus the data contain both visual
background information, and auditory noise, reverberations in
particular. This configuration best mimics what a person would
hear and see in a standard indoor environment.

We tested our multimodal clustering method with two scenarios: a
\textit{meeting} and a \textit{cocktail party}, Table~\ref{tab:scenarios}:
\begin{itemize}
 \item The meeting scenario\footnote{
    \url{http://perception.inrialpes.fr/CAVA_Dataset/Site/data.html#M1}}
    is a recording of a discussion held by five persons sitting around a table,
    only three of them being visible. It lasts 25 seconds and contains
    a total of about 8000 visual and 600 audio observations.
    The three visible persons perform head and body movements while
    taking speech turns. Sometimes two persons (visible or not) speak simultaneously.
 \item The cocktail party scenario\footnote{
    \url{http://perception.inrialpes.fr/CAVA_Dataset/Site/data.html#CTMS3}}
    shows a dynamic scene with three persons walking in a room and
    taking speech turns. Occasionally, one speaker is hidden by another
    person and two persons may speak simultaneously. Speakers may go
    in and out of the two cameras field of view. Moreover,
    there are sounds emitted by the persons' steps.
    The recording lasts 30 seconds and contains a total of about
    12500 visual and 3400 audio observations.
\end{itemize}
\begin{table}[tb]
  \begin{center}
  \begin{tabular}{|c|l|l|l|l|l|l|}
  \hline
  scenario & \begin{minipage}[c]{4em}
        \setlength{\baselineskip}{.7\baselineskip}
        visible persons
             \end{minipage}
  & \begin{minipage}[c]{4em}
    \setlength{\baselineskip}{.7\baselineskip}
   speaking persons
    \end{minipage}
  & \begin{minipage}[c]{5em}
    \setlength{\baselineskip}{.7\baselineskip}
   visual background
    \end{minipage}
  &
  \begin{minipage}[c]{4em}
    \setlength{\baselineskip}{.7\baselineskip}
   audio noise
    \end{minipage}
  & \begin{minipage}[c]{4em}
    \setlength{\baselineskip}{.7\baselineskip}
   occluded speakers
    \end{minipage}
  & \begin{minipage}[c]{4em}
    \setlength{\baselineskip}{.7\baselineskip}
   audio overlap
    \end{minipage}
  \\
  \hline
  meeting & 3 & 5 & yes & yes & no & yes \\
  \hline
  \begin{minipage}[c]{4em}
        \setlength{\baselineskip}{.7\baselineskip}
        cocktail party
             \end{minipage}
  & 3 & 3 & yes & yes & yes & yes \\
  \hline
  \end{tabular}
  \end{center}
  \caption{\label{tab:scenarios} Summary of the main characteristics
    of the two scenarios used to evaluate the multimodal clustering
    algorithm.}
\end{table}

\subsection{Preprocessing and Algorithm Initialization}

Visual observations, $\mathbf{f}$, are obtained as follows. First we
detect points of interest (POI) in both the left and right images and we
select those points that correspond to a moving scene object. Second
we perform stereo matching such that a disparity value is associated
with each matched point.

In practice we used the POI detector described in
\citep{harris88combined}. This detector is known to have high
repeatability in the presence of texture and to be photometric
invariant. We analyse each image point detected this way and we
select those points associated with a significant motion pattern.
Motion patterns are obtained in a straightforward manner. A
temporal intensity variance $\sigma_t$ is estimated at each POI.
Assuming stable lighting conditions, the POI belongs to a static
scene object if its temporal intensity variance is low and
non-zero due to a camera noise only. For image points belonging to
a dynamic scene object, the local variance is higher and depends
on the texture of the moving object and on the motion speed. In
our experiments, we estimated the local temporal intensity
variance $\sigma_t$ at each POI, from a collection of 5
consecutive frames. The point is labelled ``motion''  if
$\sigma_t>5$ (for 8-bit gray-scale images), otherwise it is
labelled as ``static''. The motion-labelled points are then
matched and the associated disparities are estimated using
standard stereo methods. In practice the results shown in this
paper are obtained with the method described in
\citep{hansard07patterns} using the INTEL's OpenCV camera
calibration software
\footnote{\url{http://www.intel.com/technology/computing/opencv}}.
Overall, this provides the $(\vcu,\vcv,\vcd)\tp$ to
$(\rcx,\rcy,\rcz)\tp$ mapping~(\ref{eq:vfuncdef}). Examples are
shown on Figure~\ref{fig:visobs}. Alternatively, we could have
used the spatiotemporal point detector described in
\citep{Laptev-IJCV-2005}. This methods is designed to detect
points in a video stream having large local variance in both the
spatial and temporal domains, thus representing abrupt events in
the stream. However, such points are quite rare in our dataset.

\begin{figure}[p]
\begin{center}
\subfigure[Meeting: There are five speakers but only three are
visible.]{\includegraphics[width=\textwidth, type=pdf, ext=.pdf,
  read=.pdf]{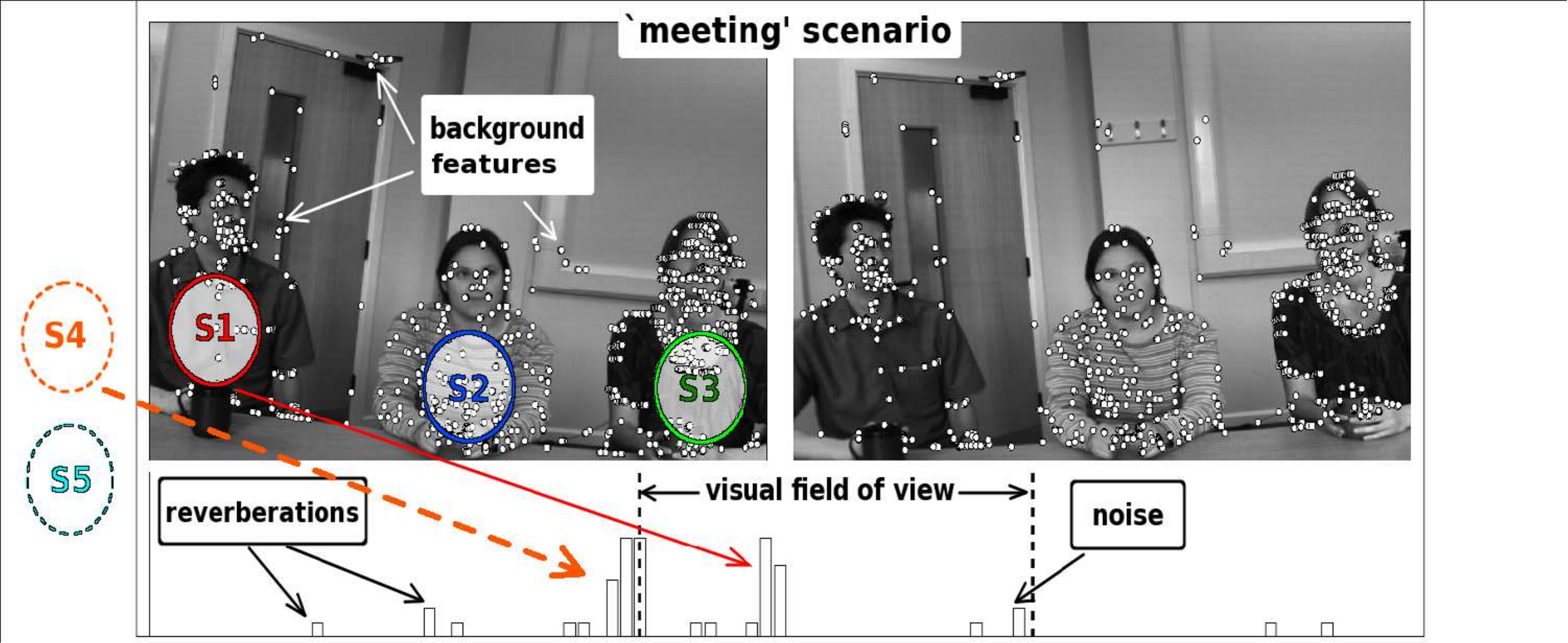}}
\subfigure[Cocktail party: The three speakers walk in the room.]{
  \includegraphics[width=\textwidth, type=pdf, ext=.pdf,
  read=.pdf]{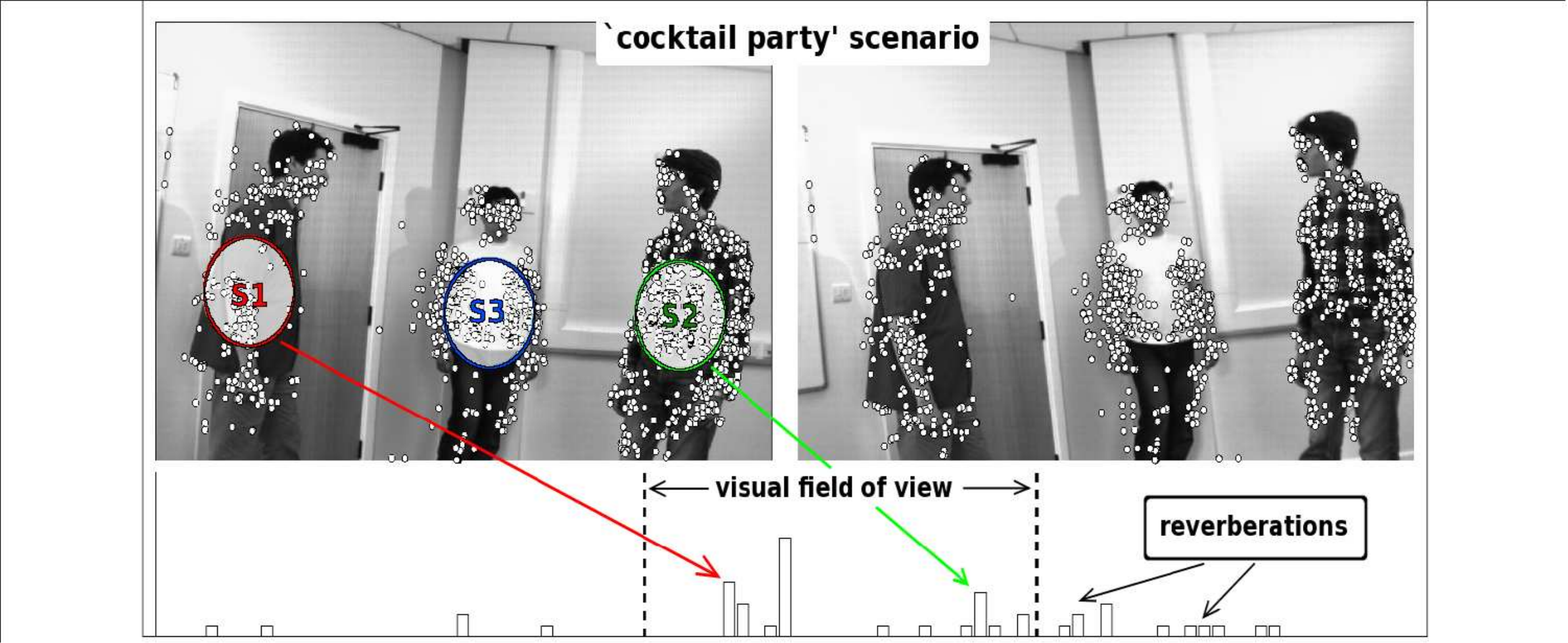}}
\end{center}
\caption{\label{fig:visobs} This figure illustrates how the
  audio-visual data are preprocessed. Visual points of interest (POI)
  associated with scene motion are matched between the left and right
  images. The histograms of the interaural time difference (ITD)
  observations correspond to a ``segment'' of 0.3 seconds. The
  audio-visual calibration allows us to filter out auditory data that
  falls outside the field of view of the two cameras. Notice the
  large number of auditory perturbations corresponding to noise,
  reverberations, as well as to speakers that are outside the visual
  field of view. In these examples, there are two simultaneous
  speakers: (a) S1 and S4 and (b) S1 and S2. Notice that S4 is easily
  eliminated because its associated ITD falls outside the visual field of view.
}
\end{figure}

Auditory observations, $\mathbf{g}$, are obtained as follows. Our
method uses \textit{interaural time differences} (ITD) which are
detected through the analysis of the cross-correlogram of the
filtered left- and right-microphone one-dimensional signals for
every frequency band~\citep{christensen07integrating}. Like any
other audio-visual fusion method, one needs to perform
audio-visual calibration, namely to estimate the positions of the
microphones and the positions and orientations of the cameras in a
common world coordinate system. This is done using the method
described in \citep{arnaud08cava}.

In order to initialize the algorithm's parameter values we used the Parameter Space Candidates
(PSC) initialization strategy described in Section~\ref{sec:init}.
Although real-data distributions do not strictly correspond to the case of Gaussian mixtures,
the initialization strategy that we have adopted remains relevant.
This originates from the fact that parameter space sampling with
configuration restrictions plays the role of a global optimization method
similar to Monte-Carlo sampling in the method of generations~\citep{zhigljavsky08stochastic}.
It helps to avoid local maxima
and allows to quickly find a set of appropriate initial parameters. Local distribution density
modes occur to be good candidates to initialize cluster centers.
As in the case of simulated data,
we used the BIC score, i.e., Section~\ref{sec:bic} to select the optimal
number of audio-visual clusters.

\subsection{Results and Discussion}

The experimental validation described below was performed with two
goals in mind. Firstly, we wanted to check that our method was
stable and robust with real data gathered in complex situations,
that it correctly finds the number of clusters and that it
efficiently determines  the model's parameters, i.e., the 3D
positions of the audio-visual objects composing a scene. Secondly,
we wanted to test the model's capability to deal with dynamic
changes in the scene, yet in the presence of acoustic
noise/reverberations and visually occluded persons, etc. Below we
provide a detailed account of the results obtained with the
meeting and cocktail-party audio-visual sequences.

The audio-visual recordings are split into ``segments'', each segment
lasts 0.3 seconds. At 25 frames/second this corresponds to
approximately eight video frames. The initialization method described
in Section~\ref{sec:init} and the model selection method
described
in Section \ref{sec:bic} are combined and applied to the first segment
in order to find initial parameter values and to
estimate the number of components (the number of audio-visual objects)
to be used by the conjugate EM algorithm. Consequently, the parameters
estimated for one segment are used to initialize the parameters for
the next segment, while the number of components remains constant.

\begin{itemize}
 \item Quasi-static scene. The meeting situation corresponds to the
   well-separated case which is referred to as \textit{GoodSep}
   in the previous section.
The initialization strategy performs well and the candidate
configuration obtained by the initialization step is relatively
close to the optimal one found by the EM algorithm described in
detail in Section~\ref{sec:indep:em:gem}. In fact, the likelihood
evolution reported in Figure~\ref{fig:realres} shows that
convergence is reached in about 20 iterations of EM, which is
three times faster than in the simulated GoodSep case reported in
Figure~\ref{fig:lhoods}. The 3D position estimates are quite
accurate, in particular the natural alignment of the speakers
along the table is clearly seen in the $XZ$ plane. Even though in
practice, the data are not piecewise Gaussian and the outliers are
not uniformly distributed, our method performs quite well, which
illustrates its robustness when dealing with real-data
distributions. Figure~\ref{fig:m1res} shows sequential results
obtained in this case. The speech sources are correctly detected
even in the case when two persons are simultaneously active.

 \item Dynamic scene. The cocktail party situation corresponds to the
   partially occluded case which is referred to as \textit{PoorSep} in
   the previous section. In this case, the locations of the
   audio-visual objects varies over time, as well as their
   number. Nevertheless, we assume that these changes are rather
   slow. We did not attempt to tune our algorithm to the dynamic
   case. Hence, we use the same initialization strategy as in the
   quasi-static case  which is briefly summarized above.
   Figure~\ref{fig:visres} shows the results obtained in this case.
\end{itemize}

\begin{figure}[p!]
\begin{center}
     \includegraphics[width=0.48\textwidth, type=pdf, ext=.pdf, read=.pdf]{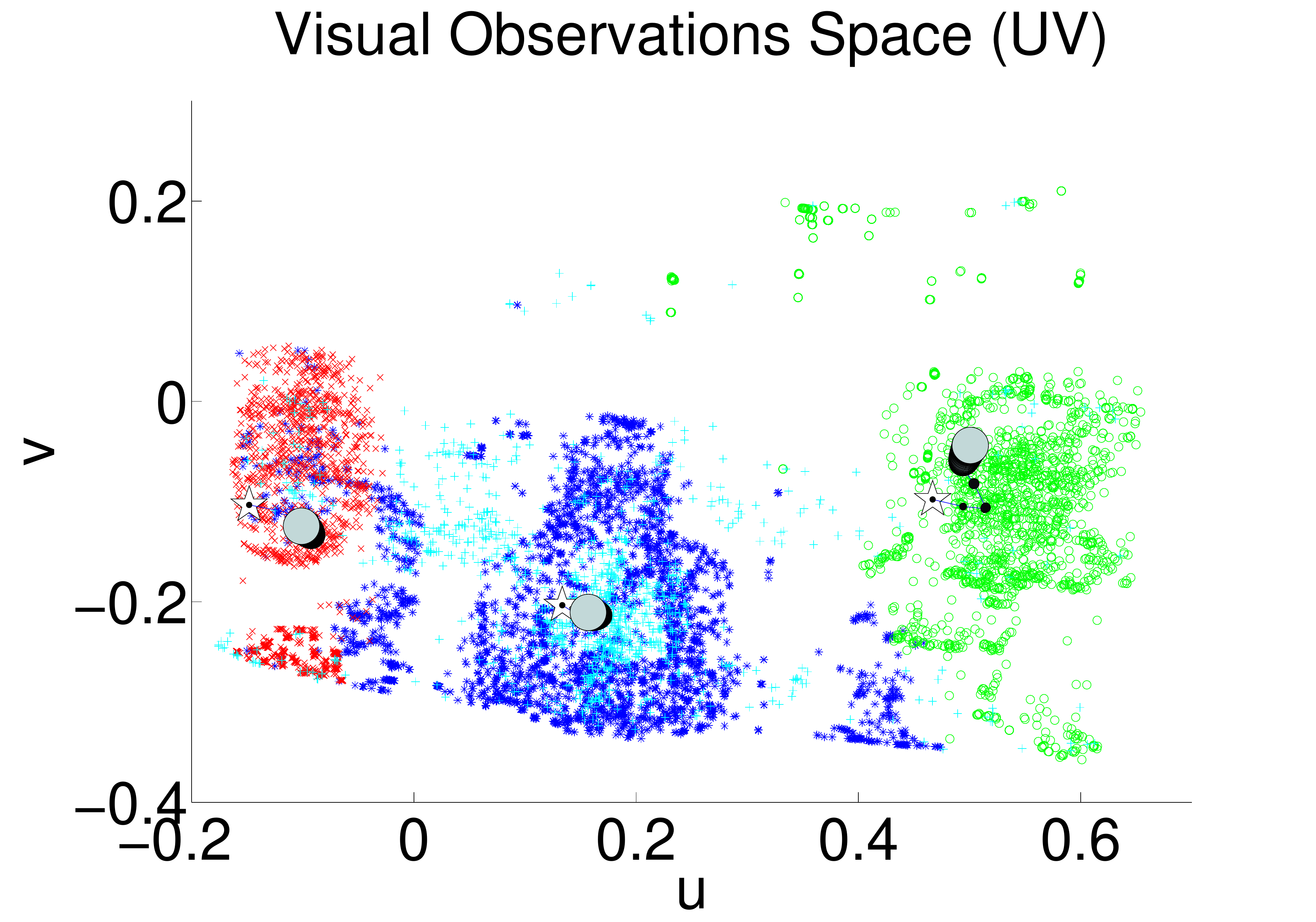}
     \includegraphics[width=0.48\textwidth, type=pdf, ext=.pdf, read=.pdf]{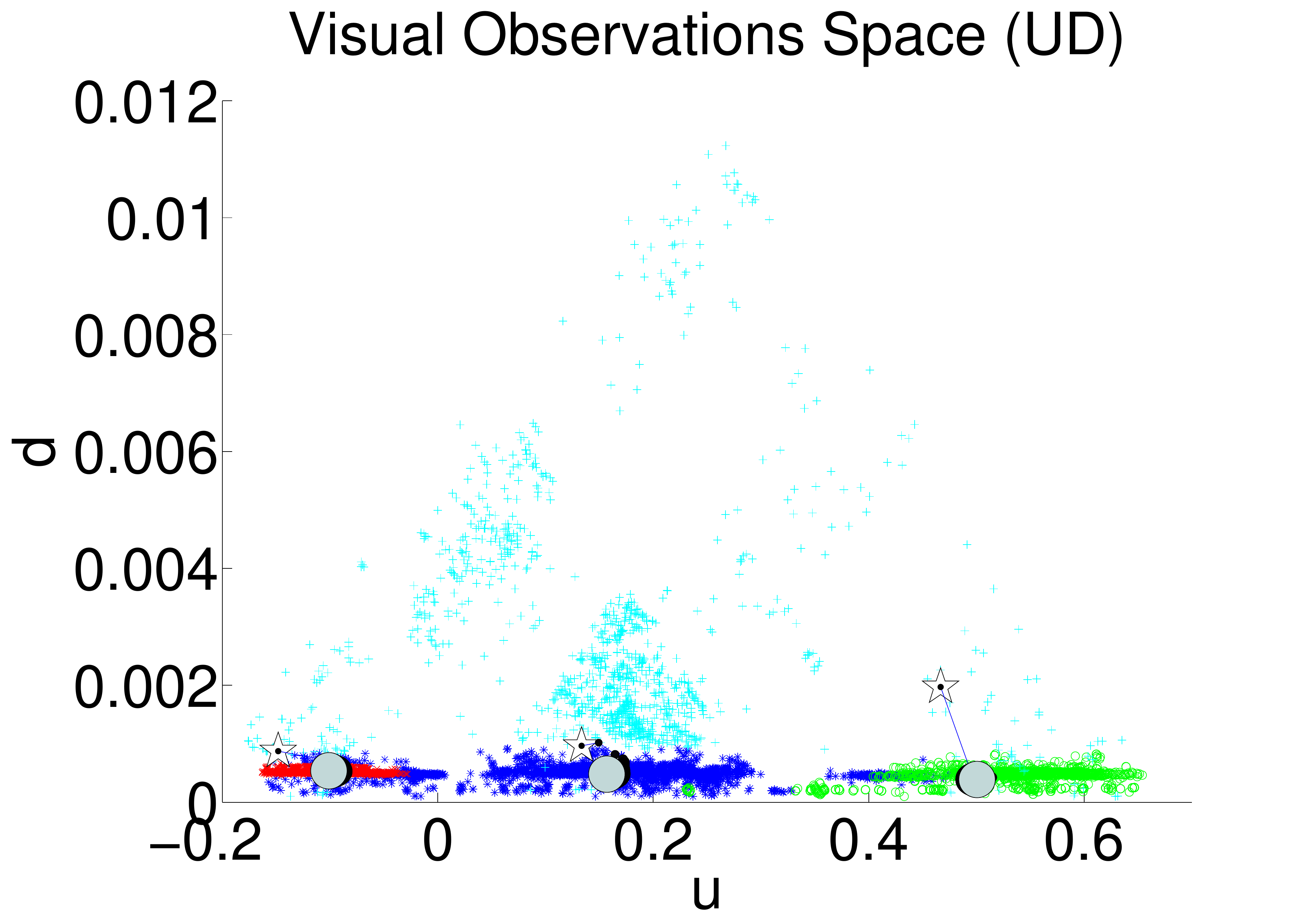}
     \includegraphics[width=0.48\textwidth, type=pdf, ext=.pdf, read=.pdf]{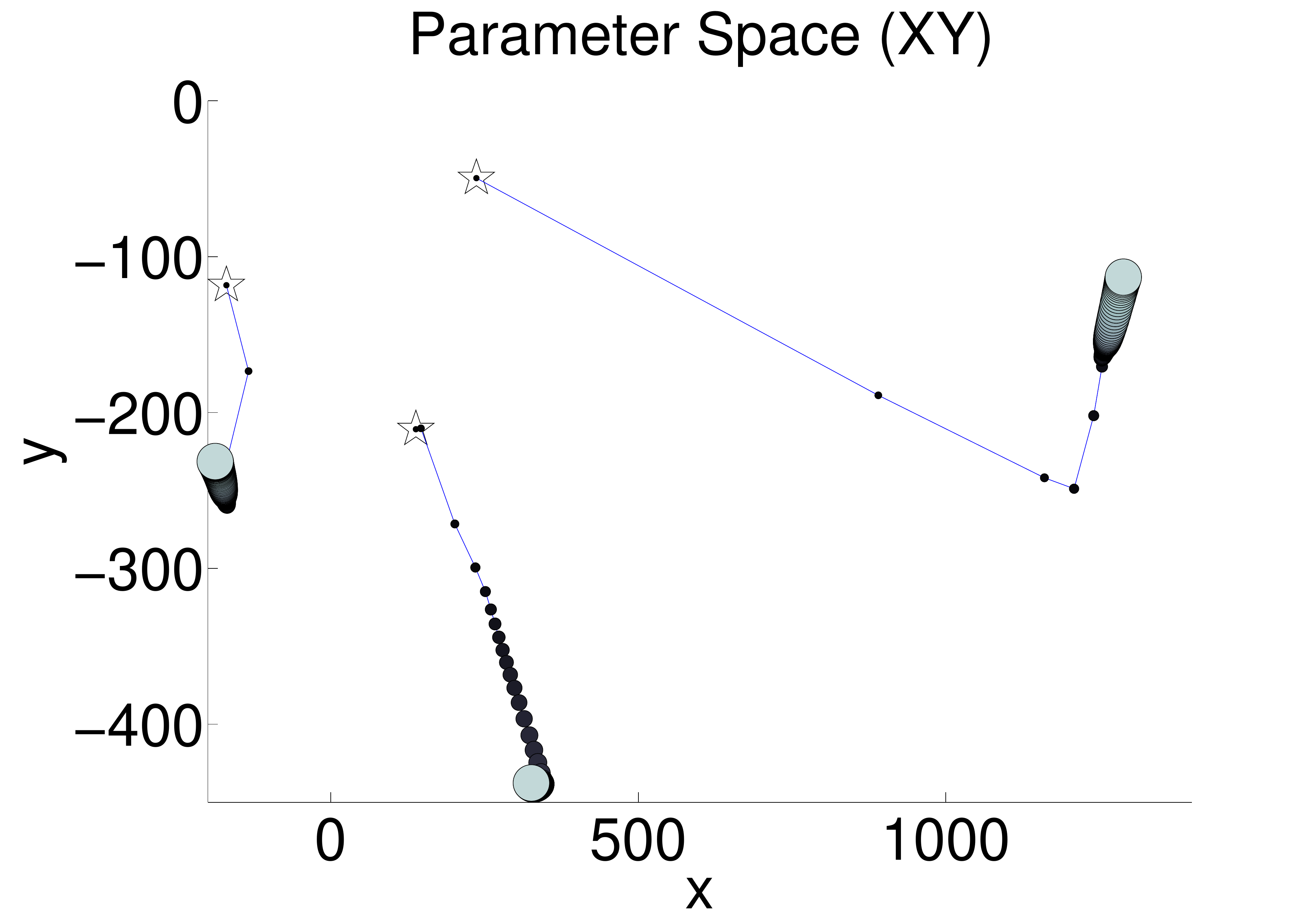}
     \includegraphics[width=0.48\textwidth, type=pdf, ext=.pdf, read=.pdf]{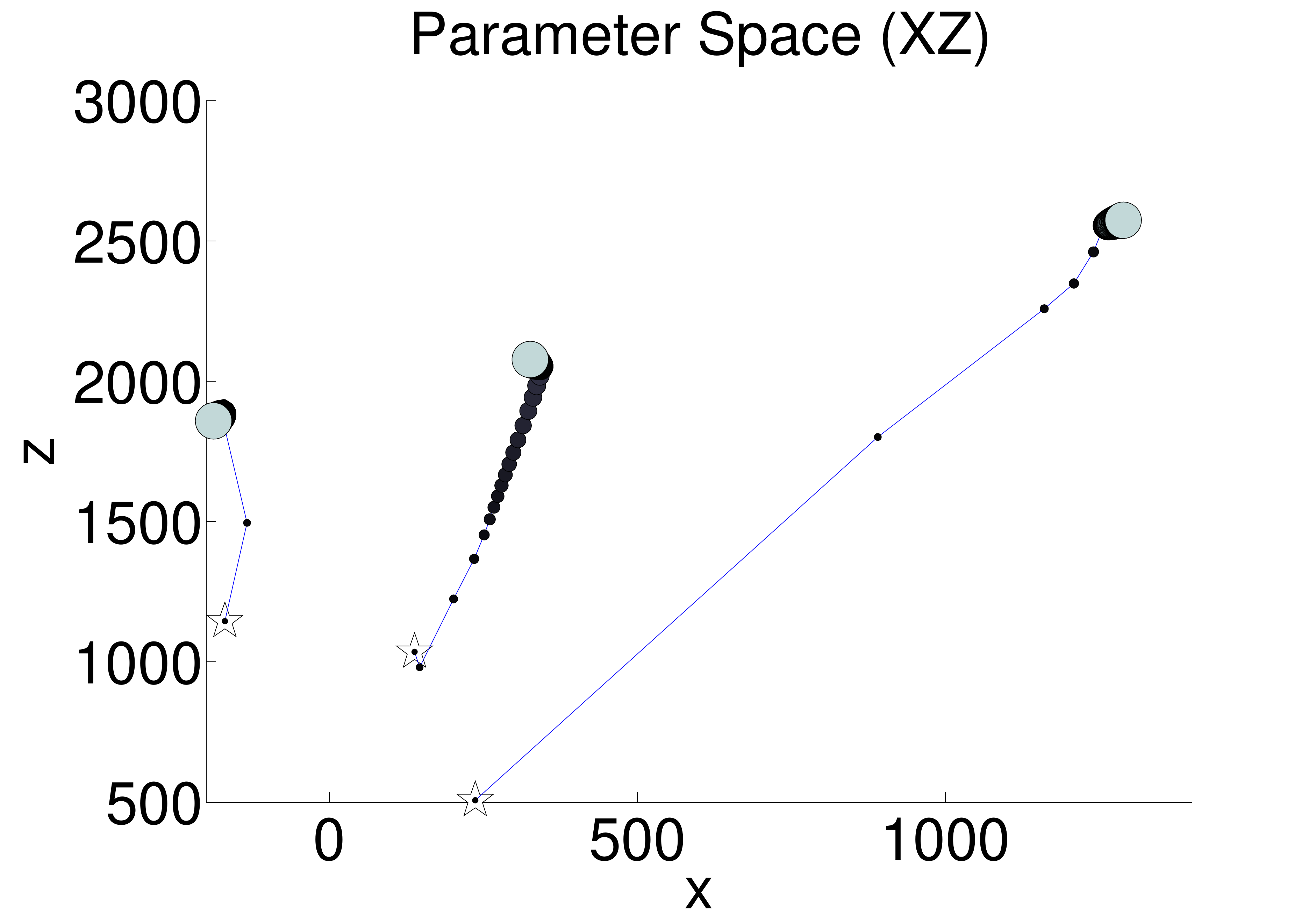}
     \includegraphics[width=0.48\textwidth, type=pdf, ext=.pdf, read=.pdf]{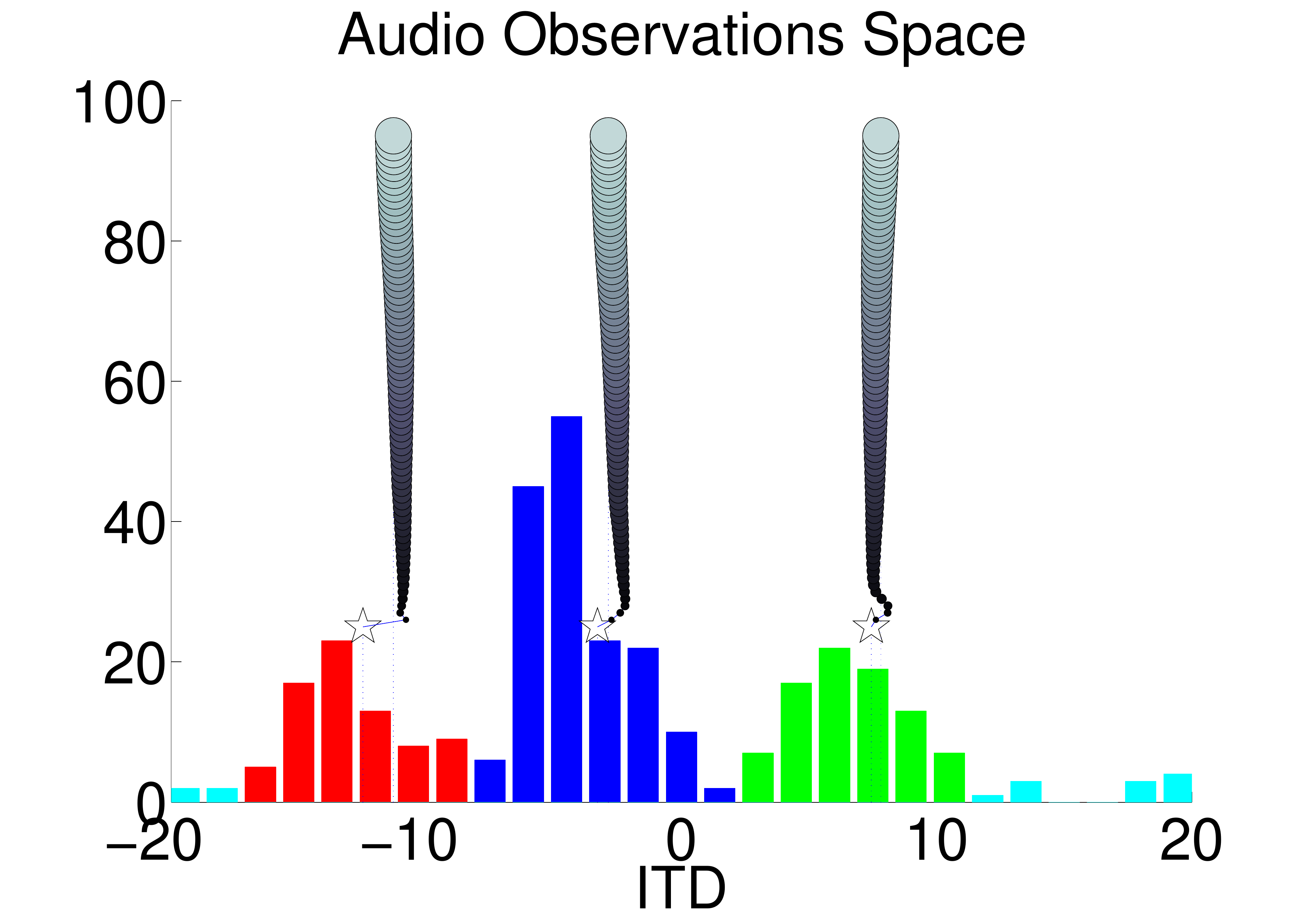}
     \includegraphics[width=0.48\textwidth, type=pdf, ext=.pdf, read=.pdf]{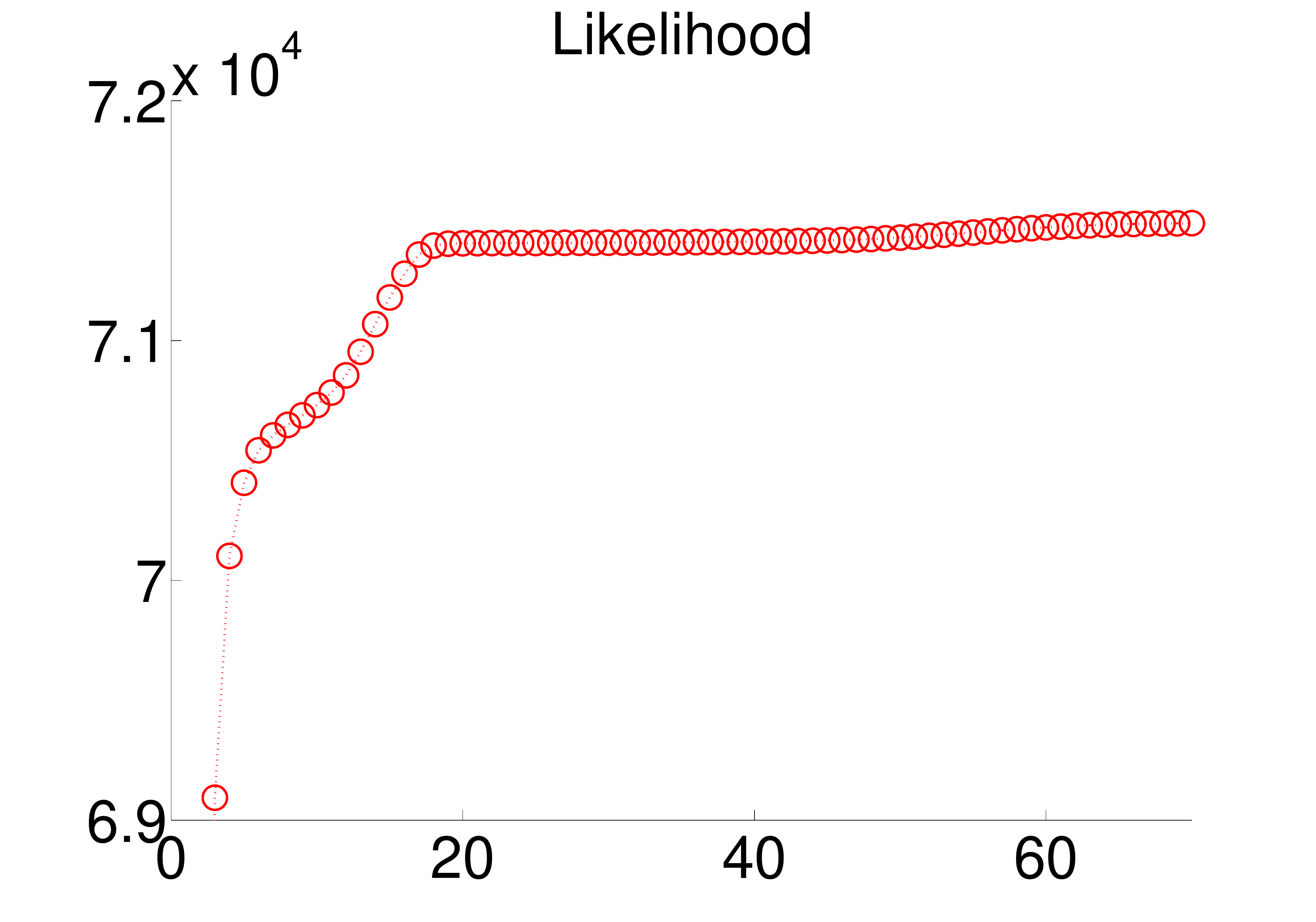}
\end{center}
\caption{An example of applying the proposed EM algorithm to a time interval of
  20 seconds of the meeting scenario. The results are shown in the
  visual and auditory observation spaces as well as in the parameter
  space. The initial parameter values are shown with three stars while
  the parameter evolution trajectories are shown with circles of
  increasing size. The final observation-to-cluster assignments are
  shown in color: red, blue, and green for the three Gaussian components
  and light-blue for the outlier component. The log-likelihood curve
  (bottom-right) shows that the algorithm converged after 20 iterations.
} \label{fig:realres}
\end{figure}

\begin{figure}[p!]
\begin{tabular}{ccc}
     \includegraphics[width=0.45\textwidth, type=pdf, ext=.pdf, read=.pdf]{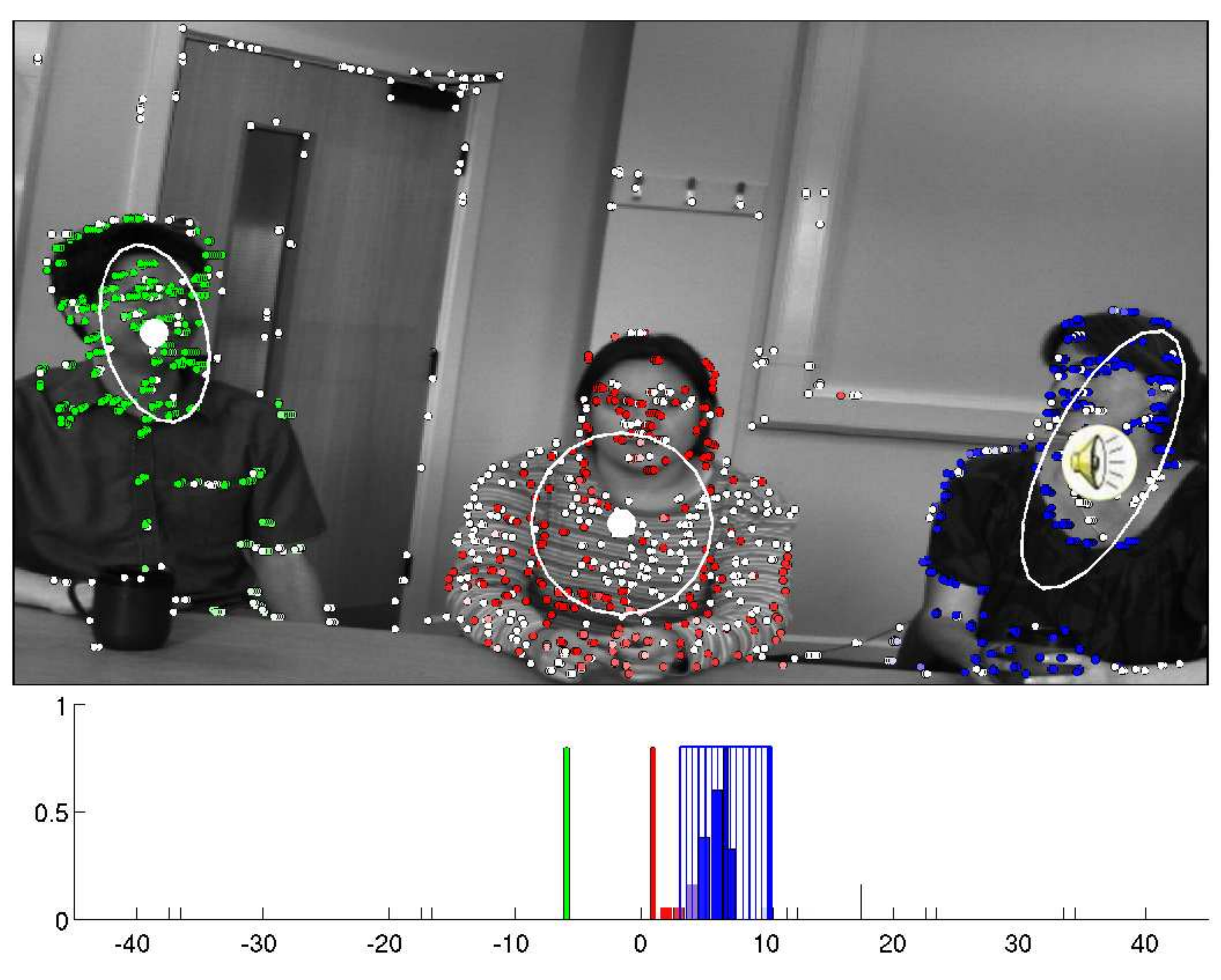} &
     \includegraphics[width=0.45\textwidth, type=pdf, ext=.pdf, read=.pdf]{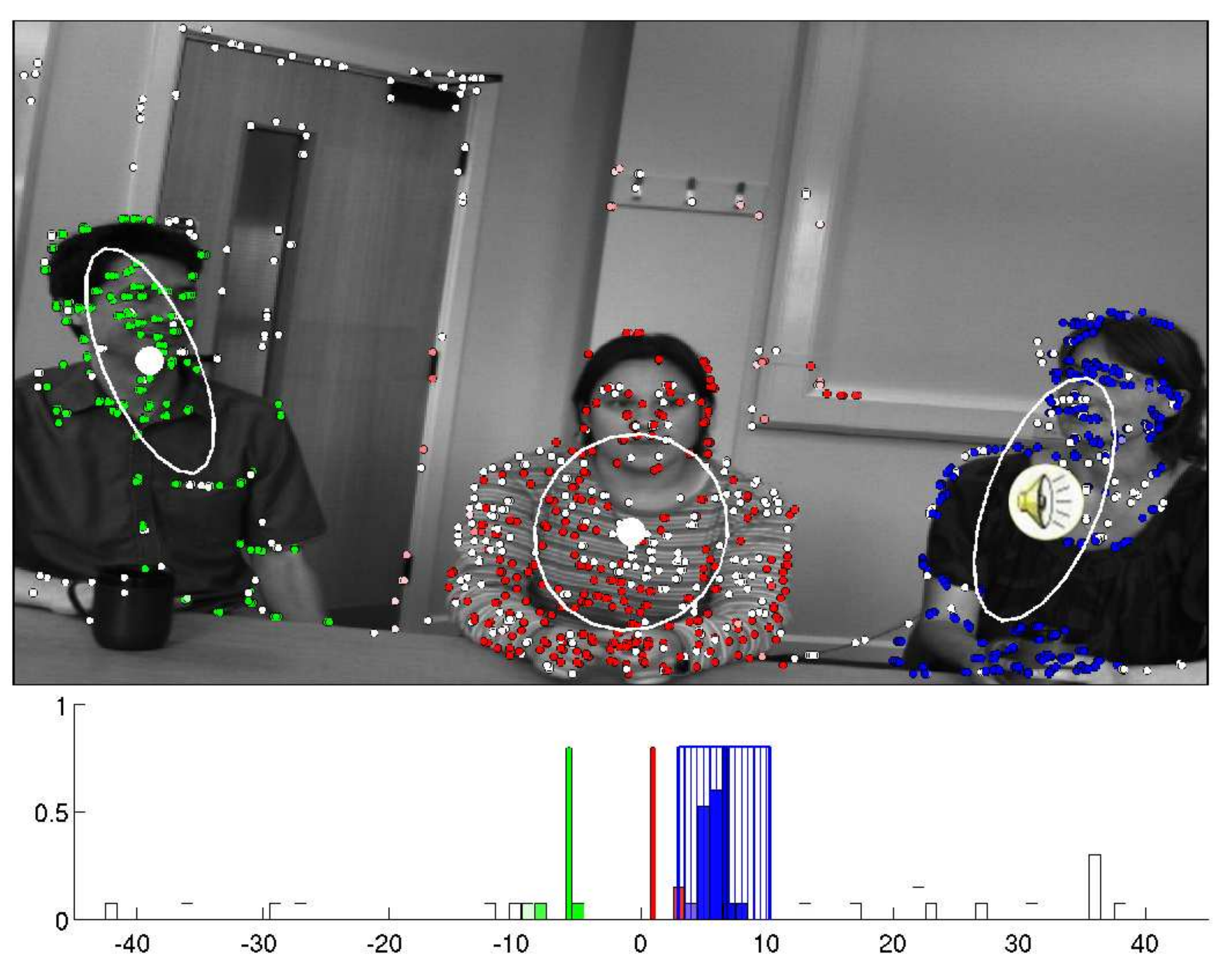} \\
      (a) frames 1001-1010 & (b) frames 1011-1020 \\
     \includegraphics[width=0.45\textwidth, type=pdf, ext=.pdf, read=.pdf]{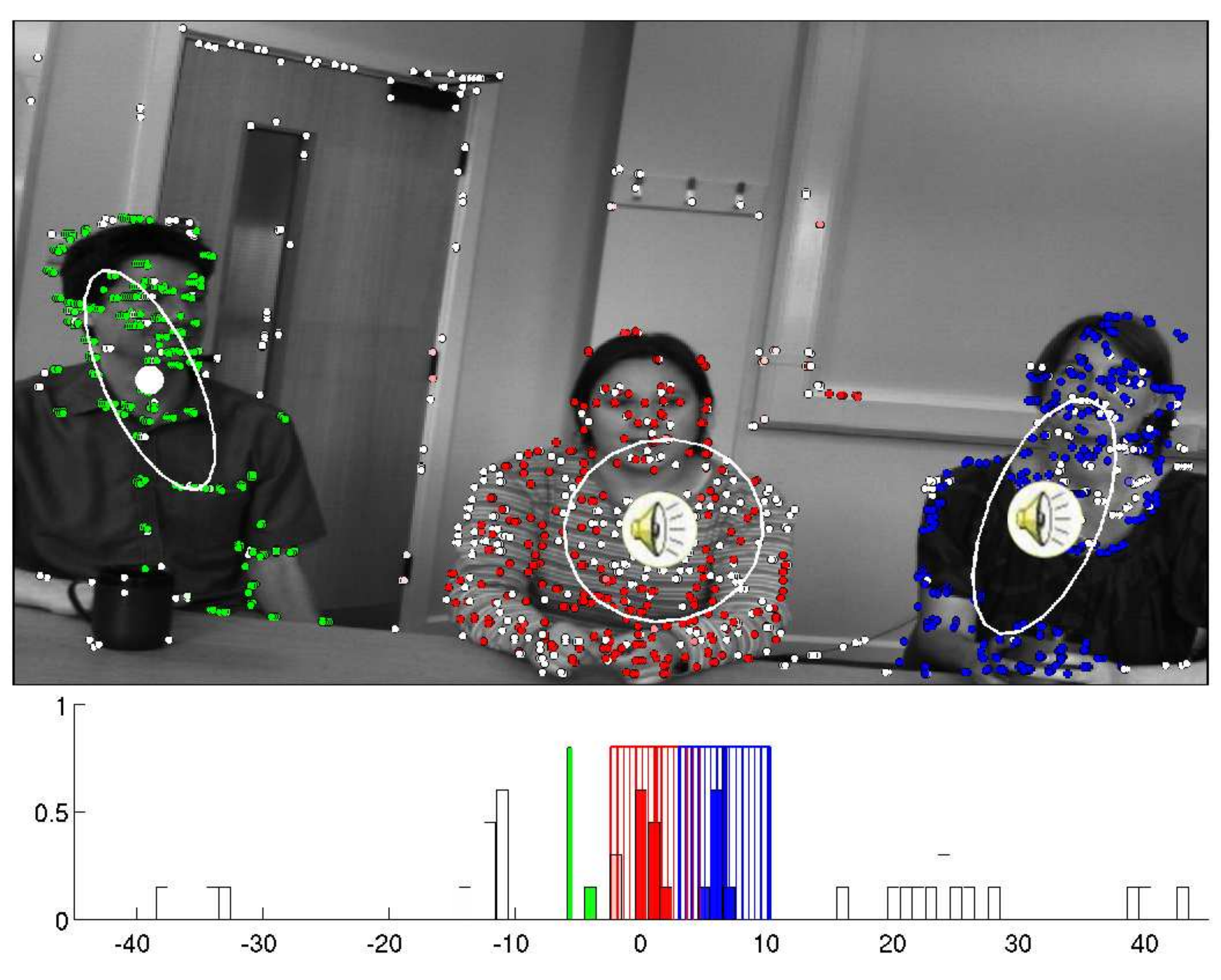} &
     \includegraphics[width=0.45\textwidth, type=pdf, ext=.pdf, read=.pdf]{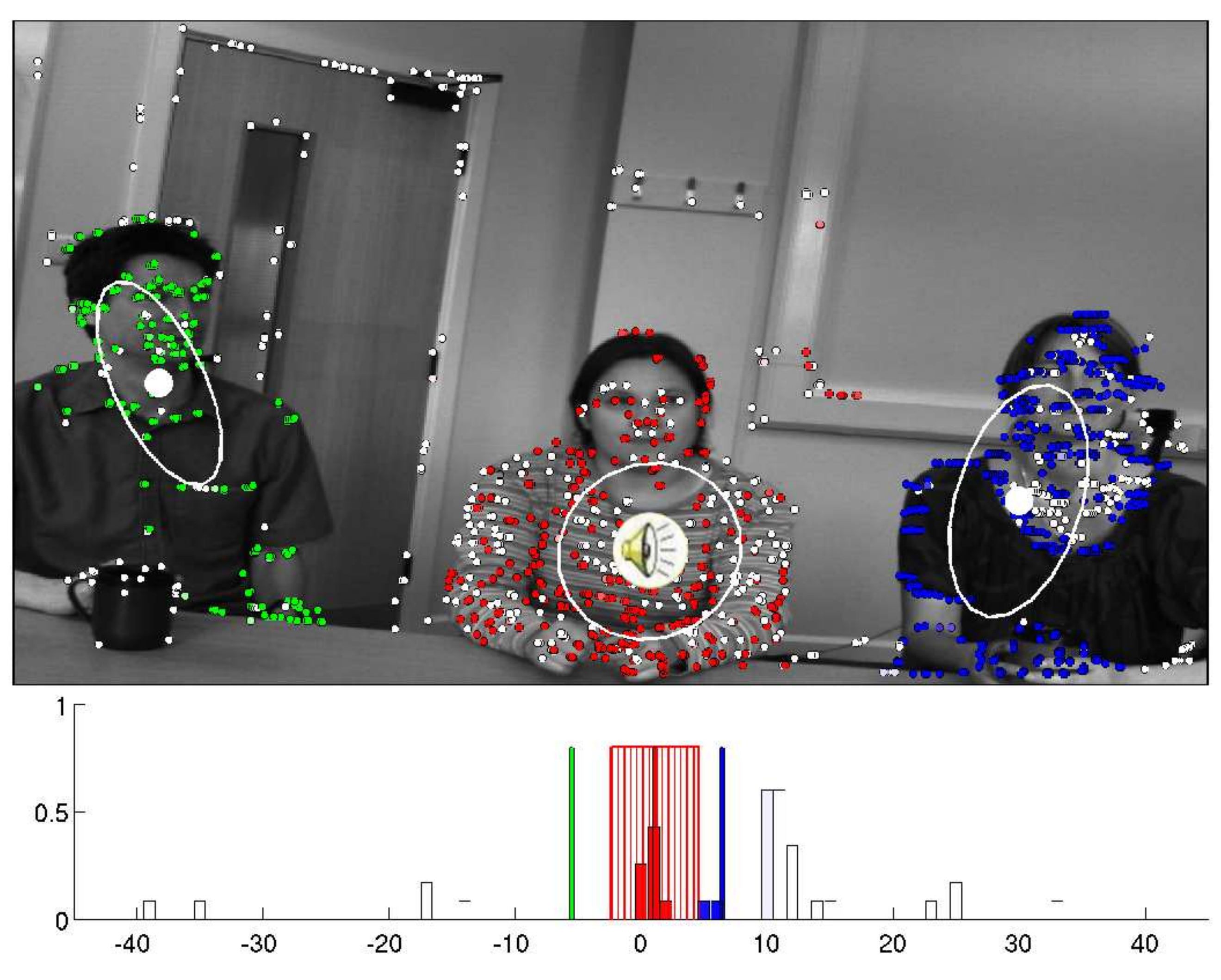} \\
      (c) frames 1021-1030 & (d) frames 1031-1040 \\
     \includegraphics[width=0.45\textwidth, type=pdf, ext=.pdf, read=.pdf]{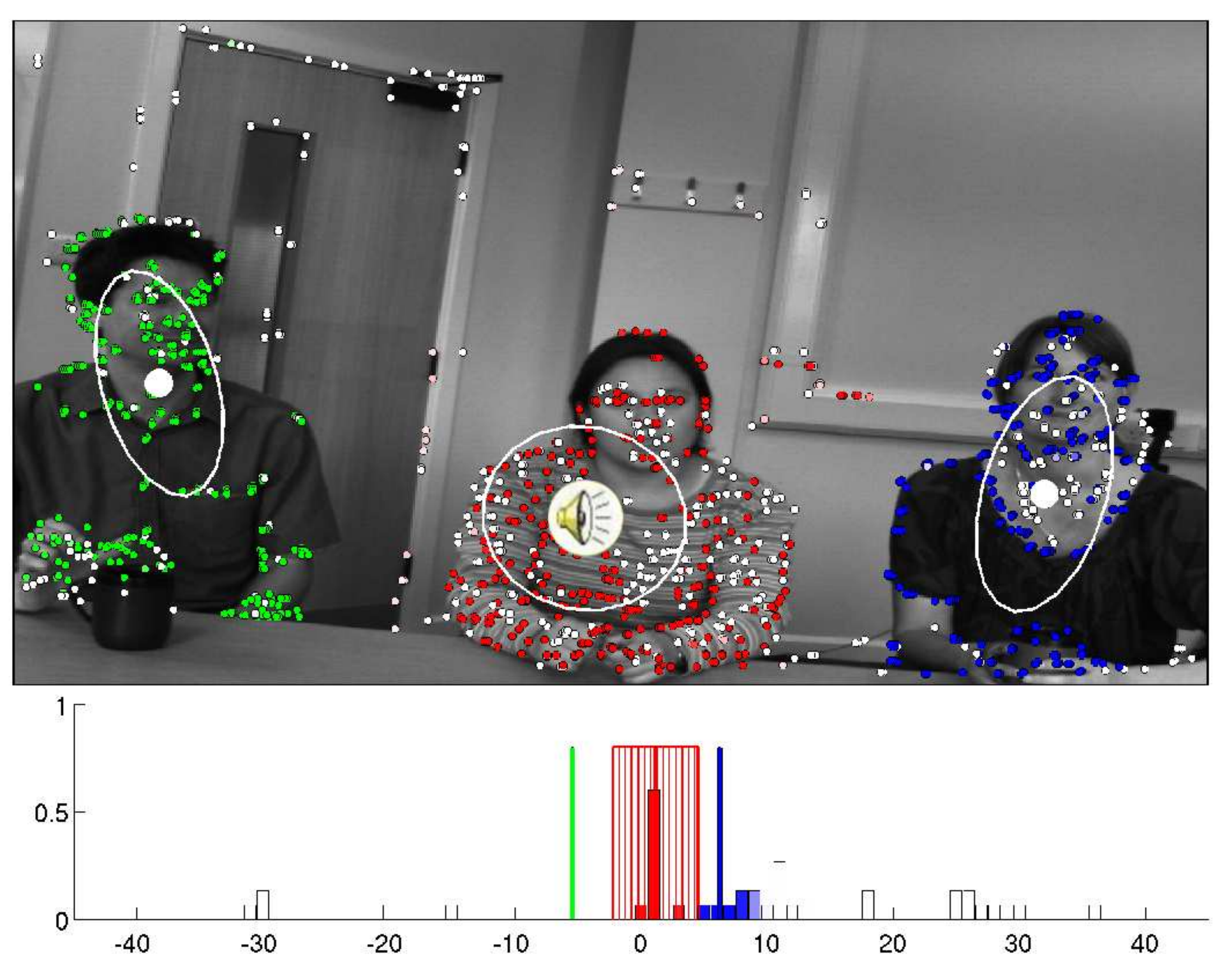} &
     \includegraphics[width=0.45\textwidth, type=pdf, ext=.pdf, read=.pdf]{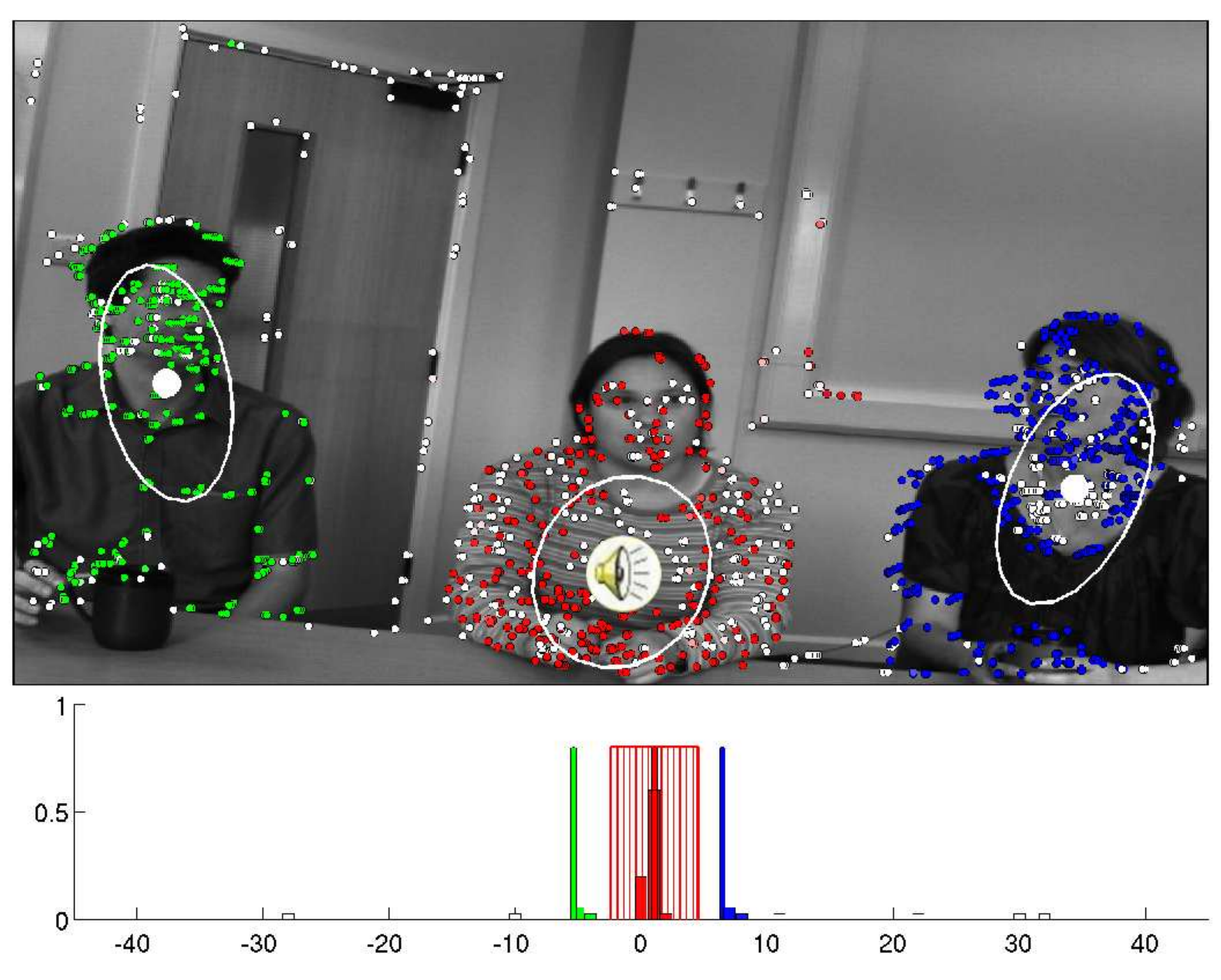} \\
     (e) frames 1041-1050 & (f) frames 1051-1060
\end{tabular}
\caption{ \label{fig:m1res} Results obtained in the case of the
meeting scenario shown overlapped onto the left image. Sixty
frames (1001 to 1060) were split into six segments. Parameter
initialization and model selection were performed on the first
segment (frames 1-10) and are not shown. The ``visual'' covariance
matrices associated with the 3 Gaussian components are projected
onto the image plane. The white dots correspond to the projected
3D locations estimated by the algorithm. The blue, green, and red
colors encode the observation-to-cluster assignments and the
active speaker is marked with a corresponding symbol. The
algorithm correctly estimates speech sources, even in the case
when two speakers are active. }
\end{figure}

\begin{figure}[p!]
\begin{tabular}{ccc}
     \includegraphics[width=0.45\textwidth, type=pdf, ext=.pdf, read=.pdf]{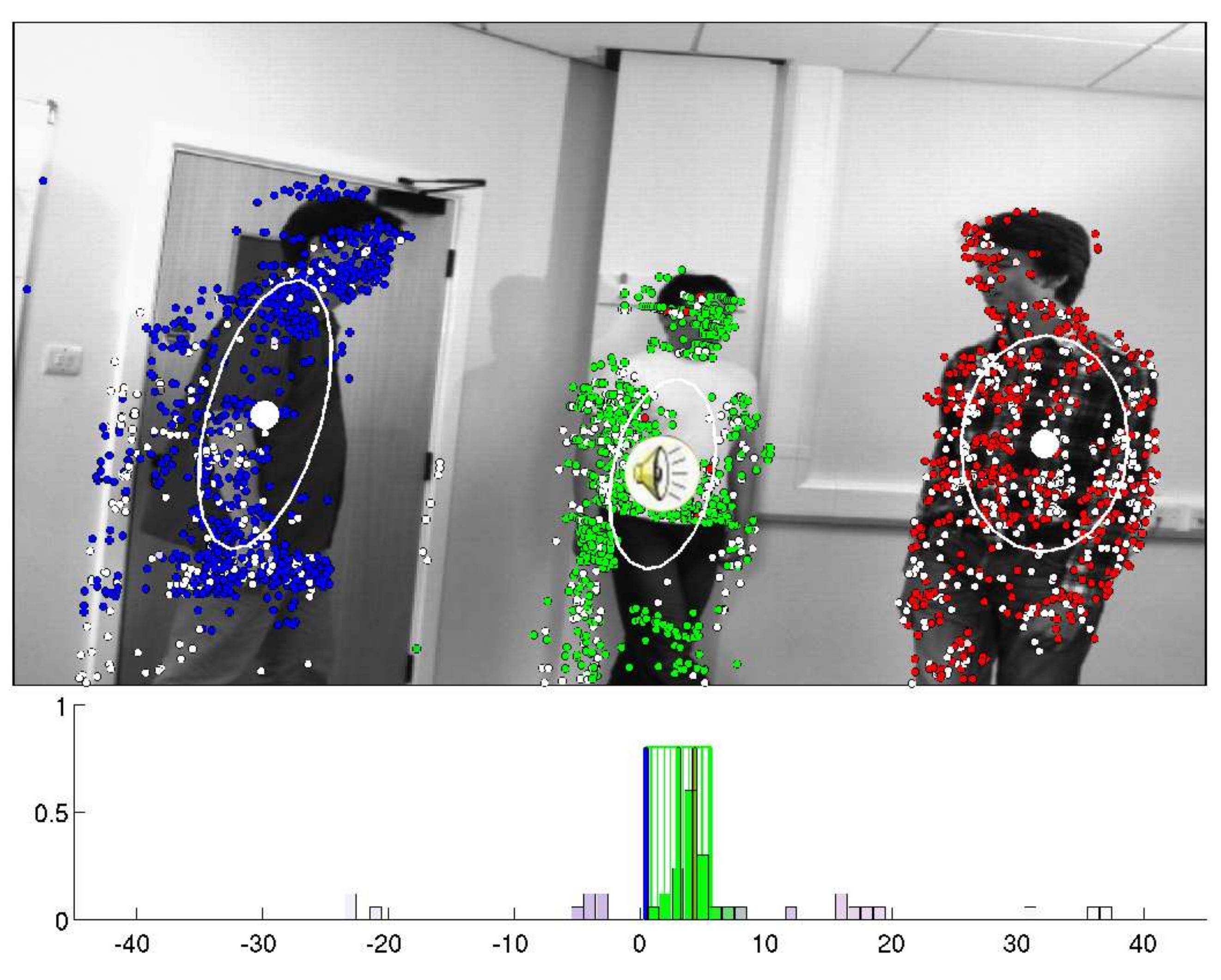} &
     \includegraphics[width=0.45\textwidth, type=pdf, ext=.pdf, read=.pdf]{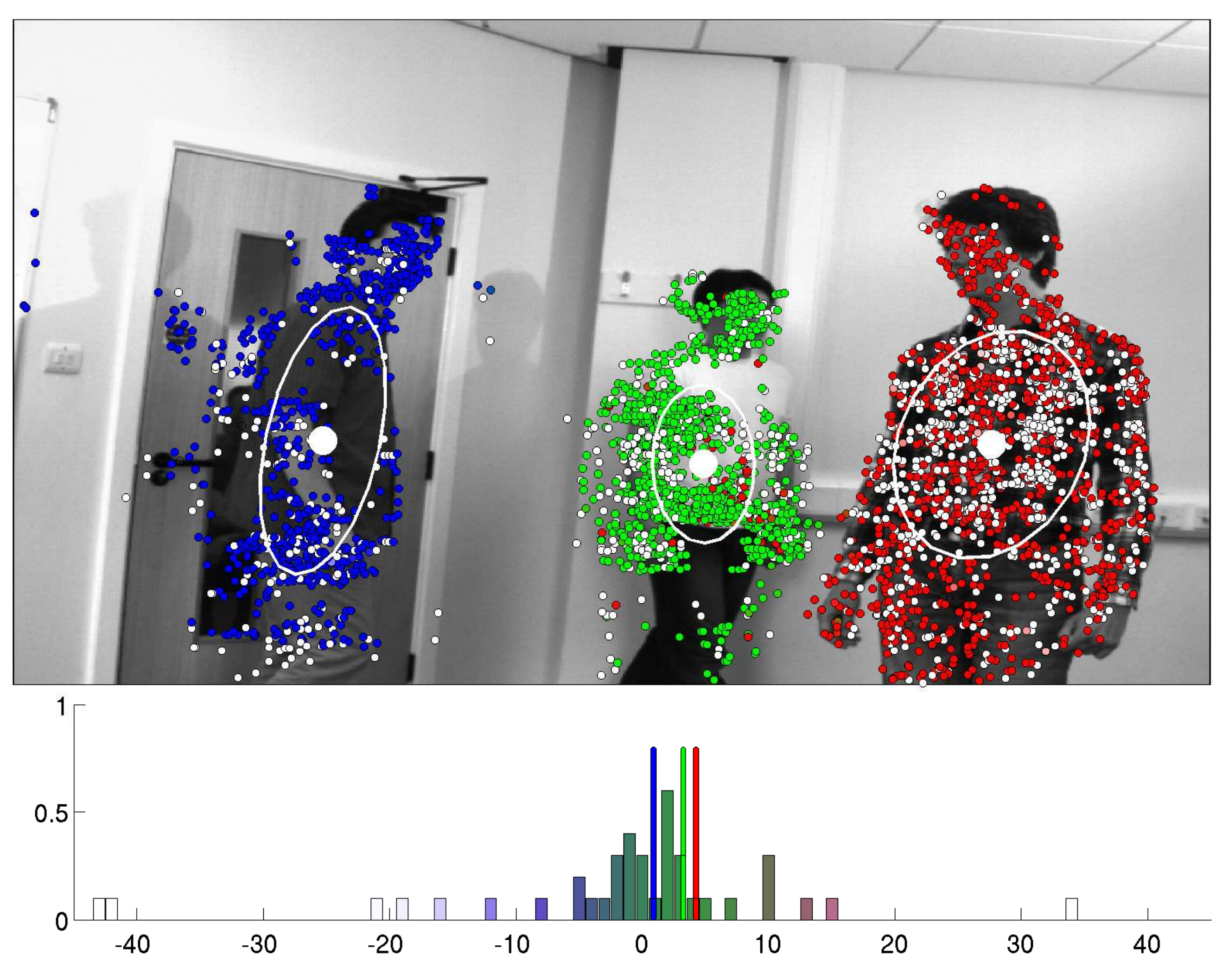} \\
      (a) frames 181-190 & (b) frames 191-200 \\
     \includegraphics[width=0.45\textwidth, type=pdf, ext=.pdf, read=.pdf]{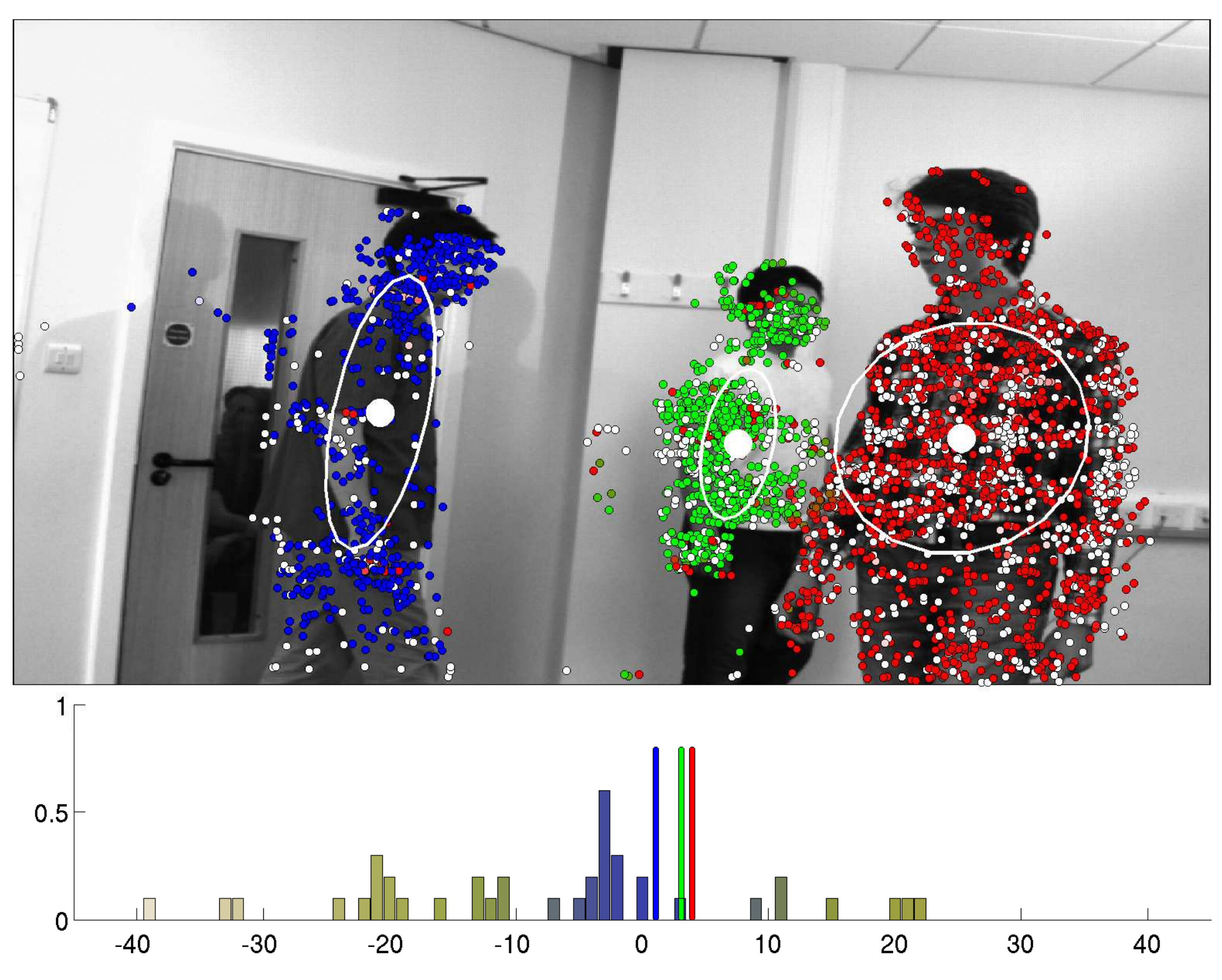} &
     \includegraphics[width=0.45\textwidth, type=pdf, ext=.pdf, read=.pdf]{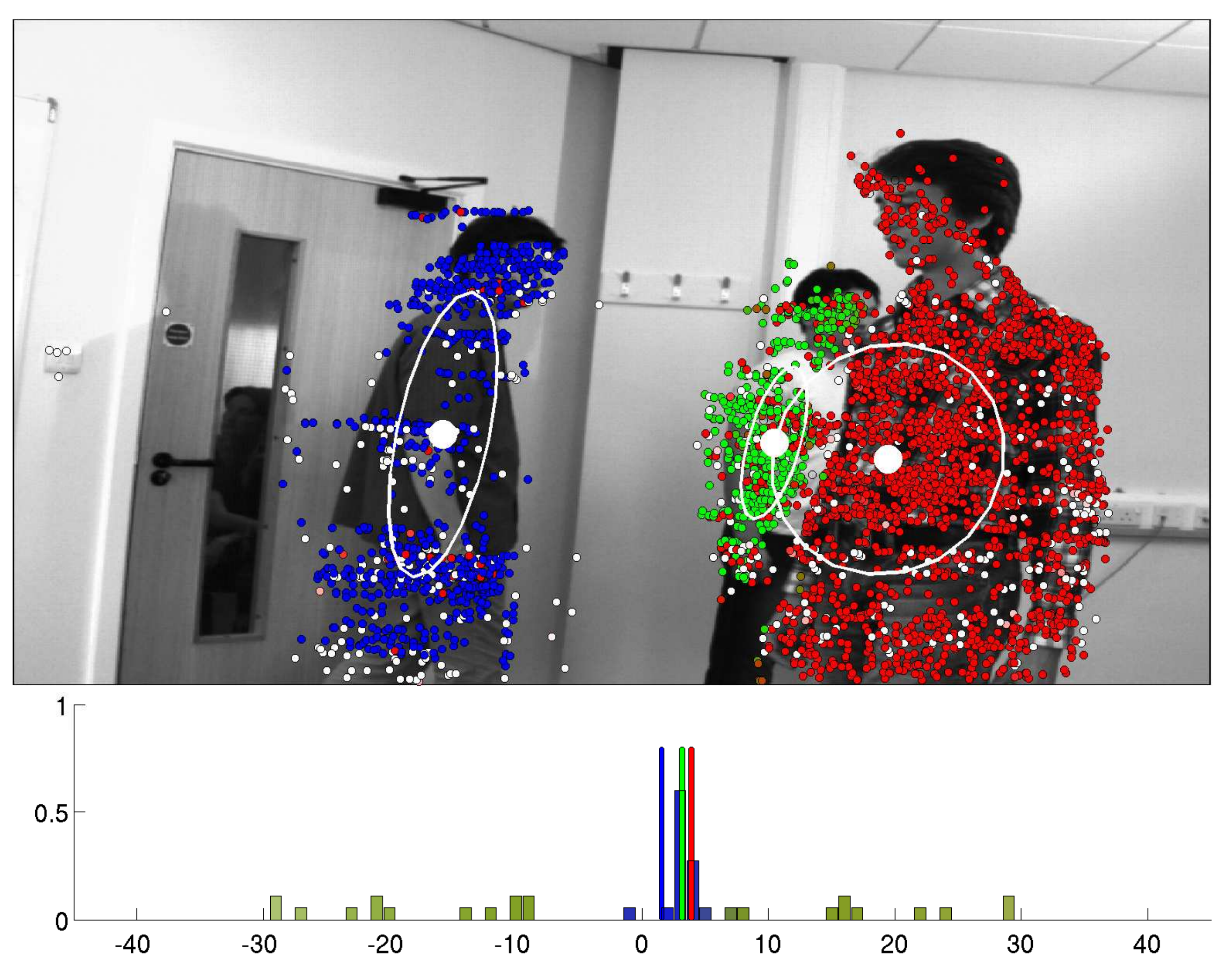} \\
      (c) frames 201-210 & (d) frames 211-220 \\
     \includegraphics[width=0.45\textwidth, type=pdf, ext=.pdf, read=.pdf]{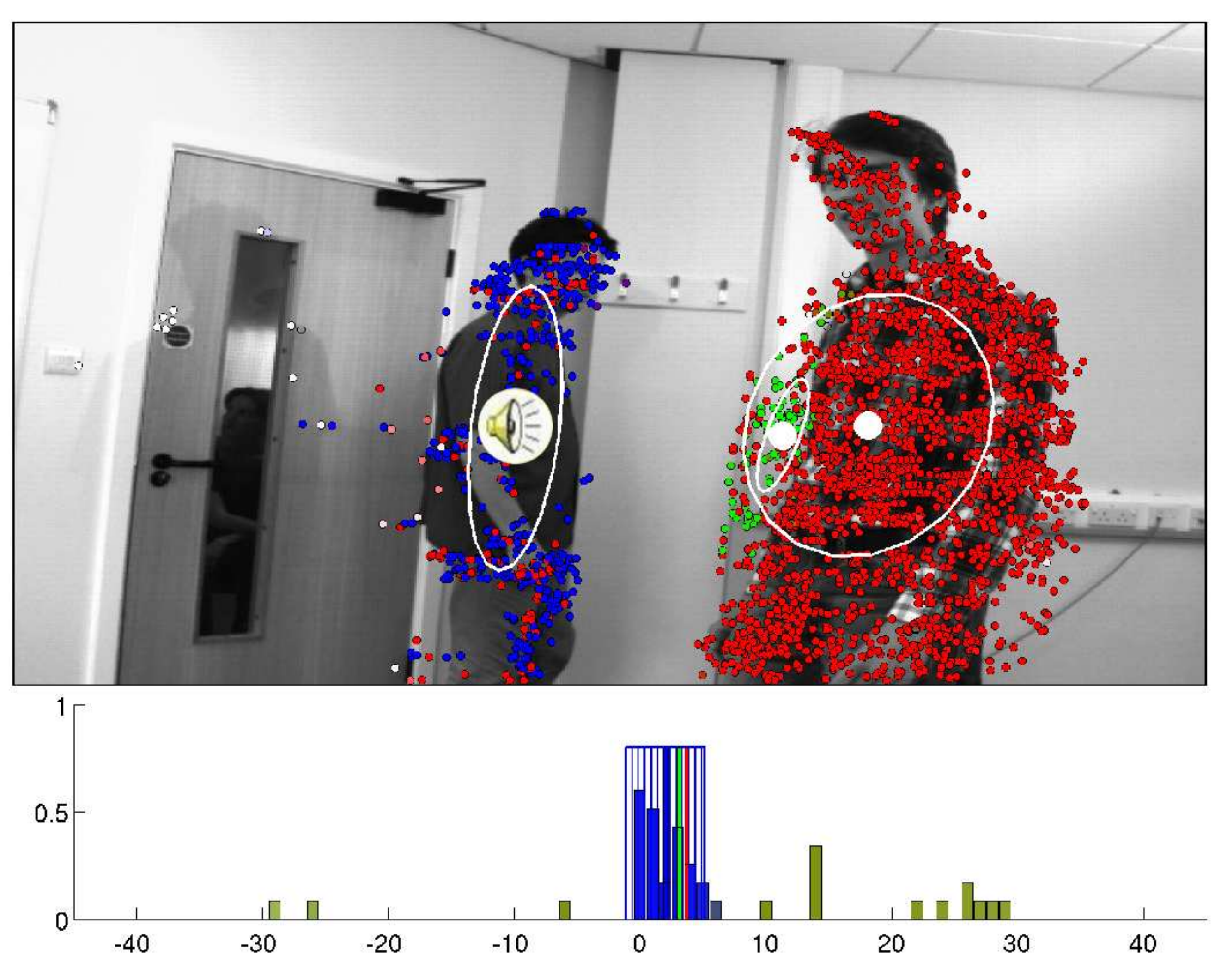} &
     \includegraphics[width=0.45\textwidth, type=pdf, ext=.pdf, read=.pdf]{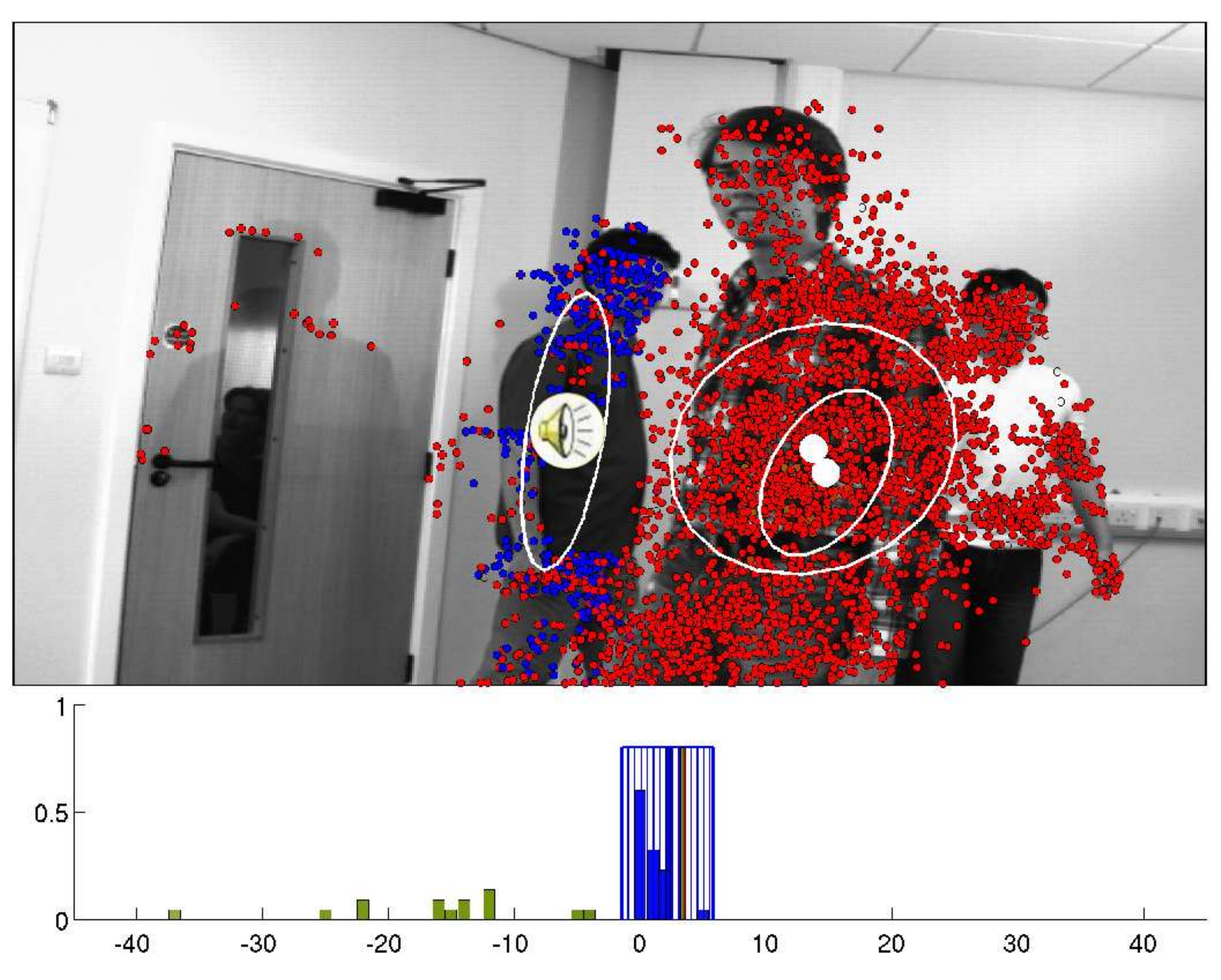} \\
     (e) frames 221-230 & (f) frames 231-240
\end{tabular}
\caption{ \label{fig:visres} Results obtained in the case of the
cocktail party scenario shown overlapped onto the left image. As
in the previous case, sixty frames (181 to 240) were split into
six segments. Parameter initialization and model selection were
performed on the first segment (frames 1-10) and are not shown. As
expected, well separated objects, (a)-(c), are correctly handled.
While partial occlusion, (d)-(e) is also handled correctly, the
algorithm fails to deal with a complete occlusion, (f). }
\end{figure}

Overall, the proposed method performs well on data collected in a
natural environment.
 The initialization strategy and the
model selection criterion proved to be robust to noise and to
minor deviations from the Gaussian distribution assumption. It
possesses the features of a global optimization method which
enables to find initial parameter values that are close to optimal
ones. In both examples, the parameter initialization and model
selection were performed on the first audio-visual data segment.
This certainly biases the overall results. Indeed, in both cases,
the initialization and model selection algorithms dealt with a
case were the objects were well separated. One could rerun
initialization and model selection on every data segment, at the
cost of a less efficient procedure.

The conjugate clustering method automatically weights the auditory and
visual modalities, in terms of precision
and amount of observations, to infer the parameter values. We noticed
that, in general,
the visual data are considered by the algorithm as more reliable. 
This can be explained by the fact that,
in practice, the auditory signals are contaminated with noise and reverberations.
This typically smooths the histogram peaks in the ITD domain and adds false peaks,
as can be seen in Figures~\ref{fig:m1res} and~\ref{fig:visres}. As reverberations
are natural for most of the environments and sound sources, we added auditory
cluster variances to model the local smoothing effect, as well as an outlier category to treat
false peaks. In general, if the data is gathered using a small time interval,
reverberations and noise have higher effect, the observations are scattered
and auditory spatial localization is poor.
At the same time, widening the time interval would result in sharper peaks for sound sources
that are smoothed due to reverberations and dynamics of the scene, and hence the
auditory temporal localization will be less accurate.
Thus the auditory data are typically sparse both in time
and space. The temporal discontinuity of the auditory data together with
the lack of resolution makes it less reliable than the visual data.

Although our multimodal clustering model has no built-in dynamic capability, as is the case
with target-tracking methods based on the Kalman filter, the
implemented algorithm performs quite well in the case of partial
visual occlusions, as illustrated in the cocktail party scenario.

\section{Conclusions}
\label{sec:discuss}

We proposed a novel framework to cluster
heterogeneous data gathered with physically different sensors.
Our approach differs from other existing approaches
in that it combines in a single statistical model a number of clustering
tasks while ensuring the consistency of their results.
In addition, the fact that the clustering is performed in observation spaces
allows one to get useful statistics on the data, which is an advantage of
our approach over particle filtering models.
The task of simultaneous clustering in spaces of
different nature, related through known functional dependencies to
a common parameter space, was formulated as a likelihood
maximization problem. Using the ideas underlying the classical EM
algorithm we built the conjugate EM algorithm to perform the
multimodal clustering task, while keeping attractive convergence
properties. The analysis of the conjugate EM algorithm and, more
specifically, of the optimization task arising in the M-step,
revealed several possibilities to increase the convergence speed.
We  proposed to decompose the M-step into two procedures,  namely
the {\it Local Search} and {\it Choose} procedures,
which allowed us to derive a number of acceleration strategies. We
exhibited appealing properties of the target function which
induced  several implementations of these procedures resulting in
a significantly improved convergence speed.
We introduced the {\it Initialize} and {\it Select} procedures to
efficiently choose initial parameter values and determine
the number of clusters in a consistent manner respectively.
A non trivial
audio-visual localization task was considered to illustrate the
conjugate EM performance on both simulated and real data.
Simulated data experiments allowed us to assess the average method
behaviour in various configurations. They showed that the obtained
clustering results were precise as regards the observation spaces
under consideration. They also illustrated the theoretical
dependency between the precisions in observation and parameter
spaces. Real data experiments then showed that the observed data
precision was high enough to guarantee high precision in the
parameter space.

One of the strong points of the formulated model is that it is
open to different useful extensions. It can be easily extended to
an arbitrary number $J$ of observation spaces $\fos_1,\ldots,\fos_J$.
The main results, including {\it Local Search} and {\it Choose}
acceleration strategies stay valid with minor changes. The sum of
two terms, related to spaces $\fos$ and $\sos$, would have
to be replaced by a sum of $J$ terms  corresponding to
$\fos_1,\ldots,\fos_J$ in the formulas of
Section~\ref{sec:conjem}.

In particular, adding Gaussian priors
on parameters (i.e., priors, covariance matrices and objet locations)
would not essentially change the formulae.
For a large class of dynamics equations, the update
expressions~(\ref{eq:msteppriors})-(\ref{eq:optvars-more})
for priors and variances will remain in closed form,
whereas the function $Q_{\cind}^{\iter}(\param)$
in~(\ref{eq:mstepmax}) will receive an additional term $\log P(\param)$.
For instance, multimodal dynamic inference of parameter values for
Brownian dynamics~\citep{vankampen07stochastic}
can be performed by means of the formulated model.
Gaussian priors would add a quadratic term similar to the others in~(\ref{eq:mstepmax}),
that can be viewed as an `observation' from the ambient space modality.
Thus the optimization algorithm would not require any changes
and would give an unbiased estimate.

Also, the  assumption that assignment
variables $\fsasss$ and $\ssasss$ are independent could be
 relaxed. An appropriate approach  to perform
inference in a non independent case would be to consider
variational approximations \citep{jordan98introduction} and in
particular a variational EM (VEM) framework. The general idea
would be to approximate the joint distribution $P(\fsasss)$ by a
distribution from a restricted class of probability distributions
that factorize as
$\tilde{P}(\fsasss)=\prod\limits_{\find=1}^{\fmind}
\tilde{P}(\fsassind)$. For any such distribution, our model would
be  applicable without any changes so that for  a variational
version of the conjugate EM algorithm, all the results from
Section~\ref{sec:conjem} would hold.

It appears that as a generalization of Gaussian mixture models,
our model has larger modelling capabilities. It is entirely based
on a mathematical framework in which each step is theoretically
well-founded. Its ability to provide good results in a non trivial
multimodal clustering task is particularly promising for
applications requiring the integration of  several heterogenous
information sources. Therefore, it has advantages over other
methods that include ad-hoc processing while being open to
incorporation of more task dependent information.


\section*{Acknowledgements}

The authors would like to thank Miles Hansard (INRIA Grenoble Rh{\^o}ne-Alpes)
and Heidi Christensen (Department of Computer Science, University of Sheffield)
for providing their software for visual and auditory feature detection.



\end{document}